\documentclass{article}

\usepackage{microtype}
\usepackage{graphicx}
\usepackage{booktabs} 
\usepackage{amssymb,amsthm,amsmath,amscd}

\usepackage{mathrsfs}
\usepackage{nicefrac}
\usepackage{etexcmds}
\usepackage{thmtools}
\usepackage{thm-restate}
\usepackage{bbold}
\usepackage{xcolor}
\usepackage{setspace}
\usepackage{multicol}
\usepackage{enumitem}
\usepackage{calc}
\usepackage{changepage}
\usepackage{wrapfig}
\usepackage{todonotes}
\usepackage{tikz}
\usepackage{pgfplots}
\pgfplotsset{select coords between index/.style 2 args={
    x filter/.code={
        \ifnum\coordindex<#1\fi
        \ifnum\coordindex>#2\fi
    }
}}
\usetikzlibrary{arrows.meta}
\usetikzlibrary{math}
\usetikzlibrary{decorations.pathreplacing,calligraphy}
\usetikzlibrary{plotmarks}
\usetikzlibrary{external}
\usepackage{pgfplots}
\usepackage{xxcolor}
\pgfplotsset{compat=1.15}

\pgfdeclareradialshading[tikz@ball]{ball}{\pgfqpoint{0bp}{0bp}}{%
	color(0bp)=(tikz@ball!0!white);
	color(7bp)=(tikz@ball!0!white);
	color(15bp)=(tikz@ball!70!black);
	color(20bp)=(black!70);
	color(30bp)=(black!70)}

\makeatletter
\newcommand{\gettikzcoordinates}[3]{%
	\tikz@scan@one@point\pgfutil@firstofone#1\relax
	\pgfmathsetmacro{\myx}{round(0.99626*\the\pgf@x/0.0283465)/1000}
	\pgfmathsetmacro{\myy}{round(0.99626*\the\pgf@y/0.0283465)/1000}
	\pgfmathsetmacro{\myz}{round(0.99626*\the\pgf@z/0.0283465)/1000}
	\global\edef#2{(\myx,\myy,\myz)}%
}
\makeatother

\tikzset{viewport/.style 2 args={
		x={({cos(-#1)*1cm},{sin(-#1)*sin(#2)*1cm})},
		y={({-sin(-#1)*1cm},{cos(-#1)*sin(#2)*1cm})},
		z={(0,{cos(#2)*1cm})}
}}

\pgfplotsset{only foreground/.style={
		restrict expr to domain={rawx*\CameraX + rawy*\CameraY + rawz*\CameraZ}{-0.05:100},
}}
\pgfplotsset{only background/.style={
		restrict expr to domain={rawx*\CameraX + rawy*\CameraY + rawz*\CameraZ}{-100:0.05}
}}

\def\addFGBGplot[#1]#2;{
	\addplot3[#1,only background, opacity=0.25] #2;
	\addplot3[#1,only foreground] #2;
}

\newcommand{\ViewAzimuth}{-10}
\newcommand{\ViewElevation}{18}
\def\rad{.9}

\newcommand{\drawsphere}[1]{
	\pgfmathsetmacro{\CameraX}{sin(\ViewAzimuth)*cos(\ViewElevation)}
	\pgfmathsetmacro{\CameraY}{-cos(\ViewAzimuth)*cos(\ViewElevation)}
	\pgfmathsetmacro{\CameraZ}{sin(\ViewElevation)}
	\path[use as bounding box] (-1,-1) rectangle (1,1); 
	\begin{scope}
		\clip (0,0) circle (\rad);
		\begin{scope}[transform canvas={rotate=-20}]
			\shade [ball color=white] (0,0.5*\rad) ellipse ({1.8*\rad} and {1.5*\rad});
		\end{scope}
	\end{scope}
	\begin{axis}[
		at = {(0,0)},
		hide axis,
		view={\ViewAzimuth}{\ViewElevation},     
		every axis plot/.style={very thin},
		disabledatascaling,                      
		anchor=origin,                           
		viewport={\ViewAzimuth}{\ViewElevation}, 
		]
		\foreach \t in {0,45,...,180}
		\addFGBGplot[domain=0:2*pi, samples=100, samples y=1, color=gray, line width=.5] (\rad*cos(\t)*sin(deg(x)), {\rad*sin(\t)*sin(deg(x))}, {\rad*cos(deg(x))});
		
		\addFGBGplot[domain=0:2*pi, samples=100, samples y=1, line width=.5] ({\rad*sin(#1)*cos(deg(x))}, {\rad*sin(#1)*sin(deg(x))}, {\rad*cos(#1)});
}

\tikzstyle{border} = [line width = 1]
\tikzstyle{legendstyle} = [draw = gray, font = \small]

\definecolor{unigreen}{HTML}{1A4C39}
\definecolor{uniblue}{HTML}{405994}
\definecolor{unired}{HTML}{A6624C}

\definecolor{tabred}{HTML}{D62728}
\definecolor{tabgreen}{HTML}{2CA02C}
\definecolor{tabblue}{HTML}{1F77BE}

\definecolor{cecolor}{HTML}{3274A1}
\definecolor{cefixcolor}{HTML}{E1812C}
\definecolor{sccolor}{HTML}{3A923A}

\newcommand{\bbN}{\mathbb{N}}

\newcommand{\bbR}{\mathbb{R}}
\newcommand{\bbS}{\mathbb{S}}

\newcommand{\R}{\mathbb{R}}
\newcommand{\N}{\mathbb{N}}

\newcommand{\bbc}{\mathbb{c}}

\DeclareMathOperator{\bb1}{\mathbb{1}}

\newcommand{\mcB}{\mathcal{B}}
\newcommand{\mcC}{\mathcal{C}}

\newcommand{\mcL}{\mathcal{L}}

\newcommand{\mcP}{\mathcal{P}}

\newcommand{\mcW}{\mathcal{W}}
\newcommand{\mcX}{\mathcal{X}}
\newcommand{\mcY}{\mathcal{Y}}
\newcommand{\mcZ}{\mathcal{Z}}

\newcommand{\inprod}[2]{\langle #1 , #2 \rangle}

\newcommand\restr[2]{{
  \left.\kern-\nulldelimiterspace 
  #1 
  \vphantom{\big|} 
  \right|_{#2} 
  }}

\DeclareMathOperator*{\argmax}{\arg\!\max}
\DeclareMathOperator*{\softmax}{softmax}

\newcommand{\ie}{i.e.}
\newcommand{\eg}{e.g.}
\newcommand{\wrt}{w.r.t.~}

\newcommand{\supcon}{SC}
\newcommand{\ce}{CE}

\newcommand*{\lmset}{\{\mskip-4mu\{}
\newcommand*{\rmset}{\}\mskip-4mu\}}
\newcommand{\mset}[1]{\lmset #1 \rmset}
\newcommand{\set}[1]{\left\{ #1\right\}}

\newcommand{\mult}[1]{\operatorname{m}_{#1}}

\newcommand{\msetch}[2]{
\left.\mathchoice
  {\left(\kern-0.45em\binom{#1}{#2}\kern-0.45em\right)}
  {\big(\kern-0.10em\binom{\smash{#1}}{\smash{#2}}\kern-0.10em\big)}
  {\left(\kern-0.10em\binom{\smash{#1}}{\smash{#2}}\kern-0.10em\right)}
  {\left(\kern-0.10em\binom{\smash{#1}}{\smash{#2}}\kern-0.10em\right)}
\right.}
\newcommand{\By}[3]{B_{#1}}

\newcommand{\ymset}{\Upsilon}

\newcommand{\dm}[2]{\inprod{#1}{#2}}

\newcommand{\LCE}{\mcL_{\operatorname{CE}}}
\newcommand{\LSC}{\mcL_{\operatorname{SC}}}
\newcommand{\LLCE}{\ell_{\operatorname{CE}}}
\newcommand{\LLSC}{\ell_{\operatorname{SC}}}
\newcommand{\LSCB}{\ell_{\operatorname{SC}}}

\newcommand{\SCE}{S}

\newcommand{\SSCB}{S}

\newcommand{\enc}{\varphi}
\newcommand{\dec}{W}

\newcommand{\norm}[1]{\left\lVert#1\right\rVert}

\usepackage{hyperref}
\usepackage{bookmark}

\usepackage[accepted]{icml2021}

\usepackage[labelfont={bf,small},font={small,stretch=1.1}]{caption}
\usepackage{subcaption}

\newcommand{\papertitle}{Dissecting Supervised Contrastive Learning}
\icmltitlerunning{\papertitle}

\theoremstyle{plain}

\newtheorem{lemma}{Lemma}

\newtheorem{definition}{Definition}

\newtheorem{remark}{Remark}

\usepackage{inconsolata}
\usepackage[capitalise]{cleveref}

\begin{document}
\bookmarksetup{startatroot}
\twocolumn[
\icmltitle{Dissecting Supervised Contrastive Learning}



\icmlsetsymbol{equal}{*}

\begin{icmlauthorlist}
  \icmlauthor{Florian Graf}{sbg}
  \icmlauthor{Christoph D. Hofer}{sbg}
  \icmlauthor{Marc Niethammer}{unc}
  \icmlauthor{Roland Kwitt}{sbg} 
\end{icmlauthorlist}

\icmlaffiliation{sbg}{Department of Computer Science, University of Salzburg, Austria}
\icmlaffiliation{unc}{UNC Chapel Hill}

\icmlcorrespondingauthor{Florian Graf}{\href{mailto:florian.graf@sbg.ac.at}{\texttt{florian.graf@sbg.ac.at}}}

\icmlkeywords{Machine Learning, ICML}

\vskip 0.3in
]



\printAffiliationsAndNotice{}  

\begin{abstract}
Minimizing cross-entropy over the softmax scores of a linear map composed with a high-capacity encoder is arguably the most popular choice for training neural networks on supervised learning tasks. 
However, recent works show that one can \emph{directly} optimize the encoder instead, to obtain equally (or even more) discriminative representations via a supervised variant of a contrastive objective. 
In this work, we address the question whether there are fundamental differences in the sought-for representation geometry in the output space of the encoder at minimal loss.
Specifically, we prove, under mild assumptions, that both losses attain their minimum once the representations of each class collapse to the vertices of a regular simplex, inscribed in a hypersphere. 
We provide empirical evidence that this configuration is attained in practice and that reaching a close-to-optimal state typically indicates good generalization performance. 
Yet, the two losses show remarkably different optimization behavior. 
The number of iterations required to perfectly fit to data scales \emph{superlinearly} with the amount of randomly flipped labels for the supervised contrastive loss. 
This is in contrast to the approximately \emph{linear} scaling previously reported for networks trained with cross-entropy.
\end{abstract}

\section{Introduction}
\label{sec:introduction}
In modern machine learning, neural networks have become the prevalent choice to  parametrize maps from a complex input space $\mcX$ to some target space $\mcY$. In supervised learning tasks, where the output space is a set of discrete labels, $\mcY = \{1,\ldots,K\}$, it is common to implement predictors of the form
\begin{equation}
    f = \argmax\circ \,\dec \circ \enc \enspace.
    \label{eqn:predictor}
    \end{equation}
In this construction, $f$ is realized as the composition of an encoder $\enc: \mcX \to \mcZ \subseteq \bbR^h$, a linear map/classifier $\dec: \bbR^h \to \bbR^K$ and the $\argmax$ operation which handles the transition from continuous output to discrete label space.

Despite myriad advances in designing networks that implement $\enc$, such as \cite{Krizhevsky12a,He16a,Zagoruyko16a,Huang17a}, the training routine rarely deviates from minimizing the \emph{cross-entropy (CE)} between softmax scores of $W \circ \enc$ and one-hot encoded discrete labels.
Assuming sufficient encoder capacity, it is clear that at minimal loss, 
the representations of training instances, i.e., their images under $\enc$, are in a linearly separable configuration (as the classifier is implemented as a linear map).
Remarkably, this behavior is not only observed on real data with semantically meaningful labels, but also on real data with randomly flipped labels \cite{CZhang2017a}.

\begin{figure*}
    \centering{
    \begin{tikzpicture}
    \def \w{3.6cm}
    \def \shift{2.9cm}
    \def \ms{.75}
    \foreach \idx/\pos\loss in {
        {000.csv}/0/9.029,
        {002.csv}/1/8.800, 
        {004.csv}/2/7.513, 
        {010.csv}/3/4.608, 
        {020.csv}/4/4.598, 
        {180.csv}/5/4.596}{
        
        \begin{scope}
            \begin{axis}[
                at={(\pos*\shift,0)},
                width = \w,
                height = \w,
                legend pos = north west,
                legend cell align = left,
                axis equal,
                title ={\small \textbf{SC} = \loss},
                ticks = none,
                hide axis,
                xlabel = {},
                ylabel = {},
                ylabel style = {font=\small, yshift=-3pt},
                xlabel style = {font=\small},
                tick label style = {font=\small},
                legend style = {font=\small},
                ymin = -1.1,
                ymax=1.1,
                xmin=-1.1,
                xmax=1.1,
            ]
            \addplot[
                color=tabred,
                mark=*,
                only marks,
                mark size = \ms,
                line width = 2,
                select coords between index={0}{99}] table [col sep=comma]  {figures/sphere/raw/\idx};
            \addplot[
                color=tabblue,
                mark=*,
                only marks,
                mark size = \ms,
                line width = 2,
                select coords between index={100}{199}] table [col sep=comma]  {figures/sphere/raw/\idx};
            \addplot[
                color=tabgreen,
                mark=*,
                only marks,
                mark size = \ms,
                line width = 2,
                select coords between index={200}{299}] table [col sep=comma] {figures/sphere/raw/\idx};
                \addplot [domain=-180:180, samples=100, color=gray, line width = 1] ({cos(x)},{sin(x)});
        \end{axis}
        \end{scope}
    }  
    \draw [-{Latex[round,length=2mm]}, line width = .6] (0.005*\textwidth,-.5) -- ({.995\textwidth},-.5) node[above left] {\footnotesize Progress};

    \foreach \idx/\pos\loss in {
        {000.csv}/0/1.322,
        {100.csv}/1/1.015,
        {150.csv}/2/0.867, 
        {200.csv}/3/0.710, 
        {300.csv}/4/0.499, 
        {900.csv}/5/0.265
    }{            
        \begin{scope}
            \begin{axis}[
                clip mode=individual,
                at={(\pos*\shift,-4cm)},
                width = \w,
                height = \w,
                legend pos = north west,
                legend cell align = left,
                axis equal,
                title ={\small \textbf{CE} = \loss},
                ticks = none,
                hide axis,
                xlabel = {},
                ylabel = {},
                ylabel style = {font=\small, yshift=-3pt},
                xlabel style = {font=\small},
                tick label style = {font=\small},
                legend style = {font=\small},
            ]
            \addplot[
                color=tabblue,
                mark=*,
                only marks,
                mark size = \ms,
                line width = 2,
                select coords between index={1}{99}] table [col sep=comma]  {figures/ce/raw/\idx};
            \addplot[
                color=tabred,
                mark=*,
                only marks,
                mark size = \ms,
                line width = 2,
                select coords between index={100}{199}] table [col sep=comma]  {figures/ce/raw/\idx};
            \addplot[
                color=tabgreen,
                mark=*,
                only marks,
                mark size = \ms,
                line width = 2,
                select coords between index={200}{299}] table [col sep=comma] {figures/ce/raw/\idx};
            \addplot[
                color=black,                        
                line width = 1,
                select coords between index={300}{301}] table [col sep=comma] {figures/ce/raw/\idx};
            \addplot[
                color=black,                        
                line width = 1,
                select coords between index={302}{303}] table [col sep=comma] {figures/ce/raw/\idx};
            \addplot[
                color=black,                        
                line width = 1,
                select coords between index={304}{305}] table [col sep=comma] {figures/ce/raw/\idx};
            \addplot[
                only marks,
                color=black,                        
                line width = 1,
                select coords between index={300}{300}] table [col sep=comma] {figures/ce/raw/\idx};
            \addplot[
                only marks,
                color=black,                        
                line width = 1,
                select coords between index={302}{302}] table [col sep=comma] {figures/ce/raw/\idx};
            \addplot[
                only marks,
                color=black,                        
                line width = 1,
                select coords between index={304}{304}] table [col sep=comma] {figures/ce/raw/\idx};
        \end{axis}
        \end{scope}
}  
\end{tikzpicture}
    }
    \caption{Loss comparison on a three-class toy problem in 2D with 100 representations ($z_n$) per class. \emph{Left to right} indicates optimization progress. The \emph{top} row shows the point configurations while minimizing the supervised contrastive (SC) loss, \wrt $z$ drawn uniformly on $\bbS^1$. The \emph{bottom} row shows the point configurations when minimizing cross-entropy (CE) over $\softmax(Wz)$ scores (and an $L_2$ penalty $\lambda \|W\|_F^2$),
    \wrt $W$ and $z$ drawn uniformly within the unit disc. For the \ce~loss, black discs (\begin{tikzpicture}\fill (0,0) circle(2pt);\end{tikzpicture})
    indicate the weights and the rays show their directions. In both cases, the $z_n$ with equal label collapse to the vertices of a regular simplex. \emph{Best-viewed in color.}
    \label{fig:toy_sc_vs_ce}}
\end{figure*}
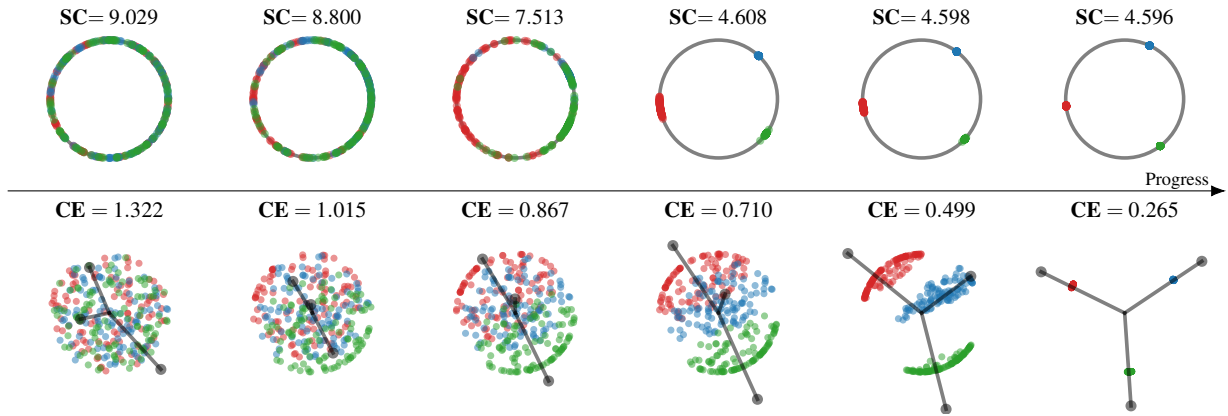

Alternatively, one could aim for \emph{directly} learning an encoder that is compatible with a linear classifier and the $\argmax$ decision rule. Recent works \cite{Koshla20a, Han20a} have shown that this is indeed possible via a supervised variant of a contrastive loss \cite{Chopra05a,Hadsell06a} that has full access to label information. 
Informally, this \emph{supervised contrastive (SC)} loss comprises two competing dynamics: an attraction and a repulsion force. 
The former pulls representations from the same class (positives) closer together, the latter pushes representations from different classes (negatives) away from each other. 
A similar mechanic underpins the triplet loss \cite{Weinberger09a}, the $N$-pairs loss \cite{Sohn16a}, or the soft nearest-neighbor loss \cite{Salakhutdinov07a,Frosst19a} and contributes to the success of self-supervised learning, framed as an instance discrimination task \cite{Oord18a,Chen20a,Henaff20a}. In the context of the latter, positives are typically defined as different views of the \emph{same} instance.

Most notably, predictors obtained by first learning $\enc$ via the supervised contrastive loss, followed by a composition with a linear map, not only yield state-of-the-art results on popular benchmarks, but show increased robustness towards input corruptions and hyperparameter choices \cite{Koshla20a}.
This warrants a closer analysis of the underlying effects. 
While we focus on the formulation of \cite{Koshla20a}, a similar analysis most likely holds for related variants. Specifically, we take a first step toward understanding potential differences in the output space of the encoder, induced either by (1) minimizing (softmax) cross-entropy over $W \circ \enc$, or (2) minimizing the supervised contrastive loss directly over the outputs of $\enc$. 
Characterizing the sought-for \emph{geometric arrangement} of representations of training instances, at minimal loss, is an immediate starting point.  
Our analysis yields \emph{two insights}, summarized below:

\paragraph*{Insight 1 (theoretical).} 
Under the assumption of an encoder $\enc$ that is powerful enough to realize any geometric arrangement of the representations in $\mcZ$, we analyze all loss minimizing configurations of the supervised contrastive and cross-entropy loss, respectively.
More precisely, we prove (see \cref{subsection:analysis_sc}) that the supervised contrastive loss (see \cref{def:sc}) attains its minimum if and only if the representations of each class collapse to the vertices of an origin-centered \emph{regular $K-1$ simplex}, cf. Fig.~\ref{fig:toy_simplex}.
For the cross-entropy loss, we prove a similar, but more nuanced result (see \cref{subsection:analysis_ce}) which is supplemental to an existing line of research. 
In particular, under a norm constraint on the outputs of $\enc$, we show that (1) representations also collapse to the vertices of an origin-centered regular $K-1$ simplex and (2) the classifier weights are (positive) scalar multiples of the simplex vertices. 
Additionally, when subject to $L_2$
penalization, the weights attain equal norm, characterized by a function of the regularization strength. Fig.~\ref{fig:toy_sc_vs_ce} visualizes the convergence to such a configuration on a toy example. 
In \cref{sec:related_work}, we link these results to recent prior work, where an evenly spaced arrangement of classifier weights on the unit hypersphere is either \emph{prescribed} or \emph{explicitly} enforced.

\paragraph*{Insight 2 (empirical).} While our theoretical results assume an \emph{ideal} encoder, we provide empirical evidence on popular vision benchmarks, that the sought-for regular simplex configurations can be attained in practice.
Yet, networks trained with the supervised contrastive loss (1) tend to converge to a state closer to the loss minimizing configuration and (2) empirically yield better generalization performance. 
Hence, as loss minimization strives for a similar geometry of the encoder output for both loss functions (cf. Insight 1), we conjecture that differing optimization dynamics are the primary cause for obtaining solutions of different quality. 
One striking difference is observed when training on data with an increasing fraction of randomly flipped labels, illustrated in Fig.~\ref{fig:time_to_overfit} for a ResNet-18 (CIFAR10), trained with (1) cross-entropy and (2) the supervised contrastive loss (with a subsequently optimized linear classifier $W$).

\begin{figure}[h!]
    \centering
    \begin{tikzpicture}
\begin{axis}[
    width = .99\columnwidth,
    height = 5cm,
    legend pos = north west,
    legend cell align = left,
    axis line style = legendstyle,
    xlabel = {Label corruption},
    ylabel = {Time to fit (max. 100k)},
    xlabel style = {font=\small},
    ylabel style = {font=\small, yshift=-3pt},
    tick label style = {font=\small},
    legend style = {font=\small},
    ymin = 0,
    ymax=60000,
    ytick= {0,20000,40000,60000},
    yticklabels={0,20k,40k,60k},
    ytick align = outside,
    ytick pos = left,
    scaled y ticks=false, 
    xmin=0,
    xmax=1,
    xtick = {0.0,0.2,0.4,0.6,0.8,1.0},
    xtick align = outside,
    xtick pos = bottom,
    grid=both,
    grid style={line width=.1pt, draw=gray!30},
]
\addplot[
    color=cecolor,
    mark=*,
    mark size = {1.5},
    line width = 2,
    ] table [x=x, y=y, col sep=space] {figures/overfit_data/CE_raw.csv};
    \addlegendentry{Cross-Entropy (\textbf{CE})}
    \addlegendentry{Supervised Contrastive (\textbf{SC})}
        \addlegendimage{color=sccolor,
        mark=*,
        mark size = {1.5},
        line width = 2,}
    \addplot[
        color=sccolor,
        mark=*, 
        mark indices = {1,...,8},
        mark size = {1.5},
        line width = 2,
        ] table [x=x, y=y, col sep=space] {figures/overfit_data/SC_raw.csv};
    \addplot[
        color=sccolor,
        mark=square*, 
        mark indices = {9},
        only marks,
        mark size = {3},
        ] table [x=x, y=y, col sep=space] {figures/overfit_data/SC_raw.csv};
\end{axis}
\end{tikzpicture}
    \vspace{-0.1cm}
      \caption{\emph{Time to fit} of a ResNet-18 (on CIFAR10) as a function of increasing label corruption. The green square (\textcolor{sccolor}{\scriptsize $\blacksquare$}) marks the point at which zero training error can no longer be achieved.
    \label{fig:time_to_overfit}}
\end{figure}
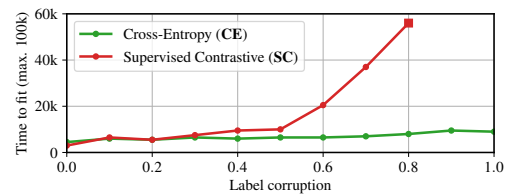

While \cite{CZhang2017a} report an approximately \emph{linear} increase in the time to fit\footnote{i.e., the number of iterations to reach zero training error.} for networks trained with cross-entropy, training with the supervised contrastive loss exhibits a clearly \emph{superlinear} behavior. 
In fact, for a given iteration budget, fitting becomes impossible beyond 
a certain level of label corruption.
This suggests that the supervised contrastive loss exerts some form of implicit regularization during optimization, yielding a parameter incarnation of the network which effectively prevents fitting to random labels.

\vskip-.1cm
\paragraph*{Overview.} \cref{sec:preliminaries} and \cref{sec:analysis} provide the technical details that underpin Insight 1. 
\cref{sec:related_work} draws connections to prior work and \cref{sec:experiments} presents further experiments along the lines of Insight 2. 
\cref{sec:discussion} concludes with a discussion of the main points.

\vskip0.5cm
\section{Preliminaries}
\label{sec:preliminaries}
\vskip.1cm
Consider a supervised learning task with $N\in\N$ training samples, \ie, the learner has access to data $X = (x_1,\ldots,x_N) \in \mcX^N$, drawn i.i.d. from some distribution, and labels, $\{1,\ldots,K\} = [K] \ni y_n = \bbc(x_n)$, assigned to each $x_n$ by an unknown function $\bbc:\mcX \rightarrow [K]$.

We denote the unit-hypersphere (in $\bbR^h$) of radius $\rho>0$ by $\bbS^{h-1}_\rho = \{x \in \bbR^h: \|x\|=\rho \}$; in case of $\rho=1$, we write $\bbS^{h-1}$.
The map $\enc_{\theta}: \mcX \rightarrow \mcZ \subseteq \R^h$ identifies an encoder (see \cref{sec:introduction}), parametrized by a neural network with parameters $\theta$; 
we write $Z_{\theta} = \big(\enc_{\theta}(x_1), \dots \enc_{\theta}(x_N)\big)$ as the image of $X$ under $\enc_{\theta}$. 
When required (e.g., in \cref{subsection:analysis_sc}), we denote a batch by $B$, and identify the batch as the multi-set of indices $\lmset n_1,\ldots,n_b \rmset$ with $n_i \in [N]$. 
For our analysis of the supervised contrastive loss, in \cref{subsection:analysis_sc}, we assume $b\geq 3$.

Under the assumption of an \emph{ideal encoder}, i.e., a map $\enc_\theta$ that can realize every possible geometric arrangement $Z_\theta$ of the representations, we can decouple the loss formulations from the encoder. This facilitates to interpret $Z_\theta$ as a \emph{free configuration} $Z=(z_1,\ldots,z_N)$ of $N$ labeled points (hence, we can omit the dependency on $\theta$).
\vskip.1cm
\subsection{Definitions} 
\label{subsection:definitions}
\vskip-.1cm
For our purposes, we define the \ce~and \supcon~loss, resp., as the loss over all $N$ instances in $Z$. In case of the \ce~loss, this is the average over all instance losses; in case of the \supcon~loss, we sum over \emph{all} batches of size $b \in \bbN$. 
While the normalizing constant is irrelevant for our results, we point out that normalizing the \supcon~loss would depend on the cardinality of \emph{all} multi-sets of size $b$.

\vskip1cm
\begin{restatable}[\textbf{Cross-entropy loss}]{definition}{rest@def@ce} 
\label{def:ce}
    Let $\mcZ \subseteq \bbR^h$ and let $Z$ be an $N$ point configuration, $Z = (z_1,\ldots,z_N) \in \mcZ^N$, with labels $Y=(y_1,\ldots,y_N) \in [K]^N$; let $w_y$ be the $y$-th row of the linear classifiers weight matrix $W \in \R^{K \times h}$. The cross-entropy loss $\LCE(\,\cdot\,, W;\,Y): \mcZ^N \to \bbR$ is defined as
    \begin{equation}
         \ Z \mapsto \frac{1}{N}\sum\limits_{n=1}^N \LLCE(Z, W;\, Y, n)
        \label{def:lce}  
    \end{equation}
    with $\LLCE(\,\cdot\,, W; Y, n): \mcZ^N \to \bbR$ given by
    \begin{equation}
        \LLCE(Z, W;\, Y, n) = 
        -
        \log
        \left(
            \frac
            {
                \exp(
                \dm{z_n}{w_{y_n}})
            }
            {
                \sum\limits_{l=1}^K 
                \exp(
                \dm{z_n}{w_l})
            }
        \right)
        \enspace. 
    \label{def:llce}
    \end{equation}
\end{restatable}

\begin{restatable}[\textbf{Supervised contrastive loss}]{definition}{rest@def@sc}
\label{def:sc}
    Let $\mcZ = \mathbb{S}^{h-1}_{\rho_{\mathcal Z}} \subseteq \bbR^h$ and let $Z$ be an $N$ point configuration, $Z = (z_1,\ldots,z_N) \in \mcZ^N$, with labels $Y=(y_1,\ldots,y_N) \in [K]^N$. 
    For a fixed batch size $b\in \bbN$, we define 
    \begin{equation}
        \mcB = \{\lmset n_1, \dots, n_b \rmset: n_1, \dots, n_b \in [N]\}
    \label{def:index_multi_set}
    \end{equation}
    as the set of all index multi-sets of size $b$. 
    The supervised contrastive loss $\LSC(\,\cdot\,;\, Y): \mcZ^N \to \bbR$ is defined as 
    \begin{equation}  
    Z \mapsto \sum\limits_{B \in \mcB}
    \LLSC(Z; Y, B)
    \label{def:lsc}
    \end{equation}
    with the batchwise loss $\LLSC(\,\cdot\, ;\, Y,B):\mcZ^N \to \bbR$ given by
    \footnote{For notational reasons, we set $\frac{1}{|B_{y_i}|-1}\bb1_{\set{|B_{y_i}|>1}}=0$ when $|B_{y_i}|=1$.}
    \begin{equation}
    -\sum\limits_{i \in B}
    \frac{\bb1_{\set{|B_{y_i}|>1}}}{|B_{y_i}|-1}\hspace{-0.1cm}
    \sum\limits_{j \in B_{y_i}\setminus \lmset i \rmset} \hspace{-0.2cm}
    \!\log
    \left(
        \frac
        {\exp\big(\dm{z_i}{z_j}\big)}
        {
            \sum\limits_{k \in B \setminus \lmset i \rmset}\hspace{-0.2cm}
            \exp\big(\dm{z_i}{z_k}\big)
        }
    \right)
    \label{def:llsc}
    \end{equation}
    where $B_{y_i} = \lmset j \in B: y_{j} = y_i \rmset$ denotes the multi-set of indices in a batch $B \in \mcB$ with label equal to $y_i$.
    \footnote{
        {\cref{def:sc} differs from the original definition by Khosla et al. \cite{Koshla20a} in the following aspects: 
        First, we do not explicitly duplicate batches (\eg, by augmenting each instance). 
        For fixed index $n$, this does not guarantee that at least one other instance with label equal to $y_n$ exists. 
        However, this is formally irrelevant, as the contribution to the summation is zero in that case. 
        Nevertheless, batch duplication is subsumed in our definition. 
        Second, we adapt the definition to multi-sets, allowing for instances to occur more than once. 
        If batches are drawn with replacement, this could indeed happen in practice. 
        Third, we omit scaling the inner products $\dm{\cdot}{\cdot}$ in \cref{def:llsc} by a temperature parameter $1/\tau, \tau >0$, as this complicates the notation. Instead we implicitly subsume this scaling into the radius $\rho_{\mathcal{Z}}$ of~ $\bbS^{h-1}_{\rho_{\mcZ}}$.}
    }
\end{restatable}

As the regular simplex inscribed in a hypersphere plays a key role in our results, we formally define this object next.

\vskip1ex
\begin{restatable}[\textbf{$\rho$-Sphere-inscribed regular simplex}]{definition}{rest@def@simplex}
\label{def:simplex}
    Let $h,K \in \bbN$ with $K\le h+1$.
    We say that $\zeta_1, \dots, \zeta_K \in \R^h$ form the vertices of a regular simplex inscribed in the hypersphere of radius $\rho>0$, if and only if the following conditions hold:
    \vspace{-0.2cm}
    \begin{enumerate}
        [labelindent=7pt,leftmargin=*,label=(S\arabic*),labelwidth=\widthof{\ref{def:simplex:s3}}]
        \item\label{def:simplex:s1} $\sum_{i \in [K]} \zeta_i = 0$ 
        \item\label{def:simplex:s2} $\| \zeta_i \| = \rho$\, for $i \in [K]$
        \item\label{def:simplex:s3} $\exists d \in \R: d = \inprod{\zeta_i}{\zeta_j}$\, for $1 \leq i < j \leq K$ 
    \end{enumerate}
\end{restatable}
\noindent Fig.~\ref{fig:toy_simplex} shows such configurations (for $K=2,3,4$) on $\bbS^2$.

\begin{figure}
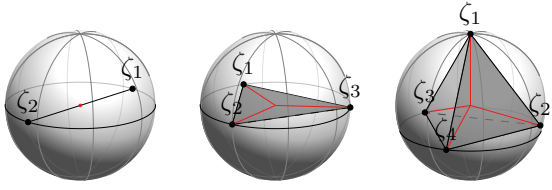

    \centering{
        \newsavebox{\tempboxa}
        \savebox\tempboxa{\begin{tikzpicture}
          \input{figures/simplex/simplex1.tex}
        \end{tikzpicture}}
        \newsavebox{\tempboxb}
        \savebox\tempboxb{\begin{tikzpicture}
          \input{figures/simplex/simplex2.tex}
        \end{tikzpicture}}
        \newsavebox{\tempboxc}
        \savebox\tempboxc{\begin{tikzpicture}
          \input{figures/simplex/simplex3.tex}
        \end{tikzpicture}}
    
        \begin{tikzpicture}
          \node at (0,0) {\usebox\tempboxa};
          \node at (2.5,0) {\usebox\tempboxb};
          \node at (5,0) {\usebox\tempboxc};
          \path[] 
          (-1.5,-1.5) -- (5.5,-1.5) -- (5.5,2) -- (-1.5,2)--cycle;
        \end{tikzpicture}
    }
    \caption{Regular simplices inscribed in $\bbS^2$.\label{fig:toy_simplex}}
    \end{figure}
\begin{remark}
    The assumption $K\le h+1$ is crucial, as it is a necessary and sufficient condition for the existence of the regular simplex.
    In our context, $K$ denotes the number of classes and $K\leq h+1$ is typically satisfied, as the output spaces of encoders in contemporary neural networks are high-dimensional, e.g., 512-dimensional for a ResNet-18 on CIFAR10/100.
    If it is violated, then the bounds derived in \cref{sec:analysis} still hold, but are not tight. Studying the loss minimizing configurations in this regime is much harder. Even for the related and more studied Thomson problem of minimizing the potential energy of $K$ equally charged particles on the 2-dimensional sphere, the minimizers are only known for $K\in \set{2,3,4,5,6,12}$ \cite{Borodachov19a}.
\end{remark}

\section{Analysis}
\label{sec:analysis}
We recap that we aim to address the following question: which $N$ point configurations $Z = (z_1,\ldots,z_N)$ yield minimal \ce~and \supcon~loss? \cref{subsection:analysis_ce} and \cref{subsection:analysis_sc} answer this question, assuming a sufficiently high dimensional representation space $\mcZ\subseteq \bbR^h$, i.e., $K\le h+1$, and \emph{balanced} class labels $Y$, i.e., $\big|\{ i \in [N]: y_i = y\}\big| = \nicefrac{N}{K}$, irrespective of the class $y$. \emph{For detailed proofs we refer to the supplementary material}. 

\subsection{Cross-Entropy Loss}
\label{subsection:analysis_ce}
We first provide a lower bound on the \ce~loss, under the constraint of norm-bounded representations.

\vskip1ex
\begin{restatable}[\textbf{Cross-entropy loss}]{theorem}{rest@thm@ce@bound@frob}
\label{thm:ce_bound_frob}
    Let $\rho_{\mathcal Z}>0$, $\mcZ = \set{z\in \bbR^h:~ \norm{z}\le \rho_{\mathcal Z}}$.
    Further, let $Z=(z_1,\ldots,z_N) \in \mcZ^N$ be an $N$ point configuration with labels $Y=(y_1,\ldots,y_N) \in [K]^N$ and let $W \in \bbR^{K \times h}$ be the weight matrix of the linear classifier from \cref{def:ce}.
    If the label configuration $Y$ is balanced, then 
    \begin{align*}
        &\LCE(Z, W;\,Y)
        \\ &\ge
        \log \left(
            1 + (K-1) \exp \left( 
                - \rho_{\mathcal Z} \frac {\sqrt{K}}{K-1}
                \|W\|_F
            \right)
        \right)
        \enspace,
    \end{align*}
    with equality if and only if there are $\zeta_1, \dots, \zeta_K \in \R^h$ such that:
    \begin{enumerate}[labelindent=8pt,leftmargin=*,label=(C\arabic*),labelwidth=\widthof{\ref{thm:ce_bound_frob:c3}}]
        \item\label{thm:ce_bound_frob:c1} 
        $\forall n \in [N]: z_n = \zeta_{y_n}$ 
        \item\label{thm:ce_bound_frob:c2} 
        $\{\zeta_y\}_{y}$ form a $\rho_{\mcZ}$-sphere-inscribed regular simplex
        \item\label{thm:ce_bound_frob:c3} 
        $\exists \rho_{\mcW} > 0: \forall y \in \mcY: w_{y} = \frac{\rho_{\mcW}}{\rho_{\mcZ}}  \zeta_{y}$ 
    \end{enumerate}
\end{restatable}

Importantly, \cref{thm:ce_bound_frob} states that the bound is tight, if and only if all instances with the same label collapse to points and these points form the vertices of a regular simplex, inscribed in a hypersphere of radius $\rho_{\mcZ}$. 
Additionally, all weights $w_y$ have to attain equal norm and have to be scalar multiples of the simplex vertices, thus also forming a regular simplex (inscribed in a hypersphere of radius $\rho_{\mathcal W}$).

\vskip1ex
\begin{remark}
    \label{rem:papayan}
    Our result complements recent work by Papyan et al. \cite{Papyan20a}, where it is empirically observed that training neural predictors as in \cref{eqn:predictor}\footnote{also including bias terms, i.e., $Wx +b$} leads to a within-class covariance collapse of the representations when continuing to minimize the \ce~loss beyond zero training error. 
    By assuming representations to be Gaussian distributed around each class mean and taking the covariance collapse into account, the regular simplex arrangements of \cref{thm:ce_bound_frob} arise. 
    Specifically, this is the optimal  configuration from the perspective of  recovering the correct class labels. 
    While the analysis in \cite{Papyan20a} is decoupled from the loss function and hinges on a probabilistic argument, we study what happens as the \ce~loss attains its lower bound; 
    our result, in fact, implies the covariance collapse.
\end{remark}
 
\vskip1ex
\begin{restatable}[\textbf{Cross-entropy loss} -- norm-bounded weights]{corollary}{rest@cor@ce@bound@r}
    \label{cor:ce_bound_r}
    Let $Z, Y, W$ be defined as in \cref{thm:ce_bound_frob}. Upon requiring that $\forall y \in [K]: \|w_y\| \leq r_{\mcW}$, it holds that
\begin{equation*}
    \LCE(Z, W;\,Y) 
    \ge 
    \log \left(
        1 + (K-1) \exp \left( 
            -  \frac {K\,\rho_{\mcZ}\,r_{\mcW}}{K-1}
        \right)
    \right)
\end{equation*}
with equality if and only if 
\begin{enumerate}[labelindent=8pt,leftmargin=*,labelwidth=\widthof{\ref{cor_ce_bound_r:c3}}, label = (C\arabic*--r)]
    \item[(C1)]\refstepcounter{enumi}
    $\forall n \in [N]: z_n = \zeta_{y_n}$ 
    \item[(C2)]\refstepcounter{enumi}
    $\{\zeta_y\}_{y}$ form a $\rho_{\mcZ}$-sphere-inscribed regular simplex
    \item
    $\forall y \in \mcY: w_y =\frac{r_{\mcW}}{\rho_{\mcZ}}\zeta_y$\enspace,
    \label{cor_ce_bound_r:c3} 
\end{enumerate}
\noindent \ie,
\ref{thm:ce_bound_frob:c1} and \ref{thm:ce_bound_frob:c2} from \cref{thm:ce_bound_frob} are satisfied and condition \ref{thm:ce_bound_frob:c3} changes to \ref{cor_ce_bound_r:c3}.
\end{restatable}
Notably, a special case of \cref{cor:ce_bound_r} appears in Proposition 2 of \cite{Wang17a}, covering the case where $\forall n: z_n = w_{y_n}$ and $\forall y \in \mcY: \|w_y\|=l$, i.e., equinorm weights and already collapsed classes. 
\cref{cor:ce_bound_r} obviates these constraints and provides a more general result, only assuming that $\forall n: \|z_n\|\leq \rho_{\mcZ}$ and $\forall y: \|w_y\|\leq r_{\mcW}$. 
However, constraining the norm of the weights seems artificial as, in practice, the weights are typically subject to an additional $L_2$ penalty. 
\cref{cor:ce_bound_wd} directly addresses this connection, showing that applying an $L_2$ penalty of the form $\lambda \|W\|_F^2$ eliminates the necessity of an explicit norm constraint. 
\vskip2ex
\begin{restatable}[\textbf{Cross-entropy loss} -- $L_2$ penalty]{corollary}{rest@cor@ce@bound@wd}
    \label{cor:ce_bound_wd}
    Let $Z, Y, W$ be defined as in \cref{thm:ce_bound_frob}.  
    For the $L_2$-regularized objective $\LCE(Z, W;\,Y) + \lambda \|W\|_F^2$ with $\lambda>0$, it holds that 
    \begin{align*}
        \label{eq:cor_l2}
        & \LCE(Z, W;\,Y) + \lambda  \|W\|_F^2
        \\\ge
        &\log \left(
            1 + (K-1) \exp \left( 
                - \rho_{\mathcal Z} \frac {{K}}{K-1}
                r_{\mathcal W}(\rho_{\mcZ}, \lambda)
            \right)
        \right)  \\ &+ 
        \lambda K  r_{\mathcal W}(\rho_{\mcZ}, \lambda)^2\enspace,
    \end{align*}
    where $r_{\mathcal W}(\rho_{\mathcal Z},\lambda)>0$ denotes the unique solution, in $x$, of
    \begin{equation*}
        \label{eq:corl2_eq}
        2  \lambda  x-\frac{\rho_{\mathcal Z} }{\exp\left(\frac{K \rho_{\mathcal Z}  x}{K-1}\right)+K-1} = 0
        \enspace.
    \end{equation*}
    Equality is attained in the bound if and only if
    \begin{enumerate}[labelindent=8pt,leftmargin=*,label=(C\arabic*--wd),labelwidth=\widthof{\ref{cor_ce_bound_wd:c3}}]
        \item[(C1)]\refstepcounter{enumi}
        $\forall n \in [N]: z_n = \zeta_{y_n}$ 
        \item[(C2)]\refstepcounter{enumi}
        $\{\zeta_y\}_{y}$ form a $\rho_{\mcZ}$-sphere-inscribed regular simplex
        \item
        $\forall y \in \mcY: w_y = \frac{r_{\mathcal W}(\rho_{\mcZ},\lambda)}{\rho_{\mcZ}}\zeta_y \enspace,$
        \label{cor_ce_bound_wd:c3} 
    \end{enumerate}
    \noindent \ie,
\ref{thm:ce_bound_frob:c1} and \ref{thm:ce_bound_frob:c2} from \cref{thm:ce_bound_frob} are satisfied and condition \ref{thm:ce_bound_frob:c3} changes to \ref{cor_ce_bound_wd:c3}.
\end{restatable}
\cref{cor:ce_bound_wd} differs from \cref{cor:ce_bound_r} in that the characterization of $w_y$  depends on $r_{\mcW}(\rho_{\mcZ}, \lambda)$, i.e., a function of the norm constraint $\rho_{\mathcal Z}$ on the representations and the regularization strength $\lambda$.
While \cref{eq:corl2_eq} has, to the best of our knowledge, no closed-form solution, $r_{\mcW}(\rho_{\mcZ}, \lambda)$ can still be computed numerically. 
Fig.~\ref{fig:toy_sc_vs_ce} illustrates the attained regular simplex configuration, on a toy example, in case of added $L_2$ regularization.

It is important to note that the assumed norm-constraint on points in $\mcZ$ is not purely theoretical.
In fact, such a constraint often arises\footnote{\dots although it might not be explicitly enforced}, e.g., via batch normalization \cite{Ioffe15a} at the last layer of a network implementing $\enc_\theta$.
While one could, in principle, derive a normalization dependent bound for the CE loss, it is unclear (to the best of our knowledge) if a regular simplex solution satisfying the corresponding equality conditions always exists.

\paragraph*{Numerical Simulation.}
To empirically assess our theoretical results, we revisit the toy example from Fig.~\ref{fig:toy_sc_vs_ce}, where we minimize (via gradient descent) the $L_2$ regularized \ce~loss over $W$ and $Z$ with $\forall n: \|z_n\|\leq 1$. This setting corresponds to having an \emph{ideal} encoder $\enc$ that can realize any configuration of points and matches the assumptions of  \cref{cor:ce_bound_wd}. 
Fig.~\ref{fig:ce_cor2_assessment} (\emph{right}) shows that the lower bound, for varying values of the regularization strength $\lambda$, closely matches the empirical loss. 
Additionally, Fig.~\ref{fig:ce_cor2_assessment} (\emph{left}) shows a direct comparison of the empirical weight average, $\|\overline{w}\|$, \emph{vs.} the predicted theoretical value of $\|w_y\|$ (which is equal for all $y$ in case of minimal loss).
These experiments empirically confirm that conditions \ref{thm:ce_bound_frob:c1} and \ref{thm:ce_bound_frob:c2}, as well as the adapted condition \ref{cor_ce_bound_wd:c3} from \cref{cor:ce_bound_wd} are satisfied. 
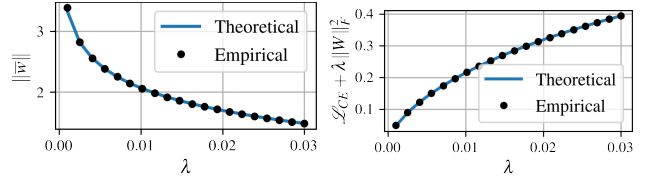
\begin{figure*}[t!]
    \centering{
    \begin{tikzpicture}
\begin{axis}[
    width = .475\textwidth,
    height = 5cm,
    legend pos = north east,
    legend cell align = left,
    axis line style = border,
    xlabel = {Regularization strength $\lambda$},
    ylabel = {Avg. weight norm $\lVert \bar w \rVert$},
    ylabel style = {font=\small, yshift=-3pt},
    xlabel style = {font=\small},
    tick label style = {font=\small},
    legend style = legendstyle,
    ymin = 0.9,
    ymax=3.2,
    ytick= {1,2,3},
    yticklabels={1,2,3},
    ytick align = outside,
    ytick pos = left,
    scaled y ticks=false, 
    xmin=0,
    xmax=0.0525,
    xtick = {0.00, 0.01, 0.02, 0.03, 0.04, 0.05},
    xticklabels = {0.00, 0.01, 0.02, 0.03, 0.04, 0.05},
    xtick align = outside,
    xtick pos = bottom,
    scaled x ticks=false,
    grid=both,
    grid style={line width=.1pt, draw=gray!30},
]
\addplot[
    line width = 2, 
    domain=0.85:4,
    samples y=1,
    smooth,
    color=cecolor] 
    ({1/(2*x) * 1/(exp(3/2*x)+2)},x);
\addplot[
    only marks,
    mark = *,
    mark size = 1.5
] table [x={wd}, y={norm}, col sep=comma] {figures/cor2/raw.csv};
\addlegendentry{Theoretical}
\addlegendentry{Empirical}
\end{axis}
\end{tikzpicture}%
\hfill%
\begin{tikzpicture}
    \begin{axis}[
        width = .475\textwidth,
        height = 5cm,
        legend pos = south east,
        legend cell align = left,
        axis line style = border,
        xlabel = {Regularization strength $\lambda$},
        ylabel = {$\mathcal L_{CE} + \lambda \lVert W \rVert^2_F$},
        ylabel style = {font=\small, yshift=-3pt},
        xlabel style = {font=\small},
        tick label style = {font=\small},
        legend style = legendstyle,
        ymin = 0.,
        ymax=0.525,
        ytick= {0, .1,.2,.3, .4, .5},
        yticklabels={0.0, 0.1,0.2,0.3,0.4,0.5},
        ytick align = outside,
        ytick pos = left,
        scaled y ticks=false, 
        xmin=0,
        xmax=0.0525,
        xtick = {0, 0.01, 0.02, 0.03, 0.04, 0.05},
        xticklabels = {0.00, 0.01, 0.02, 0.03, 0.04, 0.05},
        xtick align = outside,
        xtick pos = bottom,
        scaled x ticks=false,
        grid=both,
        grid style={line width=.1pt, draw=gray!30},
    ]
        \addplot[
            line width = 2, 
            domain=0.85:10,
            samples y=1,
            smooth,
            color=cecolor] 
            (1/(2*x) * 1/(exp(3/2*x) + 2),
            {ln(1 + 2*exp(
                - 3/2*x
            )) +  1/(2*x) * 1/(exp(3/2*x)+2) * 3 * x^2});
        \addlegendentry{Theoretical}
        \addplot[
            only marks,
            mark = *,
            mark size = 1.5
        ] table [x={wd}, y={loss}, col sep=comma] {figures/cor2/raw.csv};
        \addlegendentry{Empirical}
        \end{axis}
\end{tikzpicture}
}
\caption{Numerical simulation for \cref{cor:ce_bound_wd} (on the toy data of Fig.~\ref{fig:toy_sc_vs_ce}), as a function of the $L_2$ regularization strength $\lambda$. The \emph{left} plot shows the theoretical norm of $w_y$ (which is equal for all $y$ at minimal loss) vs. the observed mean norm of the three weights. The \emph{right} plot shows the theoretical bound vs. the empirical $L_2$ regularized \ce~loss.\label{fig:ce_cor2_assessment}}
\end{figure*}
In \cref{sec:experiments}, we will see that the sought-for regular simplex configurations actually arise (with varying quality) when minimizing the $L_2$ regularized \ce~loss for a ResNet-18 trained on popular vision benchmarks.

\subsection{Supervised Contrastive Loss}
\label{subsection:analysis_sc}

An analysis of the \supcon~loss, similar to \cref{subsection:analysis_ce}, is less straightforward. In fact, as the loss is defined over \emph{batches}, we can not simply sum up per-instance losses to characterize the ideal $N$ point configuration. Instead, we need to consider \emph{all} batch configurations of a specific size $b \in \bbN$. 

 We next state our lower bound for the \supcon~loss with the corresponding equality conditions. 
 
\begin{restatable}[\textbf{Supervised contrastive loss}]{theorem}{thm@supcon}
    \label{thm:supcon}
    Let $\rho_{\mathcal Z}>0$ and let $\mcZ = \bbS_{\rho_{\mathcal Z}}^{h-1}$.
    Further, let $Z=(z_1,\ldots,z_N) \in \mcZ^N$ be an $N$ point configuration with labels $Y=(y_1,\ldots,y_N) \in [K]^N$. If the label configuration $Y$ is balanced, it holds that 
    \begin{align*}
        &\LSC(Z;Y)\\
        & \ge 
        \sum_{l=2}^{b} 
        l\, M_l
        \log 
        \left( 
            l - 1 + (b-l)
            \exp \left( 
                - \frac{K\rho_{\mcZ}^2}{K-1}                 
            \right)
        \right)\enspace,
    \end{align*}
    where 
    \begin{equation*}
        M_l = \sum_{y\in\mcY} |\set{ {B \in \mcB}:~ |B_y|=l }|\enspace.
    \end{equation*}
    Equality is attained if and only if the following conditions are satisfied. There are $\zeta_1, \dots, \zeta_{K} \in \R^h$ such that:
    \begin{enumerate}[labelindent=8pt,leftmargin=*,label=(C\arabic*),labelwidth=\widthof{\ref{con:supcon:2}}]
        \item\label{con:supcon:1} 
        $\forall n \in [N]: z_n = \zeta_{y_n}$
        \item\label{con:supcon:2} 
        $\{\zeta_y\}_y$ form a $\rho_{\mcZ}$-sphere-inscribed regular simplex
    \end{enumerate}
\end{restatable}

\cref{thm:supcon} characterizes the geometric configuration of points in $Z$ at minimal loss. 
We see that the equality conditions \ref{thm:ce_bound_frob:c1} and \ref{thm:ce_bound_frob:c2} from \cref{thm:ce_bound_frob} equally appear in \cref{thm:supcon}. 
This implies that, at minimal loss, each class collapses to a point and these points form a regular simplex.

Considering the guiding principle of the \supcon~loss, i.e., separating instances from distinct classes and attracting instances from the same class, it seems plausible that constraining instances to the hypersphere would yield an evenly distributed arrangement of classes. 
However, a closer look at the \supcon~loss reveals that this is not obvious by any means. 
In contrast to the physical (electrostatic) intuition, the involved attraction and repulsion forces are not pairwise, but depend on groups of samples, \ie, batches.
Naively, one could try to characterize the loss minimizing configuration of points for each batch separately. 
Yet, this is destined to fail, as the minimizing  arrangement of points in each batch depends on the label configuration; an example is visualized in Fig.~\ref{fig:minimizers}. 
Hence, there is no simultaneous minimizer for all batch-wise losses.
It is therefore crucial to understand the interaction of the attraction and repulsion forces across different batches. 
We sketch the argument of the proof below and refer to the supplementary material for details.
 
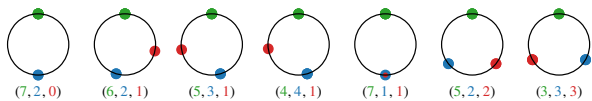
\begin{figure*}[t]
    \centering
\begin{tikzpicture}
    \def \w{3.4cm}
    \def \shift{2.5cm}
    \def \ms{4.5}
        \begin{axis}[
            at={(0*\shift,0)},
            width = \w,
            height = \w,
            axis equal,
            ticks = none,
            hide axis,
            title = {$(\textcolor{tabgreen}{7},\textcolor{tabblue}{2},\textcolor{tabred}{0})$},
            title style = {font = \small},
            ylabel = {},
            ylabel style = {font=\small, yshift=-3pt},
            xlabel style = {font=\small}
            ymin = -1.1,
            ymax=1.1,
            xmin=-1.1,
            xmax=1.1,
        ]
        \addplot[
            color=tabgreen,
            mark=*,
            only marks,
            mark size = \ms,
            line width = 2,
            select coords between index={0}{6}] table [col sep=comma]  {figures/minimizers/720.csv};
        \addplot[
            color=tabblue,
            mark=*,
            only marks,
            mark size = \ms,
            line width = 2,
            select coords between index={7}{8}] table [col sep=comma]  {figures/minimizers/720.csv};
        \addplot [domain=-180:180, samples=100, color=black, line width = 1] ({cos(x)},{sin(x)});
    \end{axis}
    \begin{axis}[
        at={(1*\shift,0)},
        width = \w,
        height = \w,
        axis equal,
        ticks = none,
        hide axis,
        title = {$(\textcolor{tabgreen}{6},\textcolor{tabblue}{2},\textcolor{tabred}{1})$},
        title style = {font = \small},
        ylabel = {},
        ylabel style = {font=\small, yshift=-3pt},
        xlabel style = {font=\small}
        ymin = -1.1,
        ymax=1.1,
        xmin=-1.1,
        xmax=1.1,
    ]
    \addplot[
        color=tabgreen,
        mark=*,
        only marks,
        mark size = \ms,
        line width = 2,
        select coords between index={0}{5}] table [col sep=comma]  {figures/minimizers/621.csv};
    \addplot[
        color=tabblue,
        mark=*,
        only marks,
        mark size = \ms,
        line width = 2,
        select coords between index={6}{7}] table [col sep=comma]  {figures/minimizers/621.csv};
    \addplot[
        color=tabred,
        mark=*,
        only marks,
        mark size = \ms,
        line width = 2,
        select coords between index={8}{8}] table [col sep=comma]  {figures/minimizers/621.csv};
    \addplot [domain=-180:180, samples=100, color=black, line width = 1] ({cos(x)},{sin(x)});
\end{axis}
\begin{axis}[
    at={(2*\shift,0)},
    width = \w,
    height = \w,
    axis equal,
    ticks = none,
    hide axis,
    title = {$(\textcolor{tabgreen}{5}, \textcolor{tabblue}{3}, \textcolor{tabred}{1})$},
    title style = {font = \small},
    ylabel = {},
    ylabel style = {font=\small, yshift=-3pt},
    xlabel style = {font=\small}
    ymin = -1.1,
    ymax=1.1,
    xmin=-1.1,
    xmax=1.1,
]
    \addplot[
        color=tabgreen,
        mark=*,
        only marks,
        mark size = \ms,
        line width = 2,
        select coords between index={0}{4}] table [col sep=comma]  {figures/minimizers/531.csv};
    \addplot[
        color=tabblue,
        mark=*,
        only marks,
        mark size = \ms,
        line width = 2,
        select coords between index={5}{7}] table [col sep=comma]  {figures/minimizers/531.csv};
    \addplot[
        color=tabred,
        mark=*,
        only marks,
        mark size = \ms,
        line width = 2,
        select coords between index={8}{8}] table [col sep=comma]  {figures/minimizers/531.csv};
    \addplot [domain=-180:180, samples=100, color=black, line width = 1] ({cos(x)},{sin(x)});
\end{axis}
\begin{axis}[
    at={(3*\shift,0)},
    width = \w,
    height = \w,
    axis equal,
    ticks = none,
    hide axis,
    title = {$(\textcolor{tabgreen}{4},\textcolor{tabblue}{4},\textcolor{tabred}{1})$},
    title style = {font = \small},
    ylabel = {},
    ylabel style = {font=\small, yshift=-3pt},
    xlabel style = {font=\small}
    ymin = -1.1,
    ymax=1.1,
    xmin=-1.1,
    xmax=1.1,
]
    \addplot[
        color=tabgreen,
        mark=*,
        only marks,
        mark size = \ms,
        line width = 2,
        select coords between index={0}{4}] table [col sep=comma]  {figures/minimizers/441.csv};
    \addplot[
        color=tabblue,
        mark=*,
        only marks,
        mark size = \ms,
        line width = 2,
        select coords between index={5}{7}] table [col sep=comma]  {figures/minimizers/441.csv};
    \addplot[
        color=tabred,
        mark=*,
        only marks,
        mark size = \ms,
        line width = 2,
        select coords between index={8}{8}] table [col sep=comma]  {figures/minimizers/441.csv};
    \addplot [domain=-180:180, samples=100, color=black, line width = 1] ({cos(x)},{sin(x)});
\end{axis}
\begin{axis}[
    at={(4*\shift,0)},
    width = \w,
    height = \w,
    axis equal,
    ticks = none,
    hide axis,
    title = {$(\textcolor{tabgreen}{7},\textcolor{tabblue}{1},\textcolor{tabred}{1})$},
    title style = {font = \small},
    ylabel = {},
    ylabel style = {font=\small, yshift=-3pt},
    xlabel style = {font=\small}
    ymin = -1.1,
    ymax=1.1,
    xmin=-1.1,
    xmax=1.1,
]
    \addplot[
        color=tabgreen,
        mark=*,
        only marks,
        mark size = \ms,
        line width = 2,
        select coords between index={0}{6}] table [col sep=comma]  {figures/minimizers/711.csv};
    \addplot[
        color=tabblue,
        mark=*,
        only marks,
        mark size = \ms,
        line width = 2,
        select coords between index={7}{7}] table [col sep=comma]  {figures/minimizers/711.csv};
    \addplot[
        color=tabred,
        mark=*,
        only marks,
        mark size = 0.5*\ms,
        line width = 2,
        select coords between index={8}{8}] table [col sep=comma]  {figures/minimizers/711.csv};
    \addplot [domain=-180:180, samples=100, color=black, line width = 1] ({cos(x)},{sin(x)});
\end{axis}
\begin{axis}[
    at={(5*\shift,0)},
    width = \w,
    height = \w,
    axis equal,
    ticks = none,
    hide axis,
    title = {$(\textcolor{tabgreen}{5},\textcolor{tabblue}{2},\textcolor{tabred}{2})$},
    title style = {font = \small},
    ylabel = {},
    ylabel style = {font=\small, yshift=-3pt},
    xlabel style = {font=\small}
    ymin = -1.1,
    ymax=1.1,
    xmin=-1.1,
    xmax=1.1,
]
    \addplot[
        color=tabgreen,
        mark=*,
        only marks,
        mark size = \ms,
        line width = 2,
        select coords between index={0}{4}] table [col sep=comma]  {figures/minimizers/522.csv};
    \addplot[
        color=tabblue,
        mark=*,
        only marks,
        mark size = \ms,
        line width = 2,
        select coords between index={5}{6}] table [col sep=comma]  {figures/minimizers/522.csv};
    \addplot[
        color=tabred,
        mark=*,
        only marks,
        mark size = \ms,
        line width = 2,
        select coords between index={7}{8}] table [col sep=comma]  {figures/minimizers/522.csv};
    \addplot [domain=-180:180, samples=100, color=black, line width = 1] ({cos(x)},{sin(x)});
\end{axis}
\begin{axis}[
    at={(6*\shift,0)},
    width = \w,
    height = \w,
    axis equal,
    ticks = none,
    hide axis,
    title = {$(\textcolor{tabgreen}{3},\textcolor{tabblue}{3},\textcolor{tabred}{3})$},
    title style = {font = \small},
    ylabel = {},
    ylabel style = {font=\small, yshift=-3pt},
    xlabel style = {font=\small}
    ymin = -1.1,
    ymax=1.1,
    xmin=-1.1,
    xmax=1.1,
]
    \addplot[
        color=tabgreen,
        mark=*,
        only marks,
        mark size = \ms,
        line width = 2,
        select coords between index={0}{2}] table [col sep=comma]  {figures/minimizers/333.csv};
    \addplot[
        color=tabblue,
        mark=*,
        only marks,
        mark size = \ms,
        line width = 2,
        select coords between index={3}{5}] table [col sep=comma]  {figures/minimizers/333.csv};
    \addplot[
        color=tabred,
        mark=*,
        only marks,
        mark size = \ms,
        line width = 2,
        select coords between index={6}{8}] table [col sep=comma]  {figures/minimizers/333.csv};
    \addplot [domain=-180:180, samples=100, color=black, line width = 1] ({cos(x)},{sin(x)});
\end{axis}
\end{tikzpicture}
\caption{Illustration of \emph{loss minimizing} point configurations of the batch-wise \supcon~loss for varying label configurations and a batch size $b=9$.
Colored numbers indicate the \emph{multiplicity} of each class in the batch.}
\label{fig:minimizers}
\end{figure*}
\vskip.2cm

\paragraph*{Proof Idea for \cref{thm:supcon}.}
The key idea is to decouple the attraction and repulsion effects from the batch-wise formulation of the loss.
Since each batch-wise loss contribution is actually a sum of label-wise contributions, the supervised contrastive loss can be considered as a sum over the Cartesian product of the set of all batches with the set of all labels.
We partition this Cartesian product into appropriately constructed subsets, i.e., by label multiplicity. This allows to apply Jensen's inequality to each sum over such a subset.
In the resulting lower bound, the repulsion and attraction effects are still allocated to the batches, but encoded more tangibly, i.e., linearly, as sums of inner products.
Therefore, their interactions can be analyzed by a combinatorial argument which hinges on the balanced class label assumption. 
Minimality of the respective sums is attained if and only if (1) all classes are collapsed and (2) the mean of all instances (\ie, ignoring the class label) is zero. 
The simplex arrangement arises as consequence of (1) \& (2) and, additionally, the equality conditions yielded by the previous application of Jensen's inequality, \ie, all intra-class and inter-class inner products are equal.

\paragraph*{Numerical Simulation.} For a large number of points, numerical computation of the bound in \cref{thm:supcon} is infeasible due to the combinatorial growth of the number of batches (even for the toy-example of Fig.~\ref{fig:toy_sc_vs_ce} with 300 points). Hence, we consider a smaller setup. In particular, we take $K=3$ classes, each consisting of $4$ points on the unit circle $\bbS^{1}$, \ie, $Z=(z_1,\ldots,z_{12})$, $h=2$ and $\rho = 1$.
For a batch size of $b=9$, this setup yields a total of 167,960 batches, i.e., the number of combinations with replacement. 
We initialize the $z_i$ as the projection of points sampled from a standard Gaussian distribution and 
then minimize the \supcon~loss (by stochastic gradient descent for 100k iterations) over the points in $Z$. Fig.~\ref{fig:sc_eval} (\emph{left}) shows that, at convergence, 
the lower bound on $\LSC(Z;Y)$ closely matches the empirical loss. Fig.~\ref{fig:sc_eval} (\emph{right}) shows the \supcon~loss over \emph{all} batches, highlighting the different loss levels depending on the label configuration in the batch (cf. Fig.~\ref{fig:minimizers}).

\begin{figure*}[t]
    \centering
\begin{tikzpicture}
        \begin{axis}[
            compat = newest,
            width = 0.99\textwidth,
            height = 5cm,
            legend pos = north east,
            legend cell align = left,
            axis line style = border,
            xlabel = {\textbf{SC} batch-wise loss level},
            ylabel = {Nr. of batches},
            ylabel style = {font=\small, yshift=-3pt},
            xlabel style = {font=\small},
            tick label style = {font=\small},
            legend style = {font=\small},
            ymode=log,
            ymin = 90,
            ymax=110000,
            ytick= {100, 1000,10000, 100000},
            yticklabels={$10^2$, $10^3$, $10^4$, $10^5$},
            ytick align = outside,
            ytick pos = left,
            scaled y ticks=false, 
            xtick align = outside,
            xtick pos = bottom,
            scaled x ticks=false,
            grid=both,
            grid style={line width=.1pt, draw=gray!30},
        ]
        
        \addplot[
            ybar,
            fill=sccolor,
            bar width=0.15cm,
        ] table [x={value}, y={count}, col sep=comma]%
        {figures/sc_eval/raw.csv};
    \end{axis}
    \begin{scope}[xshift=330, yshift=40]
        \path[fill=white, draw, rounded corners]  (-0.5,1.8) -- (3.44,1.8) -- (3.4,0.2) -- (-0.5,0.2)--cycle; 
        \node[anchor=north west] at (-0.5,1.7) 
        {\small{\begin{tabular}{rc}
             & $\mathcal L_{SC}(Z;Y)$\\
             \toprule
             Empirical & 12.12016 \\
             Theory & 12.12015
        \end{tabular}
        }};
    \end{scope}    
    \end{tikzpicture}
 
    \caption{
    Numerical optimization of the \supcon~loss for toy data on $\bbS^1$.
    The histogram shows the batch-wise loss values at convergence (over all $\approx$170k batches), illustrating the inhomogeneity of minimal loss values across batch configurations (cf. \cref{fig:minimizers}).
    The table in the upper right corner compares the mean batch-wise loss with the lower bound from \cref{thm:supcon}.
    \label{fig:sc_eval}} 
\end{figure*}
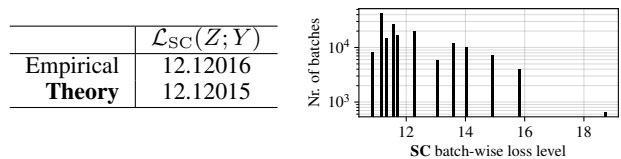

\section{Related work}
\label{sec:related_work}
We focus on works closely linked to our theoretical results of \cref{sec:analysis};  we refer the reader to \cite{Koshla20a} (and references therein) for additional background on the supervised contrastive loss and to \cite{Khac20a} for a general survey on contrastive learning.

Our results on the cross-entropy loss from \cref{subsection:analysis_ce} are partially related to a recent stream of research \cite{Soudry18a,Nacson18a,Gunasekar18a} on characterizing the convergence speed and structure of homogeneous linear predictors ($W$) when minimizing cross-entropy via gradient descent on linearly separable data (i.e., no preceding learned encoder $\enc$).
In particular, \cite{Soudry18a} show that such predictors converge to the $L_2$ max margin separator. 
In our setting, the geometric structure of $W$ (and the outputs of $\enc$) becomes even more explicit, i.e., the weights reside at the vertices of a \emph{regular simplex}.
This is in line with the special case of equinorm representations and weights, presented in \cite{Wang17a}, and complements a recent optimality result by Papyan et al. \cite{Papyan20a} (cf. \cref{rem:papayan} for details).

Along another line of research, several works focus on controlling geometric properties of the classifier weights.
In \cite{Hoffer18a}, for instance, the classifier weights are fixed prior to training, with one choice of the weight matrix being a random orthonormal projection. 
In this setup, all weights have unit norm, are well separated on the hypersphere, but do not form the vertices of a regular simplex. 
Yet, this empirically yields fast convergence, reduces the number of learnable parameters and has no negative impact on performance. 
In \cite{Liu18a}, separation of the classifier weights is achieved via a regularization term based on a Riesz $s$-potential \cite{Borodachov19a}. 
While this regularization term can be added to all network layers, in the special case of the linear classifier weights, the sought-for minimal energy (for $K \leq h+1$) is again attained once the weights form the vertices of a regular simplex.
Recently, Mettes et al. \cite{Mettes20a} presented an approach that a-priori positions so called \emph{prototypes} (one for each class) on the unit hypersphere such that the largest cosine similarity among the prototypes is minimized. Training then reduces to attracting  representations towards their corresponding prototypes. 
Again, in case of $K \leq h+1$, this yields a geometric prototype arrangement at the vertices of a regular simplex. 
In the context of \cref{eqn:predictor}, the prototypes correspond to the classifier weights and are compatible with the $\argmax$ decision rule. 

Overall, \cite{Hoffer18a,Liu18a,Mettes20a} all control, in one way or the other, the geometric arrangement of classifier weights and thereby, implicitly, the arrangement of the representations. This is decisively different to supervised contrastive learning, where the arrangement of the classifier weights is a consequence of the regular simplex arrangement of the representations at minimal loss. More precisely, if representations are already in a regular simplex configuration, the cross-entropy loss of a subsequently trained linear classifier is minimized if and only if the classifier weights are equinorm and scalar multiples of the simplex vertices (cf. \cref{cor:ce_bound_r}).

We additionally point out that several works have recently started to establish a solid theoretical foundation for using contrastive loss functions in the context of unsupervised representation learning. 

Through the concept of \emph{latent classes} (i.e., a construction formalizing the notion of semantic similarity), Arora et al. \cite{Arora19a} prove generalization bounds for downstream supervised classification, under the assumption that the supervised task is defined on a subset of the latent classes. Central to their analysis is the \emph{mean classifier} which is determined by the means of representations of training inputs with equal label. Notably, they empirically observe that this mean classifier performs well on models trained under full supervision. In light of our theoretical results, this can be easily explained by the fact that, at optimality, representations collapse to the simplex vertices.

The \emph{unsupervised} counterpart of the objective we study in this work is analyzed by Wang \& Isola \cite{Wang20a} from a probabilistic perspective.
It is shown that minimizing the (unsupervised) contrastive loss promotes \emph{alignment} and \emph{uniformity} of representations on the unit hypersphere, two properties that empirically correlate with good performance on downstream tasks. 
More precisely, the authors split the (unsupervised) contrastive loss into two summands and show that in the limit of infinite negative samples, asymptotically one is minimized by a \emph{peferctly aligned} and the other by a \emph{perfectly uniform} encoder. 
As pointed out by the authors, if the data is finite, then there is no encoder which is both, i.e., perfectly aligned and perfectly uniform. Hence, in this case, their analysis does not provide an explicit characterization of the loss minimizer.
Complementary to that, our analysis restricts to this very case of finite (training) data, but is able to characterize the loss minimizer in the \emph{supervised} setup.

\section{Experiments}
\label{sec:experiments}
In any practical setting, we do not have an \emph{ideal} encoder (as in \cref{sec:analysis}), but an encoder parameterized as a neural network, $\enc_\theta$.
Hence, in \cref{subsection:theory_vs_practice}, we first assess whether the regular simplex configurations actually arise (and to which extent), given a fixed iteration budget during optimization. Second, in \cref{subsection:random_labels}, we study the optimization behavior of models under different loss functions in a series of random label experiments.

\begin{figure*}[t!]
    \begin{subfigure}{0.68\textwidth}
    \centering{
    \includegraphics[width=1.0\textwidth]{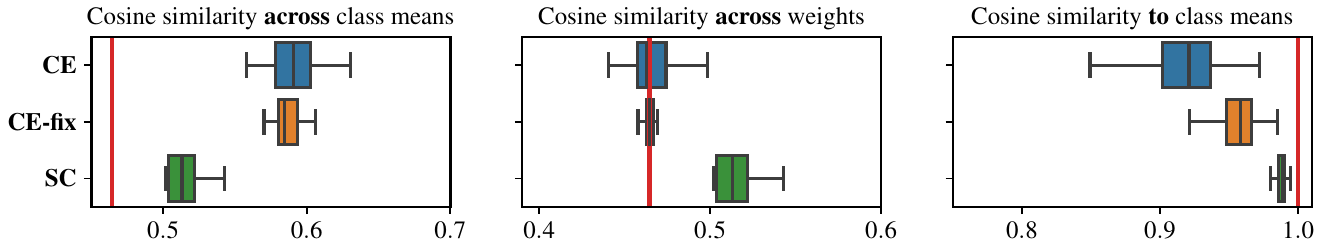}}
    \end{subfigure}\hfill
    \begin{subfigure}[t]{0.3\textwidth}
        \raisebox{0.5mm}{
        \begin{scriptsize}
        \centering{
            \begin{tabular}{r|cc}
                \hline
                & \multicolumn{2}{c}{\textbf{CIFAR10\phantom{0}}\, err. \tiny{[\%]}}\\
                Loss & w aug. & w/o aug. \\
                \hline
                \textbf{CE}            &  6.3          & 16.4  \\
                \textbf{CE-fix} &  6.2          & 14.7  \\
                \textbf{SC}  &  \textbf{5.7} & \textbf{13.8} \\
                \hline
            \end{tabular}
        }    
        \end{scriptsize}}
    \end{subfigure}
    \begin{subfigure}{0.68\textwidth}
        \centering{
        \includegraphics[width=1.0\textwidth]{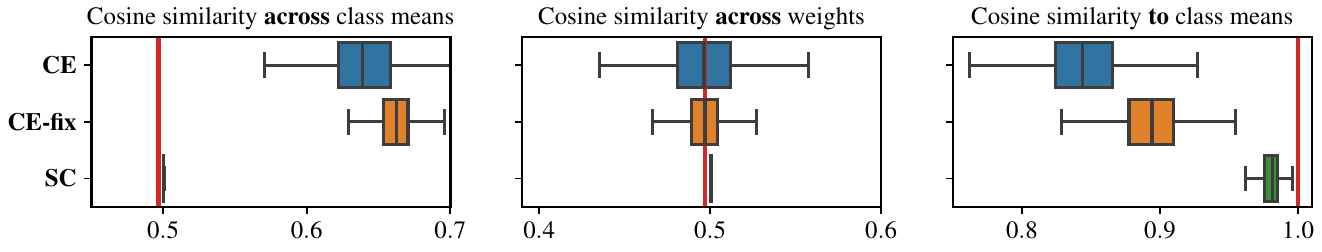}}
    \end{subfigure}\hfill
    \begin{subfigure}[t]{0.3\textwidth}
        \raisebox{0.5mm}{
        \begin{scriptsize}
        \centering{
            \begin{tabular}{r|cc}
                \hline
                & \multicolumn{2}{c}{\textbf{CIFAR100}\, err. \tiny{[\%]}}\\
                Loss   & w aug. & w/o aug. \\
                \hline
                \textbf{CE}              &  27.0         &   41.8         \\
                \textbf{CE-fix}          &  26.3         & \textbf{41.3}  \\
                \textbf{SC}              &  \textbf{24.9}&  41.5          \\
                \hline
            \end{tabular}
        }    
        \end{scriptsize}}
    \end{subfigure}
    \caption{Geometric properties of representations, $\enc_\theta(x_n)$, and weights, obtained from ResNet-18 models trained (\emph{left}: \textbf{CIFAR10}; \emph{right}: \textbf{CIFAR100}) with different losses (using data augmentation, i.e., w aug.). The \emph{left} three panels show the distribution of cosine similarities across (1) distinct class means (quantifying class separation), (2) distinct class weights (quantifying classifier weight separation) and (3) representations with respect to their respective class means (quantifying within class spread). \textcolor{tabred}{Red} lines indicate the sought-for value at a regular simplex configuration. The \emph{right-most} panel shows the testing error of all models with and without data augmentation. \label{fig:geometry}}
\end{figure*}

\subsection{Setup} 
As our choice of $\enc_\theta$, we select a ResNet-18 \cite{He16a} model, i.e., all layers up to the linear classifier.
Experiments are conducted on CIFAR10/100, for which this choice yields $512$-dim. representations (and $K\leq h+1$ holds in all cases).

We either compose $\enc_\theta$ with a linear classifier and train with the \ce~loss function (denoted as \textbf{CE}), or we directly optimize $\enc_\theta$ via the \supcon~loss function, then freeze the encoder parameters and train a linear classifier on top (denoted \textbf{SC}). 
In case of the latter, outputs of $\enc_\theta$ are always projected onto a hypersphere of radius $\rho=\nicefrac{1}{\sqrt{\tau}}$ (with $\tau=0.1$), which accounts for scaling the inner-products by the temperature parameter $\nicefrac{1}{\tau}$ in the original formulation of Khosla et al. \cite{Koshla20a}. 
\textcolor{black}
{We want to stress that while Theorem \ref{thm:supcon} holds for every $\rho>0$, the temperature crucially influences the optmization dynamics and needs to be tuned appropriately.}
For comparison, we also compose $\enc_\theta$ with a \emph{fixed} linear classifier, in particular, a classifier with weights a-priori optimized towards a regular simplex arrangement. This is similar to \cite{Mettes20a}, only that we minimize the \ce~loss (denoted as \textbf{CE-fix}) to learn predictors $W \circ \enc_\theta$, as opposed to pulling outputs of $\enc_\theta$ towards the fixed prototypes/weights.

Optimization is done via (mini-batch) stochastic gradient descent with $L_2$ regularization ($10^{-4}$) and momentum (0.9) for 100k iterations. The batch-size is fixed to 256 and the learning rate is annealed exponentially, starting from 0.1. 
When using data augmentation, we apply random cropping and random horizontal flipping, each with probability $\nicefrac{1}{2}$.

\subsection{Theory vs. Practice}
\label{subsection:theory_vs_practice}

To provide a first impression to which extend the representations of the training data achieve the loss minimizing geometric arrangement, we compare the empirical \textbf{CE} and \textbf{SC} loss values to the optima derived in \cref{sec:analysis}, using ResNet-18 models trained on CIFAR10 (with data augmentation). 
The \textcolor{tabred}{theoretical}/\textcolor{tabblue}{empirical} losses are (1) \textcolor{tabred}{7.64e-5} vs. \textcolor{tabblue}{2.48e-4} (for \textbf{CE}) and (2) \textcolor{tabred}{824.487} vs. \textcolor{tabblue}{824.731} (for \textbf{SC}), where we estimate the empirical \textbf{SC} loss over 1k training batches. Notably, when optimizing for 500k iterations (instead of 100k; see supplementary material), the loss values continue to move closer to the optimum, but at very low speed. In particular, the loss values change to \textcolor{tabblue}{2.27e-4}  (for \textbf{CE}) and \textcolor{tabblue}{824.523} (for \textbf{SC}), respectively. Overall, this suggests that out of these models, representations learned by minimizing the \textbf{SC} loss might arrange closer to the theoretically optimal configuration. As our results only cover the \emph{loss minimizers} and not (close to optima) level sets, the latter hypothesis is more of a first guess and not predicted by the theory.

For a closer look at the geometric arrangement of the representations (and classifier weights), we compute three  statistics, all based on the cosine similarity $\gamma: \bbR^h \times \bbR^h \to [0,1]$, defined as
\begin{equation}
    (x,y)\mapsto 1-\cos^{-1}\left(
            \inprod{\nicefrac{x}{\|x\|}}{\nicefrac{y}{\|y\|}}\right)
        /\pi\enspace.\label{eqn:cs}
\end{equation}
First, we measure the separation of the class representations via the cosine similarity among the class means, $\mu_1, \dots, \mu_K$, i.e., $\gamma(\mu_i,\mu_j)$ for $i \neq j$. 
Second, for the \ce~loss function, we compute the cosine similarity across the classifier weights, \ie,  $\gamma(w_i,w_j), i\neq j$, quantifying their separation.
Third, to quantify class collapse, we compute the cosine similarity among all representations and their respective class means, i.e., 
$\gamma(\enc(x_n),\mu_{y_n})$.
Note that our theoretical results imply that the classes should collapse and the pairwise similarities, as mentioned above, should be equal.

Fig.~\ref{fig:geometry} illustrates the distribution of the cosine similarities for the ResNet-18 model trained with different loss functions (and using data augmentation). We observe that the \supcon~loss leads to (1) an arrangement of the class means much closer to the ideal simplex configuration and (2) a tighter concentration of training representations around their class means. Furthermore, in case of the \ce~loss, the weight arrangement reaches, on average, a regular simplex configuration, while the representations slightly deviate. When using a-priori fixed weights in a simplex configuration, i.e., \textbf{CE-fix}, the situation is similar, but the within-class spread is smaller. In general, the statistics are comparable between CIFAR10 and CIFAR100, only that the distribution of all computed statistics widens for models trained with \textbf{CE} on CIFAR100. We conjecture that the increase in the number of classes, combined with the joint optimization of $\enc_\theta$ and $W$ complicates convergence to the loss minimizing state. Fig.~\ref{fig:geometry} (\emph{right}) further suggests that approaching this state positively correlates with generalization performance. Whether the latter is a general phenomenon, or may even have a theoretical foundation, is an interesting question for future work.

Finally, we draw attention to the comparatively large gap between the cosine similarities \emph{across} the class means and their theoretical prediction in case of models trained on CIFAR10 (Fig. \ref{fig:toy_sc_vs_ce}, \emph{top left}). The aforementioned gap indicates that the chosen encoder might not be powerful enough to arrange the representations on a sphere-inscribed regular simplex. In fact, a standard ResNet \cite{He16a} utilizes a ReLU activation function after each block, including the last block before the linear classifier. Therefore, the coordinates of representations obtained by the encoder part of a standard ResNet are always \emph{non-negative}, and so are the coordinates of the class means. Consequently, their inner products are non-negative as well, which corresponds to a minimal cosine similarity of 0.5 across the class means. Since the scalar products of vertices (considered as position vectors) of a unit sphere inscribed regular simplex with $K$ vertices are $- 1/(K-1)$, the  deviation from the optimal class separation, resulting from the choice of encoder, is unnoticeable for models trained on CIFAR100 due to the large number of classes, i.e., $K=100$, but becomes apparent in case of CIFAR10 where $K=10$.

We suspect that architectures which do not implement the aforementioned non-negativity constraint in the encoder, e.g., the \emph{pre-activation} variants of ResNets \cite{He16b}, are capable of separating the classes to a larger extend and thus match the theoretical prediction more closely when trained on data with a  small number of classes.

\subsection{Random Label Experiments}
\label{subsection:random_labels}

Despite the similarity of the loss minimizing geometric arrangements at the output of $\enc_\theta$, for both (\ce, \supcon) losses, we have seen (in Fig.~\ref{fig:geometry}) that the extent to which this ``optimal'' state is achieved differs. 
These differences likely arise as a result of the underlying optimization dynamics, driven by the loss contribution of each batch. Notably, while the \ce~loss decomposes into independent instance-wise contributions, the \supcon~loss does not (due to the interaction terms).

One way to explore this in greater detail, is to study optimization behavior as a function of label corruption. Specifically, as label corruption (i.e., the fraction of randomly flipped labels) increases, it is interesting to track the number of iterations (\emph{time to fit}) to reach zero training error \cite{CZhang2017a}, illustrated in Fig.~\ref{fig:ttf_exp}. 

On both datasets, \textbf{CE} and \textbf{CE-fix} show an approximately linear growth, while \textbf{SC} shows a remarkably \emph{superlinear} growth. We argue that the latter primarily results from the profound interaction among instances in a batch. Intuitively, as the number of attraction terms for the \supcon~loss function scales quadratically with the number of samples per class, increasing the number of semantically confounding labels equally increases the complexity of the optimization problem. In contrast, for \textbf{CE} and \textbf{CE-fix}, semantically confounding labels only impose per-instance constraints instead. 
\begin{figure}[t]
    \centering{
\begin{tikzpicture}
    \begin{axis}[
        width = \columnwidth,
        height = 5cm,
        legend pos = north west,
        legend cell align = left,
        axis line style = border,
        xlabel = {Label corruption (on \textbf{CIFAR10})},
        ylabel = {Time to fit (max. 100k)},
        ylabel style = {font=\small, yshift=-3pt},
        xlabel style = {font=\small},
        tick label style = {font=\small},
        legend style = legendstyle,
        ymin = 0,
        ymax=60000,
        ytick= {0,20000,40000,60000},
        yticklabels={0,20k,40k,60k},
        ytick align = outside,
        ytick pos = left,
        scaled y ticks=false, 
        xmin=0,
        xmax=1,
        xtick = {0.0,0.2,0.4,0.6,0.8,1.0},
        xtick align = outside,
        xtick pos = bottom,
        grid=both,
        grid style={line width=.1pt, draw=gray!30},
    ]
    \addplot[
        color=cecolor,
        mark=*,
        mark size = {1.5},
        line width = 2,
        ] table [x=x, y=y, col sep=space] {figures/overfit_data/CE_raw.csv};
        \addlegendentry{\textbf{CE}}
        \addplot[
        color=cefixcolor,
        mark=*,
        mark size = {1.5},
        line width = 2,
        ] table [x=x, y=y, col sep=space] {figures/overfit_data/CE_fix_raw.csv};
        \addlegendentry{\textbf{CE-fix}}
        \addlegendentry{\textbf{SC}}
            \addlegendimage{color=sccolor,
            mark=*,
            mark size = {1.5},
            line width = 2,}
        \addplot[
            color=sccolor,
            mark=*, 
            mark indices = {1,...,8},
            mark size = {1.5},
            line width = 2,
            ] table [x=x, y=y, col sep=space] {figures/overfit_data/SC_raw.csv};
        \addplot[
            color=sccolor,
            mark=square*, 
            mark indices = {9},
            only marks,
            mark size = {3},
            ] table [x=x, y=y, col sep=space] {figures/overfit_data/SC_raw.csv};
    \end{axis}    
    \end{tikzpicture}%
    \hfill%
    \begin{tikzpicture}
        \begin{axis}[
            width = \columnwidth,
            height = 5cm,
            legend pos = north west,
            legend cell align = left,
            axis line style = border,
            xlabel = {Label corruption (on \textbf{CIFAR100})},
            ylabel = {Time to fit (max. 100k)},
            ylabel style = {font=\small, yshift=-3pt},
            xlabel style = {font=\small},
            tick label style = {font=\small},
            legend style = legendstyle,
            ymin = 0,
            ymax=60000,
            ytick= {0,20000,40000,60000},
            yticklabels={0,20k,40k,60k},
            ytick align = outside,
            ytick pos = left,
            scaled y ticks=false, 
            xmin=0,
            xmax=1,
            xtick = {0.0,0.2,0.4,0.6,0.8,1.0},
            xtick align = outside,
            xtick pos = bottom,
            grid=both,
            grid style={line width=.1pt, draw=gray!30},
        ]
        \addplot[
            color=cecolor,
            mark=*,
            mark size = {1.5},
            line width = 2,
            ] table [x=x, y=y, col sep=space] {figures/overfit_data/CE100_raw.csv};
            \addlegendentry{\textbf{CE}}
            \addplot[
            color=cefixcolor,
            mark=*,
            mark size = {1.5},
            line width = 2,
            ] table [x=x, y=y, col sep=space] {figures/overfit_data/CEfix100_raw.csv};
            \addlegendentry{\textbf{CE-fix}}
            \addplot[
                color=sccolor,
                mark=*,
                mark size = {1.5},
                line width = 2,
                ] table [x=x, y=y, col sep=space] {figures/overfit_data/SC100_raw.csv};
                \addlegendentry{\textbf{SC}}
            \end{axis}    
\end{tikzpicture}}
\vspace{-0.1cm}
\caption{Time to fit for models of the form $W \circ \enc_\theta$, based on ResNet-18 encoders, optimized under different loss functions.\label{fig:ttf_exp}}
\end{figure}
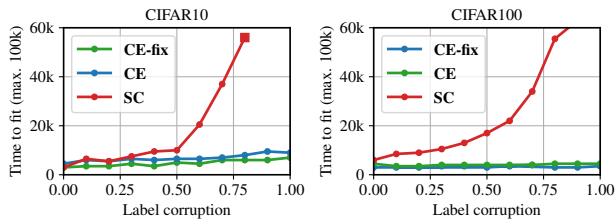
This equally explains why, on CIFAR10, \textbf{SC} cannot achieve zero error beyond 80\% corruption: fewer training instances per class (500 vs. 5,000) yield fewer pairwise intra-class constraints to be met.

\section{Discussion}
\label{sec:discussion}
By focusing on predictors $\argmax\circ\, W\circ \enc$, our results assert that the outputs of $\enc$ are \emph{strikingly similar}, at minimal loss, irrespective of whether we train with cross-entropy or in the supervised contrastive regime.

Yet, from an optimization perspective, the choice of loss makes a profound difference, visible in the differing resilience to fit in the presence of corrupt label information.
We argue that the advantages of supervised contrastive learning, reported in prior work, are  rooted in the strong interaction terms among samples in a batch. While cross-entropy acts sample-wise, the supervised contrastive loss considers pair-wise sample relations, i.e., a \emph{batch} is an \emph{atomic} computational unit during stochastic optimization; in case of cross-entropy, the atomic unit is a \emph{single} sample instead.  

While we simplified the original setup of supervised contrastive learning, in particular, by detaching the commonly used \emph{projective head}, we hope that our results provide a viable starting point for further analyses. Specifically, we think that a better theoretical understanding of the profound interaction between stochastic optimization and loss functions that capture pairwise constraints (rather than instance losses), could be a promising avenue to be explored in the context of the generalization puzzle.

\subsection*{Acknowledgements}

This research was supported in part by the
Austrian Science Fund (FWF): project FWF P31799-N38
and the Land Salzburg (WISS 2025) under project numbers 20102- F1901166-KZP and 20204-WISS/225/197-2019. We also like to 
thank the anonymous reviewers for the constructive feedback during the review process.

\subsection*{Source Code}
Source code to reproduce experiments is publicly available: \url{https://github.com/plus-rkwitt/py_supcon_vs_ce}

\bibliography{papers,books}
\bibliographystyle{icml2021}

%
%
%
\clearpage
\onecolumn















\makeatletter

\renewcommand{\bibnumfmt}[1]{[S#1]}
\renewcommand{\citenumfont}[1]{S#1}

\begin{adjustwidth}{1.5cm}{1.5cm}

  \colorlet{restated}{tabblue}
  %
  \appendix
  \renewcommand{\thesection}{S\arabic{section}}
  \icmltitle{{\large {Supplementary Material}:} \\ \vspace{4pt} \papertitle}
  In this supplementary material, we provide all proofs from \cref{sec:analysis} which were omitted in the main part of the manuscript.

\emph{In the following, we will frequently utilize standard inequalities (e.g., the Jensen inequality, or the Cauchy-Schwarz inequality) and analyzing their equality conditions will be key to get the desired results. 
    In this context, the following notation has proved to be useful:
    If a symbol appears above an inequality sign in an equation, it denotes that the corresponding equality conditions will be discussed later on and referenced with the corresponding symbol. 
    For example, 
    $$a \stackrel{(P)}{\ge} b\enspace,$$
    denotes $a \ge b$ and equality is attained if and only if the condition (P) is satisfied.}

  \section{Proofs for Section~\ref{subsection:analysis_sc}}

In this section, we will prove \cref{thm:supcon} of the main manuscript (restated below). 
Throughout this section the following objects will appear repeatedly an thus are introduced one-off:
\begin{itemize}
    \item $h, N, K \in \N$
    \item $\rho_{\mcZ}>0$
    \item $\mcZ = \bbS_{\rho_{\mathcal Z}}^{h-1}$
    \item $\mcY = \{1, \dots, K\} = [K]$
\end{itemize}

Further, we will consider batches $B\in \mcB$ of an arbitrary but fixed size $b\ge3$. We additionally assume $|\mcY| = K \leq h +1$. 

\thm@supcon*
\subsection{Definitions}

First we will recall the definition of the supervised contrastive (SC) loss and introduce some necessary auxiliary definitions. 
The \supcon~loss is given by 
\begin{equation}
    \LSC(Z;Y) = \sum_ {B \in \mcB} \LSCB(Z;Y,B) \enspace,
\end{equation}
where $\LSCB(Z;Y,B)$ is the \emph{batch-wise loss} 
\begin{equation}
    \LSCB(Z;Y,B) = 
    -\sum\limits_{\substack{i \in B\\ |B_{y_i}|>1}}
    \frac{1}{|B_{y_i}|-1}
    \sum\limits_{j \in B_{y_i}\setminus \lmset i \rmset} 
    \log
    \left(
        \frac
        {\exp\big(\dm{z_i}{z_j}\big)}
        {
            \sum\limits_{k \in B \setminus \lmset i \rmset}
            \exp\big(\dm{z_i}{z_k}\big)
        }
    \right)
    \enspace.
\end{equation}
We next introduce the \emph{class-specific batch-wise loss} 
\begin{equation}
    \LSCB(Z;Y,B,y)
    = 
    \begin{cases}
        -
        \sum\limits_{i \in B_y}     
        \frac{1}{|B_{y_i}|-1}
        \sum\limits_{j \in B_{y_i}\setminus \lmset i \rmset} 
        \log
        \left(
            \frac
            {\exp\big(\dm{z_i}{z_j}\big)}
            {
                \sum\limits_{k \in B \setminus \lmset i \rmset}
                \exp\big(\dm{z_i}{z_k}\big)
            }
        \right)
        & \text{ if } |B_{y}|>1
        \\
        0 & \text{ else}\enspace.
    \end{cases}
    \enspace,
\end{equation}
This allows us to write 
\begin{align}
    \LSCB(Z;Y,B)
    & = 
        -\sum\limits_{\substack{i \in B\\|B_{y_i}|>1}}
        \frac{1}{|B_{y_i}|-1}
        \sum\limits_{j \in B_{y_i}\setminus \lmset i \rmset} 
        \log
        \left(
            \frac
            {\exp\big(\dm{z_i}{z_j}\big)}
            {
                \sum\limits_{k \in B \setminus \lmset i \rmset}
                \exp\big(\dm{z_i}{z_k}\big)
            }
        \right)
        \\
    & = 
        - 
        \sum_{\substack{y\in \mcY\\|B_y|>1}}
        \sum\limits_{i \in B_y}     
        \frac{1}{|B_{y_i}|-1}
        \sum\limits_{j \in B_{y_i}\setminus \lmset i \rmset} 
        \log
        \left(
            \frac
            {\exp\big(\dm{z_i}{z_j}\big)}
            {
                \sum\limits_{k \in B \setminus \lmset i \rmset}
                \exp\big(\dm{z_i}{z_k}\big)
            }
        \right)
        \\
    & = 
        \sum_{y\in \mcY}
        \LSCB(Z;Y,B,y) \enspace,
\end{align} 
and so 
\begin{equation}
    \LSC(Z;Y) = \sum_{y \in \mcY} \sum_{B \in \mcB} \LSCB(Z;Y,B,y) 
    \enspace.
\end{equation}
We use the following notation:  
the multiplicity of an element $x$ in the multiset $M$ is denoted by $\mult{M}(x)$. 
Furthermore, we introduce the label configuration of a batch, \ie, 
\begin{equation}
    \ymset(B) = \mset{y_i: ~ i\in B }    
    \enspace,
\end{equation} 
thus $\mult{\ymset(B)}(y) = |{\By{y}{Y}{B}}|$.

For example, if $\mcY = \set{a,b}$, $B = \mset{1,2,2,5,10}$ and $a = y_1=y_2$, $b=y_5=y_{10} $, then $\mult{B}(2)=2$, $\ymset(B) = \mset{a,a,a,b,b}$ and $\mult{\ymset(B)}(a) = 3 = |\mset{1,2,2}| = |B_a|$. 
We will slightly abuse notation ($Y$ is a tuple, not a multiset) and write $\mult{Y}(y) = \mult{\ymset([N])}(y) = |\mset{n \in [N]: y_n = y}|$. 
For every batch $B\in \mcB$ and label $y\in \mcY$, we will also write ${B_y}^C:=\mset{i\in B:~y_i\neq y}$ for the complement of ${B_y}:=\mset{i\in B:~y_i = y}$, which was already introduced in the \cref{def:sc} of the supervised contrastive loss.

\begin{definition}[Auxiliary partition of $\mcB$]
    For every $y\in \mcY$ and every $l\in \set{0,\dots,b}$, we define
    \begin{equation}
        \mcB_{y,l} : = \set{ {B \in \mcB}:~ \mult{\ymset(B)}(y) = l}
    \end{equation}
    as the set of all batches which contain exactly $l$ instances of class $y$.
\end{definition}

\begin{definition}[Auxiliary functions $S$, $S_{\text{rep}}$, $S_{\text{att}}$]
    Let $h, N \in \N$, $\rho_{\mcZ}>0$ and $\mcZ = \bbS_{\rho_{\mathcal Z}}^{h-1}$. 
    For fixed label configuration $Y\in \mcY^N$, batch $B\in \mcB$ and label $y\in \mcY$ with $\mult{\ymset(B)}(y)>1$, we define 
    \begin{align}
        \SCE(\,\cdot\,;Y,B,y): \mcZ^N &\to \bbR \\
        Z &\mapsto 
        S_{\text{att}}(Z;Y,B,y) + S_{\text{rep}}(Z;Y,B,y)
        \enspace,
    \end{align}
    where 
    \begin{align}
        \label{eq:def_s_att}
        S_{\text{att}}(Z;Y,B,y) 
        & = -
        \frac{1}{|{\By{y}{Y}{B}}|\,(|{\By{y}{Y}{B}}|-1)}
        \sum_{i \in {\By{y}{Y}{B}}}
        \sum_{j \in {\By{y}{Y}{B}}\setminus\mset{i}} 
        \inprod{z_i}{z_j}
        \\
        \label{eq:def_s_rep}
        S_{\text{rep}}(Z;Y,B,y) 
        &= 
        \begin{cases}
            \frac{1}{|{\By{y}{Y}{B}}|\,|{\By{y}{Y}{B}}^C|}
        \sum\limits_{i \in {\By{y}{Y}{B}}}
        \sum\limits_{j \in {\By{y}{Y}{B}}^C} 
        \inprod{z_i}{z_j}
        \qquad &\text{if } \mult{\ymset(B)}(y)\neq b
        \\
        0  &\text{if } \mult{\ymset(B)}(y) = b
        \end{cases}        
        \enspace.
    \end{align}
    Furthermore, we define 
    \begin{equation}
        S_{i}(Z;Y,B,y) = 
        \begin{cases}
            - \frac{1}{|{\By{y}{Y}{B}}|-1}
        \sum\limits_{j \in {\By{y}{Y}{B}}\setminus\mset{i}} 
        \inprod{z_i}{z_j}
        +
        \frac{1}{|{\By{y}{Y}{B}}^C|}
        \sum\limits_{j \in {\By{y}{Y}{B}}^C} 
        \inprod{z_i}{z_j}
        \qquad &\text{if } \mult{\ymset(B)}(y)\neq b
        \\
        - \frac{1}{|{\By{y}{Y}{B}}|-1}
        \sum\limits_{j \in {\By{y}{Y}{B}}\setminus\mset{i}} 
        \inprod{z_i}{z_j}
        &\text{if } \mult{\ymset(B)}(y) = b
        \end{cases}        
    \end{equation}
    so that $\frac{1}{|\By{y}{Y}{B}|}\sum_{i\in \By{y}{Y}{B}} S_i(Z;Y,B,y) = S(Z;Y,B,y)$.
\end{definition}

\subsection{Proof of \texorpdfstring{\cref{thm:supcon}}{Theorem SC}}

\begin{proof}[]
We first present the main steps of the proof of \cref{thm:supcon} and refer to subsequent technical lemmas when needed.
\begin{enumerate}[itemsep=2ex,labelindent=*,leftmargin=*,label=(\textbf{Step \arabic*}),labelwidth=\widthof{\ref{thm:supcon:proof:step-4}}]
    \item
    \label{thm:supcon:proof:step-1} 
    For each class $y\in \mcY$ and each batch ${B \in \mcB}$ with $\mult{\ymset(B)}(y)>1$, the class-specific batch-wise loss $\LSCB(Z;Y,B,y)$ (see Lemma \ref{lem:supcon_S}) is bounded from below by
    \begin{equation}
        \LSCB(Z;Y,B,y) \ge 
        \sum_{i\in \By{y}{Y}{B}}                
        \log 
        \left( 
            |{\By{y}{Y}{B}}| - 1 + |{\By{y}{Y}{B}}^C|
            \exp \left(S_i(Z;Y,B,y)\right)
        \right)
        \enspace.
    \end{equation}
    \item
    \label{thm:supcon:proof:step-2}  
    We regroup the addends of the sum $\LSC(Z;Y)$, \ie,
    \begin{equation}
        \LSC(Z;Y) = \sum_ {B \in \mcB} \sum_ {y \in \mcY} \LSCB(Z;Y,B,y) \enspace,
    \end{equation}
    s.t. each group is defined by addends requiring 
    $B\in \mcB_{y,l}=\set{ {B \in \mcB}|~ \mult{\ymset(B)}(y) = l}$. 
    As a result, we can leverage the bound of \ref{thm:supcon:proof:step-1} on each group, \ie, 
    \begin{align}
        \LSC(Z;Y) 
        &= \sum_{B \in \mcB} \sum_{y \in \mcY} \LSCB(Z;Y,B,y) \\ 
        & = \sum_{l=0}^b \sum_{y \in \mcY} \sum_{B \in \mcB_{y,l}} \LSCB(Z;Y,B,y)
        \\
        & \ge 
            \sum_{l=2}^{b} \sum_{y \in \mcY} \sum_{B \in \mcB_{y,l}} 
            \sum_{i\in \By{y}{Y}{B}} 
            l                
            \log 
            \left( 
                l - 1 + (b-l)
                \exp \left( S_i(Z;Y,B,y)
                \right)
            \right)
            \enspace.
    \end{align}
    Here, the sum over the value of $l$ starts at $l=2$, because $\LSCB(Z;Y,B,y)=0$ vanishes for batches $B\in \set{\mcB_{y,0},\mcB_{y,1}}$.
    \item
    \label{thm:supcon:proof:step-3} 
    Applying Jensen's inequality (see \cref{lem:supcon_jensen}), then yields
    \begin{equation}
        \LSC(Z;Y)
        \ge
        \sum_{l = 2}^{b} l M_l                           
            \log 
            \bigg( 
                l - 1 + (b-l)
                \exp \bigg( 
                    \frac {1} {M_l}
                    \sum_{y \in \mcY} 
                    \sum_{B \in \mcB_{y,l}} 
                    S(Z;Y,B,y)
                \bigg)
            \bigg)
        \enspace,
    \end{equation}
    where 
    $M_l = \sum_ {y \in \mcY} |\mcB_{y,l}|$.

    \item 
    \label{thm:supcon:proof:step-4}   
    In \cref{lem:supcon_outer_bound}, we characterize the equality case of the bound above. 
    It is achieved
    if and only if all intra-class and inter-class inner products agree, respectively.
\end{enumerate}

\noindent The next steps investigate, for each $l\in \set{2,\dots,b-1}$, the sum 
\begin{equation}
    \sum_{y \in \mcY} \sum_ {B \in \mcB_{y,l}}
    S(Z;Y,B,y)
    = 
    \bigg(
        \sum_{y \in \mcY} \sum_ {B \in \mcB_{y,l}}
        S_{\text{att}}(Z;Y,B,y)
    \bigg)
    +
    \bigg(
        \sum_{y \in \mcY} \sum_ {B \in \mcB_{y,l}}
        S_{\text{rep}}(Z;Y,B,y)
    \bigg)
\end{equation}

\begin{enumerate}[itemsep=2ex,labelindent=*,leftmargin=*,label=(\textbf{Step \arabic*}),labelwidth=\widthof{\ref{thm:supcon:proof:step-8}}]
    \setcounter{enumi}{4}
    \item The sum of the attraction terms, $\SSCB_{\text{att}}$, is maximal if and only if all intra-class inner products are maximal, i.e., they are equal to ${\rho_{\mathcal Z}}^2$ (see \cref{lem:supcon_att}). This implies
    \begin{equation}
            \sum_{y \in \mcY} \sum_{B \in \mcB_{y,l}}S_{\text{att}}(Z;Y,B,y) 
            \ge 
            -M_l \, {\rho_{\mcZ}}^2
            \enspace.                
    \end{equation}
    \item If $|\mcY|>2$, then the trivial bound on the repulsion term (inner products $= - {\rho_{\mathcal Z}}^2$) is not tight as this could only be achieved if the classes were on opposite poles of the sphere. 
    Thus we need an additional step: instead of summing the repulsion terms, $S_{\text{rep}}(Z;Y,B,y)$, over all labels and batches, as done in \ref{thm:supcon:proof:step-4}, we rewrite this summation as a sum over pairs of indices $(n,m) \in [N]^2$ of different classes $y_n\neq y_m$ (see \cref{lem:supcon_batches_to_indices}). That is, 
    \begin{equation}
        \sum_{y \in \mcY} \sum_{B \in \mcB_{y,l}}S_{\text{rep}}(Z;Y,B,y)
        =        
        \sum_ {y \in \mcY}
        \sum_ {\substack{n \in [N] \\ y_n = y}}
        \sum_ {\substack{m \in [N] \\ y_m \neq y}}
        K_{n,m}(y,l)
        \inprod{z_n}{z_m}
        \enspace,
    \end{equation}
    where 
    $K_{n,m}(y,l) := 
    \frac{1}{l(b-l)} 
    \sum_{B \in \mcB_{y,l}}
    \mult{ B_y }(n)
    \mult{ {B_y}^C }(m)$\enspace.
    \item As we assume the label configuration $Y$ to be balanced, we get that
    \begin{equation}
        K_{n,m}(y,l) = \frac{M_l}{N^2} \frac{|\mcY|}{|\mcY|-1} \enspace,
    \end{equation}
    which only depends on $l$ (and not on $y$), see \cref{lem:supcon_combinatorics} and \cref{eq:supcon:K_balanced}. Thus, it suffices to bound 
    \begin{equation}
        \sum_{y\in \mcY}
        \sum_{\substack{n\in [N]\\y_n = y}}
        \sum_{\substack{m\in [N]\\y_n \neq y}}
        \inprod{z_n}{z_m}
        \ge
        - {\rho_{\mathcal Z}}^2 
            \sum_{y\in \mcY}
            \mult{Y}(y)^2 
        = -  \frac{N^2}{|\mcY|} {\rho_{\mathcal Z}}^2
        \enspace,
    \end{equation}
    where equality is attained if and only if (a) $\sum_ {n \in [N]} z_n = 0$, and  (b) $y_n=y_m$ $\Rightarrow$ $z_n = z_m$ (see \cref{lem:supcon_rep_sum}).
    
    \item\label{thm:supcon:proof:step-8}  Finally, we combine all results from (Steps 1-7) and obtain the asserted lower bound 
    (see \cref{lem:supcon_final}), i.e.,
    \begin{align}
        \LSC(Z;Y) 
        \!&\ge
        \sum_{l = 2}^{b} M_l
        l \log 
        \left( 
            l - 1 + (b-l)
            \exp \left(\!
                - \frac{\rho_{\mcZ}^2}{M_l}  
                \left(
                    \frac{M_l}{N^2} \frac{|\mcY|}{|\mcY|-1}
                     \frac{N^2}{|\mcY|}
                    + M_l
                \right)
            \right)
        \right)\\
        &=
        \sum_{l = 2}^{b} M_l
        l \log 
        \left( 
            l - 1 + (b-l)
            \exp \left(- \frac{|\mcY|}{|\mcY|-1}
                {\rho_{\mcZ}}^2                
            \right)
        \right)
        \enspace,
    \end{align}
    where equality is attained if and only if all instances $z_n$ with equal label $y_n$ collapse to a vertex $\zeta_{y_n}$ of a regular simplex, inscribed in a hypersphere of radius ${\rho_{\mcZ}}$, i.e., conditions \ref{con:supcon:1} and \ref{con:supcon:2} in \cref{thm:supcon}.
\end{enumerate}
\end{proof}

\clearpage
\subsection{Technical lemmas}

In the following, we provide proofs for all technical lemmas invoked throughout Steps 1-8 in the proof of \cref{thm:supcon}.

\begin{lemma}
    \label{lem:supcon_S}
    Fix a class ${y \in \mcY}$ and a batch ${B \in \mcB}$ with $\mult{\ymset(B)}(y) \in \set{2,\dots,b}$. For every $Z\in \mcZ^N$ and every $Y\in \mcY^N$, the class-specific batch-wise loss $\LSCB(Z; Y,B,y)$ is bounded from below by
    \begin{equation}
        \LSCB(Z;Y,B,y) \ge 
        \sum_{i\in \By{y}{Y}{B}}                
        \log 
        \left( 
            |{\By{y}{Y}{B}}| - 1 + |{\By{y}{Y}{B}}^C|
            \exp \left(S_i(Z;Y,B,y)\right)
        \right)
    \enspace,
    \end{equation}    
    where equality is attained if and only if all of the following hold:
    \begin{enumerate}[label={(Q\arabic*)},labelindent=10pt,leftmargin=!,labelwidth=\widthof{\ref{con:supcon_batch_class:rep}}]
        \item \label{con:supcon_batch_class:att}
        There is $C_i(B,y)\in \mathbb R$ such that $\inprod{z_i}{z_j}= C_i(B,y)$ for every $j\in {\By{y}{Y}{B}} \setminus \mset{i}$
        \item \label{con:supcon_batch_class:rep}
        There is $D_i(B,y)\in \mathbb R$ such that $\inprod{z_i}{z_j}= D_i(B,y)$ for every $j\in {\By{y}{Y}{B}}^C$.
    \end{enumerate}
\end{lemma}
\begin{proof}
    We first bring the class-specific batch-wise loss in a form amenable to Jensen's inequality.
    Since $\mult{\ymset(B)}(y)>1$, we have that
    \begin{align}
        \LSCB(Z;Y,B,y)
        & = 
        -
        \sum\limits_{i \in B_y}
        \frac{1}{|B_{y}|-1}
        \sum\limits_{j \in B_{y}\setminus \lmset i \rmset} 
        \log
        \left(
            \frac
            {\exp\big(\dm{z_i}{z_j}\big)}
            {
                \sum\limits_{k \in B \setminus \lmset i \rmset}
                \exp\big(\dm{z_i}{z_k}\big)
            }
        \right)
        \\
    & = \label{lem:supcon_S:eqn:1}
        \sum\limits_{i \in B_y}          
        \log
        \left(
            \frac            
            {
                \sum\limits_{k \in B \setminus \lmset i \rmset}
                \exp\big(\dm{z_i}{z_k}\big)
            }
            {\exp\big(\frac{1}{|B_y\setminus\mset{i}|}\sum\limits_{j \in B_{y}\setminus \mset i}\dm{z_i}{z_j}\big)}
        \right)
        \enspace.
    \end{align}
    First, assume that $|{\By{y}{Y}{B}}^C|>0$. Then, for each $i\in B_{y}$, Jensen's inequality implies
    \begin{align}
        \label{lem:supcon_S:eqn:2}
        &\sum\limits_{k \in B \setminus \lmset i \rmset}
                \exp\big(\dm{z_i}{z_k}\big)\\
        &=
        \underbrace{\sum\limits_{k \in B_{y} \setminus \lmset i \rmset}
                \exp\big(\dm{z_i}{z_k}\big)}_{
                    \stackrel{\ref{con:supcon_batch_class:att}}\ge 
        |{ \By{y}{Y}{B}}\setminus \mset{i}|
        \exp \left(
            \frac{1}{|{ \By{y}{Y}{B}}\setminus \mset{i}|}
            \sum_{k \in { \By{y}{Y}{B}} \setminus \mset{i}}\inprod{z_i}{z_k}
        \right)
                }
        +
        \underbrace{\sum\limits_{k \in {B_{y}}^C \setminus \lmset i \rmset}
                \exp\big(\dm{z_i}{z_k}\big)}_{
                    \stackrel{\ref{con:supcon_batch_class:rep}}\ge
        |{\By{y}{Y}{B}}^C|
        \exp \left(
            \frac{1}{|{\By{y}{Y}{B}}^C|}
            \sum_{k \in { \By{y}{Y}{B}}^C }\inprod{z_i}{z_k}
        \right)
                }
        .
    \end{align}
    Consequently,
    \begin{align}
        &\LSCB(Z;Y,B,y)\\
        &\ge
        \sum_{i\in \By{y}{Y}{B}}                
        \log 
        \left( 
            |{ \By{y}{Y}{B}}\setminus \mset{i}| + |{\By{y}{Y}{B}}^C|\,
            \frac
            {\exp \bigg(
                \frac{1}{|{\By{y}{Y}{B}}^C|}
                \sum_{k \in { \By{y}{Y}{B}}^C }\inprod{z_i}{z_k}
                \bigg)
            }
            {\exp\bigg(\frac{1}{|B_y\setminus\mset{i}|}\sum\limits_{j \in B_{y_i}\setminus \lmset i \rmset}\dm{z_i}{z_j}\bigg)}
        \right)
        \\
        &=
        {\sum_{i\in \By{y}{Y}{B}}                
        \log 
        \left( 
            |{ \By{y}{Y}{B}}\setminus \mset{i}| + |{\By{y}{Y}{B}}^C|
            \exp\left(
                \frac{1}{|{\By{y}{Y}{B}}^C|}
                \sum_{k \in { \By{y}{Y}{B}}^C }\inprod{z_i}{z_k}
            -
            \frac{1}{|B_y\setminus\mset{i}|}\sum\limits_{j \in B_{y_i}\setminus \lmset i \rmset}\dm{z_i}{z_j}
            \right)\!
        \right)}
        \\
        &=
        \sum_{i\in \By{y}{Y}{B}}                
        \log 
        \left( 
            |{ \By{y}{Y}{B}}| - 1 + |{\By{y}{Y}{B}}^C|
            \exp\left(
                \SSCB_i(Z;Y,B,y)
            \right)
        \right)
        \enspace.
    \end{align}
    Herein, equality is attained if and only if
    \begin{enumerate}[label={(Q\arabic*)},labelindent=10pt,leftmargin=!,labelwidth=\widthof{\ref{con:supcon_batch_class:rep}}]
        \item[\ref{con:supcon_batch_class:att}] 
            There is $C_i(B,y)\in \mathbb R$ such that $\inprod{z_i}{z_j}= C_i(B,y)$ for every $j\in {\By{y}{Y}{B}} \setminus \mset{i}$,
        \item[\ref{con:supcon_batch_class:rep}] 
            There is $D_i(B,y)\in \mathbb R$ such that $\inprod{z_i}{z_j}= D_i(B,y)$ for every $j\in {\By{y}{Y}{B}}^C$.
    \end{enumerate}

    Last, assume that ${|\By{y}{Y}{B}}^C|=0$, i.e., $\By{y}{Y}{B}=B$ and ${\By{y}{Y}{B}}^C= \emptyset$.
    Then, the second summand of \cref{lem:supcon_S:eqn:2} vanishes and so
    \begin{equation}
        \LSCB(Z;Y,B) \ge b                
        \log 
        \left( 
            b - 1 
        \right)
    \enspace,
    \end{equation}
    with equality condition~\ref{con:supcon_batch_class:att}. (the condition \ref{con:supcon_batch_class:rep} is trivially fulfilled)
\end{proof}
\begin{lemma}
    \label{lem:supcon_jensen}
    Let $l\in \set{2,\dots,b}$. For every $Y \in \mcY^N$ and every $Z\in \mcZ^N$, we have that
    \begin{align}
        \nonumber
        &\frac{1}{l M_l}
        \sum_{y \in \mcY} \sum_ {\substack{ B \in \mcB_{y,l}}}     
        \sum_{i \in\By{y}{Y}{B}}               
        \log 
        \left( 
            l - 1 + (b-l)
            \exp \left( \SSCB_i(Z;Y,B,y)
            \right)
        \right)
        \\
        \ge 
        & \log 
        \left( 
            l - 1 + (b-l)
            \exp \left( 
                \frac{1}{M_l}
                \sum_{y \in \mcY} \sum_ {B \in \mcB_{y,l}}
                S(Z;Y,B,y)
            \right)
        \right)  
        \label{eq:supcon_jensen}
        \enspace,
    \end{align}
    where $M_l =\sum_ {y \in \mcY} |\mcB_{y,l}|$ and equality is attained if and only if:
    \begin{enumerate}[label={(Q\arabic*)},labelindent=10pt,leftmargin=!,labelwidth=\widthof{\ref{con:supcon_jensen}}]
        \setcounter{enumi}{2}
        \item \label{con:supcon_jensen}
            $l=b$ or there is a constant $D(l)\in \mathbb R$ such that $S_i(Z;Y,B,y) = D(l)$ for every triple $(y,B,i)$ with $y \in \mcY$, $B \in \mcB_{y,l}$, and $i\in B_y$.
    \end{enumerate}
\end{lemma}
\begin{proof}
    If $l=b$, then \cref{eq:supcon_jensen} holds with equality. Assume $l<b$.
    Let $\alpha,\beta>0$ and consider the convex function    
    $f:\bbR \to \bbR$, $x\mapsto \log(\alpha +\beta \exp(x))$.
    By Jensen' inequality, for every finite sequence $(x_{B,y,i})_{ ({y \in \mcY},{B \in \mcB_{y,l}, i\in B})}$ 
    of length
    \begin{align}
        \lvert \set{(y,B,i)|~ {y \in \mcY},{B \in \mcB_{y,l}, i\in B_y}} \rvert
        = \sum_{y\in\mcY} \sum_{B\in \mcB_{y,l}} |B_y| 
        = \sum_{y\in\mcY} \sum_{B\in \mcB_{y,l}} l
        = l \sum_{y\in\mcY}\lvert\mcB_{y,l}\rvert 
    \end{align}
    it holds that
    \begin{equation}
        \frac{1}{l \sum_{y\in\mcY}|\mcB_{y,l}|}
        \sum_{y \in \mcY} \sum_ {\substack{ B \in \mcB_{y,l}}} 
        \sum_{i\in B_y}
        f(x_{B,y,i})         
        \stackrel{\ref{con:supcon_jensen}}{\ge}
        f \left(
            \frac{1}{l \sum_{{y \in \mcY}}|\mcB_{y,l}|}
            \sum_{y \in \mcY} \sum_ {\substack{ B \in \mcB_{y,l}}} 
            \sum_{i\in B_y}
            x_{B,y,i}
        \right)
        \enspace.
    \end{equation} 
    Now, choose $\alpha = l-1$, $\beta = b-l$ and $x_{B,y,i} = S_i(Z;Y,B,y)$. Next, recall that $M_l= \sum_{y\in\mcY}|\mcB_{y,l}|$ and that for $B\in \mcB_{y,l}$ it holds that $|\By{y}{Y}{B}| = l$. Last, insert  $ S(Z;Y,B,y) = \frac{1}{|\By{y}{Y}{B}|}\sum_{i\in \By{y}{Y}{B}} S_i(Z;Y,B,y)$. Together, this implies the lower bound \cref{eq:supcon_jensen} in which equality is attained if and only if:
    \begin{enumerate}[label={(Q\arabic*)},labelindent=10pt,leftmargin=!,labelwidth=\widthof{\ref{con:supcon_jensen}}]
        \item[\ref{con:supcon_jensen}]
        There is a constant $D(l)\in \mathbb R$ such that $S_i(Z;Y,B,y) = D(l)$ for every triple $(y,B,i)$ with $y \in \mcY$, $B \in \mcB_{y,l}$, and $i\in B_y$.
    \end{enumerate}
\end{proof}
In the next step, we combine \cref{lem:supcon_S} with \cref{lem:supcon_jensen}, which implies a bound with more tangible equality conditions.
\begin{lemma}
    \label{lem:supcon_outer_bound}
    For every $Y\in \mcY^N$ and every $Z\in \mcZ^N$
    the supervised contrastive loss $\LSC$ is bounded from below by
    \begin{equation}
        \label{eq:supcon_outer_bound}
        \LSC(Z;Y)
        \ge
        \sum_{l = 2}^{b} l \, M_l                           
            \log 
            \left( 
                l - 1 + (b-l)
                \exp \left( 
                    \frac {1} {M_l}
                    \sum_{y \in \mcY} \sum_ {B \in \mcB_{y,l}}
                    S(Z;Y,B,y)
                \right)
            \right)
        \enspace,
    \end{equation}
    where $M_l$ is defined as in \cref{lem:supcon_jensen} and equality is attained if and only if:
    \begin{enumerate}[label={(A\arabic*)},labelindent=10pt,leftmargin=!,labelwidth=\widthof{\ref{con:supcon_outer_bound:inter}}]
        \item \label{con:supcon_outer_bound:intra}
            There exists a constant $\alpha$, such that $\forall n,m\in [N]$, $y_n = y_m$ implies $\inprod{z_n}{z_m} = \alpha$\enspace.
        \item \label{con:supcon_outer_bound:inter}
        There exists a constant $\beta$, such that $\forall n,m\in [N]$, $y_n \neq y_m$ implies $\inprod{z_n}{z_m} = \beta$\enspace.
    \end{enumerate}
\end{lemma}
\begin{proof}
    First, observe that $\LSCB(Z;Y,B,y)= 0$ if $B\in \set{\mcB_{y,0}, \mcB_{y,1}}$. Leveraging \cref{lem:supcon_S} and \cref{lem:supcon_jensen}, we get
    \begin{align}
        \LSC(Z;Y) 
        &= \sum_ {B \in \mcB} \sum_ {y \in \mcY} \LSCB(Z;Y,B,y) \\
        & = \sum_{l =2}^{b} \sum_ {y \in \mcY} \sum_{B\in \mcB_{y,l}} \LSCB(Z;Y,B,y) 
        \\
        & \stackrel{\text{Lem.~\ref{lem:supcon_S}}}{\ge}
        \sum_{l =2}^{b} \sum_{y \in \mcY} \sum_{\mcB_{y,l}} 
        \sum_{i\in B_y}              
        \log 
        \left( 
            l - 1 + (b-l)
            \exp \left( S_i(Z;Y,B,y)
            \right)
        \right)
        \\
        &
        \label{lem:supcon_outer_bound:eq1} 
        \stackrel{\text{Lem.~\ref{lem:supcon_jensen}}}{\ge}
        \sum_{l =2}^{b}l \, M_l                           
        \log 
        \left( 
            l - 1 + (b-l)
            \exp \left( 
                \frac {1} {M_l}
                \sum_{y \in \mcY} \sum_ {B \in \mcB_{y,l}}
                S(Z;Y,B,y)
            \right)
        \right)
        \enspace.
    \end{align}
    Here equality is attained if and only if all the following conditions hold:
    \begin{enumerate}[labelindent=10pt,leftmargin=!,labelwidth=\widthof{\ref{con:supcon_jensen}}]
        \item[\ref{con:supcon_batch_class:att}]
            $\forall l \in \set{2,\dots,b}$, $\forall y \in \mcY$, $\forall B\in \mcB_{y,l}$ and $\forall i \in B$ there is a $C_i(B,y)$ such that $\forall j\in {\By{y}{Y}{B}} \setminus \mset{i}$ all inner products $\inprod{z_i}{z_j}= C_i(B,y)$ are equal. 
        \item[\ref{con:supcon_batch_class:rep}]
            $\forall l \in \set{2,\dots,b}$, $\forall y \in \mcY$, $\forall B\in \mcB_{y,l}$ and $\forall i \in B$ there is a $D_i(B,y)$ such that $\forall j\in {\By{y}{Y}{B}}^C$ all inner products $\inprod{z_i}{z_j}= D_i(B,y)$ are equal.
        \item[\ref{con:supcon_jensen}] 
            $\forall l \in \set{2,\dots,b-1}$, there is a constant $D(l)\in \mathbb R$ such that $S_i(Z;Y,B,y) = D(l)$ for every triple $(y,B,i)$ with $y \in \mcY$, $B \in \mcB_{y,l}$, and $i\in B_y$.
    \end{enumerate}
    It remains to show that \ref{con:supcon_batch_class:att} \& \ref{con:supcon_batch_class:rep} \& \ref{con:supcon_jensen} is equivalent to:
    \begin{enumerate}[labelindent=10pt,leftmargin=!,labelwidth=\widthof{\ref{con:supcon_outer_bound:inter}}]
        \item [\ref{con:supcon_outer_bound:intra}]
            There exists a constant $\alpha$, such that $\forall n,m\in [N]$, $y_n = y_m$ implies $\inprod{z_n}{z_m} = \alpha$\enspace.
        \item [\ref{con:supcon_outer_bound:inter}]
        There exists a constant $\beta$, such that $\forall n,m\in [N]$, $y_n \neq y_m$ implies $\inprod{z_n}{z_m} = \beta$\enspace.
    \end{enumerate}
    Recall for $B\in \mcB_{y,l}$ with $l<b$ the definition of the auxiliary function $S_i$, i.e
    \begin{equation}
        S_{i}(Z;Y,B,y) = 
            - \frac{1}{|{\By{y}{Y}{B}}|-1}
        \sum_{j \in {\By{y}{Y}{B}}\setminus\mset{i}} 
        \inprod{z_i}{z_j}
        +
        \frac{1}{|{\By{y}{Y}{B}}^C|}
        \sum_{j \in {\By{y}{Y}{B}}^C} 
        \inprod{z_i}{z_j}    \enspace.    
    \end{equation}
    We start with the direction \ref{con:supcon_outer_bound:intra} \& \ref{con:supcon_outer_bound:inter} $\implies$ \ref{con:supcon_batch_class:att} \& \ref{con:supcon_batch_class:rep} \& \ref{con:supcon_jensen}.
    \begin{enumerate}[labelindent=10pt,leftmargin=!,labelwidth=\widthof{\ref{con:supcon_jensen}}]
        \item[\ref{con:supcon_batch_class:att}]
            Fix $l\in\set{2,\dots,b}$, $y\in Y$, $B\in \mcB_{y,l}$ and $i\in B$.
            Let $j\in {\By{y}{Y}{B}} \setminus \mset{i}$, i.e., $y_j =y = y_i$. Therefore condition \ref{con:supcon_outer_bound:intra} implies $\inprod{z_i}{z_j} = \alpha$, i.e., condition \ref{con:supcon_batch_class:att} is fulfilled with $C_i(B,y) = \alpha$. 
        \item[\ref{con:supcon_batch_class:rep}]
            Fix $l\in\set{2,\dots,b}$, $y\in Y$, $B\in \mcB_{y,l}$ and $i\in B$.
            Let $j\in {\By{y}{Y}{B}}^C$, i.e., $y_j \neq y = y_i$. Therefore condition \ref{con:supcon_outer_bound:inter} implies $\inprod{z_i}{z_j} = \beta$, i.e. condition \ref{con:supcon_batch_class:rep} is fulfilled with $D_i(B,y) = \beta$. 
        \item[\ref{con:supcon_jensen}]
            Fix $l\in \set{2,\dots,b-1}$. Let $y\in \mcY$, $B\in \mcB_{y,l}$ and $i\in B$. By condition \ref{con:supcon_outer_bound:intra}, the first summand of 
            $S_i(Z;Y,B,y)$ is $-\alpha$ and by condition \ref{con:supcon_outer_bound:inter}, the second summand is
            $\beta$, so $S_i(Z;Y,B,y) = \beta - \alpha$ and condition \ref{con:supcon_jensen} is fulfilled with $D(l)=\beta-\alpha$.
    \end{enumerate}
    Next, we show \ref{con:supcon_batch_class:att} \& \ref{con:supcon_batch_class:rep} \& \ref{con:supcon_jensen} $\implies$ \ref{con:supcon_outer_bound:intra} \& \ref{con:supcon_outer_bound:inter}.

    Let $y,y'\in \mcY$ and $n,m,n',m'\in [N]$ with $y_n=y_m = y$ and $y_{n'} = y_{m'} = y'$.
    For brevity, we write multisets such that the multiplicity of each element is denoted as a superscript, e.g., $\lmset n^b \rmset$ denotes the multiset $\lmset n, \dots, n \rmset$ which contains $n$ exactly $b$ times.
    Recall, that we assume that $b\ge 3$. 
    \begin{enumerate}[labelindent=10pt,leftmargin=!,labelwidth=\widthof{\ref{con:supcon_outer_bound:inter}}]
        \item [\ref{con:supcon_outer_bound:intra}]
        We need to show that $\inprod{z_n}{z_m} = \inprod{z_{n'}}{z_{m'}}$.
        There are two cases: $y = y'$ and $y\neq y'$.

        First, assume $y \neq y'$.\\
        Choose $l=2$, and consider the batch $B_1 = \mset{n,m,(n')^{b-2}}\in \mcB_{y,2}$.
        Then 
        $$\SSCB_n(Z; Y, B_1,y) = - \inprod{z_n}{z_m} +  \inprod{z_n}{z_{n'}}\enspace.$$
        Similarly, for the batch $B_2= \mset{n',m',n^{b-2}}\in \mcB_{y',2}$ we get $\SSCB_{n'}(Z; Y, B_2,y') = - \inprod{z_{n'}}{z_{m'}} + \inprod{z_n'}{z_{n}}$.
        Since condition~\ref{con:supcon_jensen} implies that $\SSCB_n(Z; Y, B_1,y) = \SSCB_{n'}(Z; Y, B_2,y')$, we conclude that $\inprod{z_{n}}{z_{m}} = \inprod{z_{n'}}{z_{m'}}$.

        \vskip2ex
        Now, assume $y=y'$.\\
        Let $p\in [N]$ such that $y_p \neq y$.
        Again, choose $l=2$ and consider the batches $B_1 = \mset{n,m,p^{b-2}}\in \mcB_{y,2}$ and $B_2 =\mset{n',m',p^{b-2}}\in \mcB_{y,2}$. By the same argument as in the preceding case of $y\neq y'$, we have that
        $\SSCB_n(Z; Y, B_1,y) = - \inprod{z_n}{z_m} + \inprod{z_n}{z_{p}}$
        and $\SSCB_{n'}(Z; Y, B_2,y) = - \inprod{z_{n'}}{z_{m'}} + \inprod{z_{n'}}{z_{p}}$.
        Therefore, condition~\ref{con:supcon_jensen} implies that 
        $$- \inprod{z_n}{z_m} + \inprod{z_n}{z_{p}} = - \inprod{z_{n'}}{z_{m'}} + \inprod{z_{n'}}{z_{p}} \enspace.$$
        Now, pick the batch $B_3 = \mset{z_n,z_{n'},p^{b-2}}$. From condition~\ref{con:supcon_batch_class:rep} it follows that $\inprod{z_n}{p} = \inprod{z_{n'}}{p}$ and thus $\inprod{z_{n'}}{z_{m'}} = \inprod{z_{n}}{z_{m}}$.

        \item [\ref{con:supcon_outer_bound:inter}]
        We need to show that $\inprod{z_n}{z_{n'}} = \inprod{z_{m}}{z_{m'}}$ if $y\neq y'$.
        
        Choose $l=2$, and consider the two batches $B_1 = \mset{n,m,(n')^{b-2}} \in \mcB_{y,2}$ and $B_2 = \mset{n,m,(m')^{b-2}} \in \mcB_{y,2}$.
        By condition \ref{con:supcon_jensen}, 
        \begin{equation}
            \begin{split}
           - \inprod{z_n}{z_m} + \inprod{z_n}{z_{n'}} =  S_n(Z,Y,B_1,y) & =  S_m(Z,Y,B_2,y) \\
           & = - \inprod{z_n}{z_m} + \inprod{z_m}{z_{m'}} \enspace.
            \end{split}
        \end{equation}
        This immediately implies $\inprod{z_n}{z_{n'}} = \inprod{z_{m}}{z_{m'}}$.
    \end{enumerate}
\end{proof}
In the following, we address the two parts of the sum in the exponent in \cref{eq:supcon_outer_bound}, i.e.,
\begin{equation}
    \sum_{y \in \mcY} \sum_ {B \in \mcB_{y,l}}\!\!
    S(Z;Y,B,y)
    = 
    \underbrace{
    \left(
        \sum_{y \in \mcY} \sum_ {B \in \mcB_{y,l}}
        S_{\text{att}}(Z;Y,B,y)
    \right)\!
    }_{\text{Lem.~}\ref{lem:supcon_att}}
    +\!
    \underbrace{
    \left(
        \sum_{y \in \mcY} \sum_ {B \in \mcB_{y,l}}
        S_{\text{rep}}(Z;Y,B,y)\!
    \right)
    }_{\text{Lem.~}\ref{lem:supcon_rep}}
\end{equation}
While the first summand is handled easily via \cref{lem:supcon_att}, the second summand requires further considerations, encapsulated in Lemmas~\ref{lem:supcon_batches_to_indices}, 
\ref{lem:combinatorics}, \ref{lem:supcon_combinatorics}, \ref{lem:supcon_rep_sum} and finally combined to \cref{lem:supcon_rep}.
\begin{lemma}[Sum of attraction terms]
    \label{lem:supcon_att}
    Let $l\in \set{2,\dots,b}$ and let $\mcZ = \bbS_{\rho_{\mathcal Z}}$. 
    For every $Y\in \mcY^N$ and every $Z\in \mcZ^N$, it holds that 
    \begin{equation}
        \label{lem:supcon_att:eq1}
        \sum_{y \in \mcY}\sum_{B \in \mcB_{y,l}}
        S_{\text{att}}(Z;Y,B,y)
        \ge
        - \left(\sum_{y\in\mcY} |\mcB_{y,l}|\right) 
        {\rho_{\mcZ}}^2
        \enspace,
    \end{equation}
    where equality is attained if and only if:
    \begin{enumerate}[label={(Q\arabic*)}, start=4,labelindent=10pt,leftmargin=!,labelwidth=\widthof{\ref{con:supcon_att}}]
        \item \label{con:supcon_att}
            For every $n,m\in [N]$, $y_n = y_m$ implies $z_n = z_m$\enspace.
    \end{enumerate}
\end{lemma}
\begin{proof}
    Recall the definition of $S_{\text{att}}(Z;Y,B,y)$ from \cref{eq:def_s_att}:
    \begin{equation}
        S_{\text{att}}(Z;Y,B,y) 
        = -
        \frac{1}{|{\By{y}{Y}{B}}|\,|{\By{y}{Y}{B}}\setminus \mset{i}|}
        \sum_{i \in {\By{y}{Y}{B}}}
        \sum_{j \in {\By{y}{Y}{B}}\setminus\mset{i}} 
        \inprod{z_i}{z_j}
        \enspace.
    \end{equation}
    Using the Cauchy-Schwarz inequality and the assumption that $\mcZ$ is a hypersphere of radius $\rho_{\mathcal Z}$, $S_{\text{att}}(Z;Y,B,y)$ is bounded from below by 
    \begin{equation}
        S_{\text{att}}(Z;Y,B,y) 
        \stackrel{\ref{con:supcon_att}}\ge  -
        \frac{1}{|{\By{y}{Y}{B}}|\,|{\By{y}{Y}{B}}\setminus \mset{i}|}
        \sum_{i \in {\By{y}{Y}{B}}}
        \sum_{j \in {\By{y}{Y}{B}}\setminus\mset{i}} 
        \norm{z_i} \norm{z_j}
        = 
        - {\rho_{\mcZ}}^2
        \enspace,
    \end{equation}
    which already implies the lower bound \cref{lem:supcon_att:eq1}.

    For fixed $l\in\set{2,\dots,b}$, equality is attained if and only if there is equality in the Cauchy-Schwarz inequality. 
    This means, that for every $y\in\mcY$, for every $B \in \mcB_{y,l}$ and for every $i,j \in \By{y}{Y}{B}$ there exists $\lambda\ge0$, such that $z_i = \lambda z_j$.
    Since the $z_i$ and $z_j$ are on a hypersphere, this is equivalent to $z_i = z_j$.
    Furthermore,
    for each pair of indices $n,m \in [N]$ with equal class $y_n = y_m =y$, there exists a batch $B \in \mcB_{y,l}$ containing both indices. Hence the equality condition is equivalent to 
    \begin{enumerate}[labelindent=10pt,leftmargin=!,labelwidth=\widthof{\ref{con:supcon_att}}]
        \item [\ref{con:supcon_att}] For every $n,m\in [N]$, $y_n = y_m$ implies $z_n =z_m$.
    \end{enumerate}
\end{proof}
Next, we consider the repulsion component. 
Recall the definition of $S_{\text{rep}}(Z;Y,B,y)$ from \cref{eq:def_s_rep}. We want to bound
\begin{equation}
    \sum_{y \in \mcY}\sum_{B \in \mcB_{y,l}}
    \SSCB_{\text{rep}}(Z;Y,B,y)
    = 
    \sum_{y \in \mcY}\sum_{B \in \mcB_{y,l}}
    \frac{1}{|{\By{y}{Y}{B}}|\,|{\By{y}{Y}{B}}^C|}
    \sum_{i \in {\By{y}{Y}{B}}}
    \sum_{j \in {\By{y}{Y}{B}}^C} 
    \inprod{z_i}{z_j}
    \enspace.
\end{equation}
Similarly to the proof of \cref{lem:supcon_att}, we could bound each inner product by $\inprod{z_i}{z_j}\ge -{\rho_{\mcZ}}^2$. However, the infered inequality will not be tight and thus useless for identifying the minimizer of the sum. This is because equality would be attained if and only if all points $z_n,z_m\in \mcZ$ of different class $y_n \neq y_m$ were on opposite poles of the sphere. Yet, this is impossible for $|\mcY|>2$, i.e., if there are more than two classes.
In that case, the argument is more complex, and we split it in a sequence of lemmas.

\begin{lemma}
    \label{lem:supcon_batches_to_indices}
    Let $l \in \set{2,\dots,b-1}$ and let $y \in \mcY$.
    For every $Y\in\mcY^N$ and every $Z\in \mcZ^N$ the following identity holds:
    \begin{equation}
        \sum_{B \in \mcB_{y,l}}
        S_{\text{rep}}(Z;Y,B,y)
        = 
            \sum_{\substack{n\in [N]\\ y_n = y}} 
            \sum_{\substack{m\in [N] \\ y_m \neq y}}
            K_{n,m}(y,l) \,
            \inprod{z_n}{z_m}
        \enspace,
    \end{equation}
    where for each $n,m\in [N]$ with $y_n = y$ and $y_m \neq y$ the combinatorial factor $K_{n,m}(y,l)$ is defined by
    \begin{equation}
        K_{n,m}(y,l) = 
        \frac{1}{l(b-l)} 
        \sum_{B \in \mcB_{y,l}}
        \mult{ B_y }(n)
        \mult{ {B_y}^C }(m)
        \enspace.
    \end{equation}
\end{lemma}

\begin{proof}
    The lemma follows from appropriately partitioning the sum:
    \begin{align}
        \sum_{B \in \mcB_{y,l}}
        S_{\text{rep}}(Z;Y,B,y)
        & =
            \sum_{B \in \mcB_{y,l}}
            \frac{1}{|{\By{y}{Y}{B}}|\,|{\By{y}{Y}{B}}^C|}
            \sum_{i \in {\By{y}{Y}{B}}}
            \sum_{j \in {\By{y}{Y}{B}}^C} 
            \inprod{z_i}{z_j}
            \\
        & = 
            \sum_{n\in [N]} \sum_{m\in [N]}
            \frac{1}{l(b-l)}
            \sum_{B \in \mcB_{y,l}}
            \sum_{\substack{i \in {\By{y}{Y}{B}} \\ i = n}}
            \sum_{\substack{j \in {\By{y}{Y}{B}}^C \\ j = m}}
            \inprod{z_i}{z_j}
            \\
        & = 
            \sum_{\substack{n\in [N]\\ y_n = y}} 
            \sum_{\substack{m\in [N] \\ y_m \neq y}}
            \inprod{z_n}{z_m}
            \frac{1}{l(b-l)}
            \sum_{B \in \mcB_{y,l}}
            \bigg(\sum_{\substack{i \in {\By{y}{Y}{B}} \\ i = n}} 1 \bigg)
            \bigg(\sum_{\substack{j \in {\By{y}{Y}{B}}^C \\ j = m}} 1 \bigg)
            \\
        & = 
            \sum_{\substack{n\in [N]\\ y_n = y}} 
            \sum_{\substack{m\in [N] \\ y_m \neq y}}
            \inprod{z_n}{z_m}
            \frac{1}{l(b-l)}
            \sum_{B \in \mcB_{y,l}}
            \mult{ \By{y}{Y}{B} }(n)
            \mult{ \By{y}{Y}{B}^C }(m)
            \\
        & = 
            \sum_{\substack{n\in [N]\\ y_n = y}} 
            \sum_{\substack{m\in [N] \\ y_m \neq y}}
            \inprod{z_n}{z_m} \,
            K_{n,m}(y,l)
        \enspace.
    \end{align}
\end{proof}
In order to address the quantities $K_{n,m}(y, l)$, we will need the combinatorial identities of the subsequent \cref{lem:combinatorics}.
\begin{lemma}
    \label{lem:combinatorics}
    Let $n,m \in \bbN$.
    \begin{enumerate}[label={(\arabic*)}]
        \item The number of $m$-multisets over $[n]$ is
        \begin{equation} \label{eq:mset_card}
            \msetch{n}{m} = \binom{n + m -1}{m}
            \enspace.
        \end{equation}
        \item
        \begin{equation}
            \sum_{k=0}^m \msetch{n}{k} = \msetch{n+1}{m}
            \label{eq:comb_sum}
        \end{equation}
        \item 
        \begin{equation}
            \msetch{n+1}{m} = \msetch{n}{m} + \msetch{n+1}{m-1}
            \label{eq:mset_rec_rel}
        \end{equation}
        \item Let $m\ge 1$, then
        \begin{equation} \label{eq:mset_sum}
            \sum_{k\in [m]} k \msetch{n-1}{m-k} = \msetch{n+1}{m-1}=  \frac m n \msetch{n}{m}
        \end{equation}
    \end{enumerate}
\end{lemma}
\begin{proof}
    The first three identities are well known and imply the last one. We include their proofs for completeness.
    \begin{enumerate}[label={\underline{Ad (\arabic*):}}, labelindent=*,leftmargin=*,labelwidth=\widthof{\ref{lem:combinatorics:proof:4}}]
        \item Follows from the stars and bars representation of multisets. Every $m$-mutliset over $[n]$ is uniquely determined by the position of $m$ bars in an $n+m-1$ tuple of stars and bars. Hence, the number of such multisets equals the number of $m$-element subsets of an $n+m-1$-element set, which 
        is given by the binomial coefficient in \cref{eq:mset_card}. More precisely, the multiplicity of a number $k\in [n]$ in the multiset is encoded by the number of stars between the $(k-1)$-th and the $k$-th bar.
        For example, for $n=5$ and $k=4$ the multiset ${1,3,4,4}$ is represented by $(*||*|**|)$.
        \item Denote by $\mcP_{\mset{\,}}(n,m)$ the set of all $m$-multisets over $[n]$. Obviously,
        \begin{equation}
            \mcP_{\mset{\,}}(n+1,m) = \bigsqcup_{k = 0}^m \set{M \in  \mcP_{\mset{\,}}(n+1,m)|~ \mult{M}(n+1)=k}
            \enspace.
        \end{equation}
        Thinking of a $m$-multiset over $[n+1]$ containing the element $(n+1)$ exactly $k$-times as a $m-k$ multiset over $[n]$, we get from \cref{eq:mset_card}
        \begin{align}
            \msetch{n+1}{m} 
            &= |\mcP_{\mset{\,}}(n+1,m)| \\
            &= \sum_{k=0}^m 
            |\set{M \in  \mcP_{\mset{\,}}(n+1,m)| \mult{M}(n+1)=k}|
            \\
            & = 
            \sum_{k=0}^m \msetch{n}{m-k} 
            = \sum_{k=0}^m \msetch{n}{k}
            \enspace.
        \end{align}
        \item Follows directly from the previous argument. In particular, 
        \begin{equation}
            \begin{split}
            \msetch{n+1}{m} 
            \stackrel{\eqref{eq:comb_sum}}{=} 
            \sum_{k=0}^m \msetch{n}{k} 
            & = \sum_{k=0}^{m-1} \msetch{n}{k}  + \msetch{n}{m} \\
            & \stackrel{\eqref{eq:comb_sum}}{=}
            \msetch{n+1}{m-1} + \msetch{n}{m}
            \enspace.
            \end{split}
        \end{equation}
        \item\label{lem:combinatorics:proof:4} The second equality is obvious, once both sides are expanded to the level of factorials. 
        For the first equality, we prove by induction the equivalent formula
        \begin{equation}
            \sum_{k=0}^{m} (m-k) \msetch{n-1}{k} 
            =
            \msetch{n+1}{m-1} 
            \enspace.
            \label{eq:comb_equi}
        \end{equation}
        First, consider the case $m=1$. Then both 
        \begin{equation}
            \begin{split}
            & \sum_{k=0}^{1} (1-k) \msetch{n-1}{k} = (1-0) \msetch{n-1}{0}= 1\enspace,
            \text{and} \\
            & \msetch{n+1}{m-1} = \msetch{n+1}{0} = 1 
            \enspace.
            \end{split}
        \end{equation}
        Secondly, assume that \cref{eq:comb_equi} holds for $m$. We show that it then also holds for $m+1$, i.e. 
        \begin{equation}
            \sum_{k = 0}^{m+1} (m+1-k) \msetch{n-1}{k} = \msetch{n+1}{m} \enspace.
        \end{equation}
        The proof is a simple application of the previously derived summation identities:
        \begin{align}
            \sum_{k=0}^{m+1} (m+1-k) \msetch{n-1}{k}
            & = 
                \underbrace{\sum_{k=0}^m (m-k) \msetch{n-1}{k}}_{\eqref{eq:comb_equi}}
                + \underbrace{\sum_{k=0}^m \msetch{n-1}{k} }_{\eqref{eq:comb_sum}}
            \\
                & = \msetch{n+1}{m-1} + \msetch{n}{m} \\
            & \stackrel{\eqref{eq:mset_rec_rel}}{=} \msetch{n+1}{m}
            \enspace.
        \end{align}
    \end{enumerate}
\end{proof}

\begin{lemma}
    \label{lem:supcon_combinatorics}
    Let $l \in \set{1,\dots,b-1}$, $Y\in \mcY^N$  and $y\in \mcY$.
    For every $n,m \in [N]$, the combinatorial factor $K_{n,m}(y,l)$ has value
    \begin{equation}
        K_{n,m}(y,l) = \frac{|\mcB_{y,l}|}{\mult{Y}(y)(N - \mult{Y}(y))} 
        \enspace.
    \end{equation}
\end{lemma}
\begin{proof}
    We have
    \begin{equation}
        \begin{split}
        K_{n,m}(y,l)
        & = 
            \frac{1}{l(b-l)} 
            \sum_{B \in \mcB_{y,l}}
            \mult{ \By{y}{Y}{B} }(n)
            \mult{ \By{y}{Y}{B}^C }(m) \\
        & = 
            \frac{1}{l(b-l)} 
            \sum_{p=1}^l
            \sum_{q=1}^{b-l}
            \sum_{\substack{B \in \mcB_{y,l} \\ \mult{B_y}(n) = p \\ \mult{{B_y}^C}(m) = q}}
            \mult{ \By{y}{Y}{B} }(n)
            \mult{ \By{y}{Y}{B}^C }(m)
            \\
        & = 
            \frac{1}{l(b-l)} 
            \sum_{p=1}^l p
            \sum_{q=1}^{b-l} q
            \sum_{\substack{B \in \mcB_{y,l} \\ \mult{B_y}(n) = p \\ \mult{{B_y}^C}(m) = q}} 1
            \label{lem:supcon_combinatorics:eq1}
        \enspace.
    \end{split}
    \end{equation}
    Therefore, it is crucial to calculate the cardinality $|\set{B \in \mcB_{y,l}: \mult{B_y}(n) = p,~ \mult{{B_y}^C}(m) = q}|$.

    We can think of each batch $B\in \mcB$ satisfying the condition
    $$B \in \set{B \in \mcB_{y,l},~ \mult{B_y}(n) = p,~ \mult{{B_y}^C}(m) = q}$$
    as a disjoint union of multisets $B = C_n \sqcup C_m \sqcup C_y \sqcup C_{y^C}$,
    where
    \begin{itemize}
        \item $C_n$ is a $p$-multiset over the singleton $\set{n}$, 
        \item $C_m$ is a $q$-multiset over the singleton $\set{m}$,
        \item $C_y$ is a $(l-p)$-set over the multiset $\set{i \in [N]\setminus\set{n}|~ y_i = y}$ of cardinality $\mult{Y}(y)-1$ and
        \item $C_m$ is a $(b-l-q)$-set over the multiset $\set{j \in [N]\setminus\set{n}|~ y_j \neq y}$ of cardinality $N - \mult{Y}(y)-1$.
    \end{itemize}
    We write $\mcC_n, \mcC_m, \mcC_y$ and $\mcC_{y^C}$ for the respective sets of multisets. These sets are of cardinalities (see \cref{eq:mset_card})
    \begin{alignat*}{2}
        |\mcC_{n}|&= \msetch{1}{p} = 1, \qquad
        |\mcC_{y}|&&= \msetch{\mult{Y}(y)-1}{l-p}, \\
        |\mcC_{m}|&= \msetch{1}{q} = 1, \qquad
        |\mcC_{y^C}|&&= \msetch{N - \mult{Y}(y) - 1}{b-l-q},
    \end{alignat*}
    and so 
    \begin{equation}
    \begin{split}
        \left|\set{B \in \mcB_{y,l}: \mult{B_y}(n) = p,~ \mult{{B_y}^C}(m) = q}\right| 
        = |\mcC_{n}| \, |\mcC_{m}|\, |\mcC_{y}| \, |\mcC_{y^C}|\\
        =  \underbrace{\msetch{1}{p} \msetch{1}{q}}_{=1}
        \msetch{\mult{Y}(y)-1}{l-p}  \msetch{N - \mult{Y}(y) - 1}{b-l-q}
        \enspace.
    \end{split}
    \end{equation}
    By a similar argument,
    \begin{align}
        \left|\set{B \in \mcB_{y,l} }\right|
        &= \msetch{\mult{Y}(y)}{l} \msetch{N - \mult{Y}(y) }{b-l}
        \enspace.
        \label{eq:card_of_byl1}
    \end{align}
    Therefore, the sum from \cref{lem:supcon_combinatorics:eq1} simplifies to
    \begin{align}
        K_{n,m}(y,l)
        &=
        \frac{1}{l(b-l)} 
        \sum_{p=1}^l p
        \sum_{q=1}^{b-l} q\left|\set{B \in \mcB_{y,l},~ \mult{B_y}(n) = p,~ \mult{{B_y}^C}(m) = q}\right|\\
        & =
            \frac{1}{l(b-l)} 
            \sum_{p=1}^l p \msetch{\mult{Y}(y)-1}{l-p}
            \sum_{q=1}^{b-l} q \msetch{N - \mult{Y}(y) - 1}{b-l-q}
        \enspace.
    \end{align}
    Leveraging \cref{eq:mset_sum} and \cref{eq:card_of_byl1}, we get the claimed result
    \begin{align}
        K_{n,m}(y,l)
        & = 
            \frac{1}{l(b-l)} 
            \frac{l}{\mult{Y}(y)}
            \msetch{\mult{Y}(y)}{l}
            \frac{b-l}{N - \mult{Y}(y)}
            \msetch{N - \mult{Y}(y)}{b-l}
            \\
        & = 
            \frac{|\mcB_{y,l}|}{\mult{Y}(y)(N - \mult{Y}(y))} 
        \enspace.
    \end{align}    
\end{proof}
\begin{lemma}
    \label{lem:supcon_rep_sum}
    Let $\mcZ = \bbS_{\rho_{\mathcal Z}}$.
    For every $Z\in \mcZ^N$ and every $Y\in \mcY^N$, we have that
    \begin{equation}
        \sum_{y\in \mcY}
        \sum_{\substack{n\in [N]\\y_n = y}}
        \sum_{\substack{m\in [N]\\y_n \neq y}}
        \inprod{z_n}{z_m}
        \ge
        - {\rho_{\mathcal Z}}^2 
            \sum_{y\in \mcY}
            \mult{Y}(y)^2 
        \enspace,
    \end{equation}
    where equality is attained if and only if the following conditions hold:
    \begin{enumerate}[label={(Q\arabic*)}, start=5,labelindent=10pt,leftmargin=!,labelwidth=\widthof{\ref{con:supcon_rep_sum:2}}]
        \item \label{con:supcon_rep_sum:1}
            $\sum_{n\in [N]} z_n = 0$\enspace.
        \item \label{con:supcon_rep_sum:2}
            For every $n,m\in [N]$, $y_n = y_m$ implies $z_n = z_m$\enspace.        
    \end{enumerate}
\end{lemma}

\begin{proof}
    We first rewrite the sum as 
    \begin{align}
        \sum_{y\in \mcY}
        \sum_{\substack{n\in [N]\\y_n = y}}
        \sum_{\substack{m\in [N]\\y_n \neq y}}
        \inprod{z_n}{z_m}
        &=
            \sum_{y\in \mcY}
            \sum_{\substack{y'\in \mcY \\ y \neq y'}}
            \sum_{\substack{n\in [N]\\y_n = y}}
            \sum_{\substack{m\in [N]\\y_m = y'}}
            \inprod{z_n}{z_m}
            \\
        & = 
            \sum_{y\in \mcY}
            \sum_{y'\in \mcY}
            \sum_{\substack{n\in [N]\\y_n = y}}
            \sum_{\substack{m\in [N]\\y_m = y'}}
            \inprod{z_n}{z_m}
            - 
            \sum_{y\in \mcY}
            \sum_{\substack{n\in [N]\\y_n = y}}
            \sum_{\substack{m\in [N]\\y_m = y}}
            \inprod{z_n}{z_m}
            \\
        & = 
            \sum_{\substack{n\in [N]}}
            \sum_{\substack{m\in [N]}}
            \inprod{z_n}{z_m}
            - 
            \sum_{y\in \mcY}
            \sum_{\substack{n\in [N]\\y_n = y}}
            \sum_{\substack{m\in [N]\\y_m = y}}
            \inprod{z_n}{z_m}
            \\
        & =  
            \inprod{\sum_{\substack{n\in [N]}} z_n}{\sum_{\substack{n\in [N]}}z_n}
            -
            \sum_{y\in \mcY}
            \sum_{\substack{n\in [N]\\y_n = y}}
            \sum_{\substack{m\in [N]\\y_m = y}}
            \inprod{z_n}{z_m}
            \enspace,
    \end{align}
    where, for the last step, we used the linearity of the inner product.
    Using that the norm is positive-definite and applying the Cauchy-Schwarz inequality, yields the following lower bound:
    \begin{align}
        \sum_{y\in \mcY}
        \sum_{\substack{n\in [N]\\y_n = y}}
        \sum_{\substack{m\in [N]\\y_n \neq y}}
        \inprod{z_n}{z_m}
        & =
            \norm{ \sum_{n\in [N]} z_n}^2  
            -
            \sum_{y\in \mcY}
            \sum_{\substack{n\in [N]\\y_n = y}}
            \sum_{\substack{m\in [N]\\y_m = y}}
            \inprod{z_n}{z_m}
            \\
        & \stackrel{\ref{con:supcon_rep_sum:1}}{\ge}
            0  - 
            \sum_{y\in \mcY}
            \sum_{\substack{n\in [N]\\y_n = y}}
            \sum_{\substack{m\in [N]\\y_m = y}}
            \inprod{z_n}{z_m}
            \\
        & \stackrel{\ref{con:supcon_rep_sum:2}}{\ge}            
            - 
            \sum_{y\in \mcY}
            \sum_{\substack{n\in [N]\\y_n = y}} \norm{z_n}
            \sum_{\substack{m\in [N]\\y_m = y}}
            \norm{z_m}
            \\
        & = 
            - \sum_{y\in \mcY}
            \left(
                \mult{Y}(y) {\rho_{\mathcal Z}}
            \right)^2
        =        
            - {\rho_{\mathcal Z}}^2 
            \sum_{y\in \mcY}
            \mult{Y}(y)^2 
        \enspace
    \end{align}
    Equality is attained if and only if the following conditions hold:
    \begin{enumerate}[labelindent=10pt,leftmargin=!,labelwidth=\widthof{\ref{con:supcon_rep_sum:2}}]
        \item [\ref{con:supcon_rep_sum:1}]$\sum_n z_n = 0$
        \item [\ref{con:supcon_rep_sum:2}]
        We have equality in all applications of the Cauchy-Schwarz inequality, i.e., for every $y\in \mcY$ and every $n,m\in [N]$ with $y_n=y_m=y$ there exists $\lambda(y,n,m)\ge0$ such that $z_n = \lambda (y,n,m)z_m$.
        Since $\mcZ$ is a sphere $\lambda(y,n,m) = 1$ and so the above is equivalent to $y_n=y_m$ implies $z_n = z_m$.
    \end{enumerate}
\end{proof}
\begin{lemma}[Sum of repulsion terms]
    \label{lem:supcon_rep}
    Let $l\in\set{2,\dots,b-1}$ and let $\mcZ = \bbS_{\rho_{\mathcal Z}}^{h-1}$.
    For every $Z \in \mcZ^N$ and every balanced
    $Y\in \mcY^N$, we have that
    \begin{equation}
        \sum_{y \in \mcY}\sum_{B \in \mcB_{y,l}}
        S_{\text{rep}}(Z;Y,B,y)
        \ge
        -  
        |\mcB_{y,l}| \frac{|\mcY|}{|\mcY|-1}
        {\rho_{\mathcal Z}}^2 
        \enspace,
    \end{equation}
    where equality is attained if and only if
    the conditions \ref{con:supcon_rep_sum:1} {\normalfont\&} \ref{con:supcon_rep_sum:2} from \cref{lem:supcon_rep_sum} are fulfilled.
\end{lemma}
\begin{proof}
    Recall from \cref{lem:supcon_batches_to_indices} that 
    \begin{equation}
        \sum_{B \in \mcB_{y,l}}
        S_{\text{rep}}(Z;Y,B,y)
        = 
            \sum_{\substack{n\in [N]\\ y_n = y}} 
            \sum_{\substack{m\in [N] \\ y_m \neq y}}
            K_{n,m}(y,l) \,
            \inprod{z_n}{z_m}
        \enspace,
    \end{equation}
    and from \cref{lem:supcon_combinatorics} that
    \begin{equation}
        K_{n,m}(y,l) = \frac{|\mcB_{y,l}|}{\mult{Y}(y)(N - \mult{Y}(y))} \enspace.
    \end{equation}
    Therefore,
    \begin{equation}
        \sum_{y \in \mcY}\sum_{B \in \mcB_{y,l}}
        S_{\text{rep}}(Z;Y,B,y)
        \stackrel{\text{Lem.~}\ref{lem:supcon_batches_to_indices}}{=} 
            \sum_{y\in \mcY}
            \sum_{\substack{n\in [N]\\ y_n = y}} 
            \sum_{\substack{m\in [N] \\ y_m \neq y}}
            \frac{|\mcB_{y,l}|}{\mult{Y}(y)(N - \mult{Y}(y))} \,
            \inprod{z_n}{z_m}
        \enspace.
    \end{equation}
    Since $Y$ is balanced, i.e. $\mult{Y}(y) = \nicefrac{N}{|\mcY|}$ for every $y\in \mcY$, the term
    \begin{equation}
        \frac{|\mcB_{y,l}|}{\mult{Y}(y)(N - \mult{Y}(y))} 
        = \frac{|\mcB_{y,l}|}{N^2} \frac{|\mcY|^2}{|\mcY|-1}
        \label{eq:supcon:K_balanced}
    \end{equation}
    does not depend on the labels $y$ as (1) $|\mcB_{y,l}|$ is independent from $y$ due to symmetry and (2) so is $\mult{Y}(y)$.
    For brevity, we will still write $|\mcB_{y,l}|$ in the following, but keep in mind that it is constant \wrt $y$. 
    Furthermore, by \cref{lem:supcon_rep_sum}
    \begin{equation}
        \sum_{y\in \mcY}
        \sum_{\substack{n\in [N]\\ y_n = y}} 
        \sum_{\substack{m\in [N] \\ y_m \neq y}}
        \inprod{z_n}{z_m}
        \ge
            - {\rho_{\mathcal Z}}^2 
            \sum_{y\in \mcY}
            \mult{Y}(y)^2 
        = -  \frac{N^2}{|\mcY|} {\rho_{\mathcal Z}}^2 
        \enspace,
    \end{equation}
    where equality is attained if and only if the conditions \ref{con:supcon_rep_sum:1} \& \ref{con:supcon_rep_sum:2} are fulfilled.
    Therefore, we obtain the claimed bound
    \begin{align}
        \sum_{y \in \mcY}\sum_{B \in \mcB_{y,l}}
        S_{\text{rep}}(Z;Y,B,y)
        &= 
            \frac{|\mcB_{y,l}|}{N^2} \frac{|\mcY|^2}{|\mcY|-1}
            \sum_{y\in \mcY}
            \sum_{\substack{n\in [N]\\ y_n = y}} 
            \sum_{\substack{m\in [N] \\ y_m \neq y}}
            \inprod{z_n}{z_m}
            \\
        &\ge 
            -    
            \frac{|\mcB_{y,l}|}{N^2} \frac{|\mcY|^2}{|\mcY|-1}            
            \frac{N^2}{|\mcY|} {\rho_{\mathcal Z}}^2 
            \\
        & = 
            -    
            |\mcB_{y,l}| \frac{|\mcY|}{|\mcY|-1}
            {\rho_{\mathcal Z}}^2 
        \enspace.
    \end{align}
\end{proof}

As we have lower-bounded the attraction and repulsion components in Lemmas~\ref{lem:supcon_att} and \ref{lem:supcon_rep}, respectively, the following lemma, bounding the exponent in \cref{eq:supcon_outer_bound} of \cref{lem:supcon_outer_bound}, is an immediate consequence.

\begin{lemma}
    \label{lem:supcon_inner_bound}
    Let $l\in \set{2,\dots,b-1}$ and let $\mcZ = \bbS_{\rho_{\mathcal Z}}$.
    For every $Z \in \mcZ^N$ and every balanced $Y\in \mcY^N$, we have that
    \begin{equation}
        \frac{1}{M_l}
        \sum_{y \in \mcY} \sum_ {B \in \mcB_{y,l}}
        S(Z;Y,B,y)
        \ge 
        \frac{|\mcY|}{|\mcY|-1}{\rho_{\mcZ}}^2 
        \enspace,
    \end{equation}
    where equality is attained if and only if the following conditions hold:
    \begin{enumerate}[label={(A\arabic*)},start=3, labelindent=10pt,leftmargin=!,labelwidth=\widthof{\ref{con:supcon_inner_bound:mean}}]
        \item \label{con:supcon_inner_bound:collapse}For every $n,m\in [N]$, $y_n = y_m$ implies $z_n = z_m$\enspace.  
        \item \label{con:supcon_inner_bound:mean}
        $\sum_{n\in [N]} z_n = 0$\enspace.
    \end{enumerate}       
\end{lemma}
\begin{proof}
    Since $Y$ is balanced, $|\mcB_{y,l}|$ does not depend on $y$, and so 
    \begin{equation}
        M_l = {\sum_{y\in Y} |\mcB_{y,l}|} = |\mcY| |\mcB_{y,l}|
        \enspace.
    \end{equation}
    Leveraging the bounds on the sums of the attraction terms $S_{\text{att}}(Z;Y,B,y)$ and of the repulsion terms $S_{\text{rep}}(Z;Y,B,y)$ from \cref{lem:supcon_att} and \cref{lem:supcon_rep}, respectively, we get
    \begin{align}
        \sum_{y \in \mcY} \sum_ {B \in \mcB_{y,l}}
        S(Z;Y,B,y)
        & = 
            \left(
                \sum_{y \in \mcY} \sum_ {B \in \mcB_{y,l}}
                S_{\text{att}}(Z;Y,B,y)
            \right)
            {+}
            \left(
                \sum_{y \in \mcY} \sum_ {B \in \mcB_{y,l}}
                S_{\text{rep}}(Z;Y,B,y)
            \right)
            \\
        & \ge 
            - |\mcY| |\mcB_{y,l}|
            {\rho_{\mcZ}}^2
            - 
            |\mcB_{y,l}| \frac{|\mcY|}{|\mcY|-1}
            {\rho_{\mathcal Z}}^2 
            \\
        & = 
            - |\mcY| |\mcB_{y,l}|
            {\rho_{\mcZ}}^2
            \left(
                1 + \frac{1}{|\mcY|-1}
            \right)
            \\
        & = 
            - |\mcY| |\mcB_{y,l}|
            {\rho_{\mcZ}}^2
            \frac{|\mcY|}{|\mcY|-1}   
        \enspace.
    \end{align}
    This is the bound as stated in the lemma. 
    Herein, equality is attained if and only if equality is attained in \cref{lem:supcon_att} and \cref{lem:supcon_rep}.
    Since conditions \ref{con:supcon_att} and  \ref{con:supcon_rep_sum:2} are the same as condition \ref{con:supcon_inner_bound:collapse} and, additionally, since condition \ref{con:supcon_rep_sum:1} is the same as condition \ref{con:supcon_inner_bound:mean}, the lemma follows.
\end{proof}
\begin{lemma}
    \label{lem:supcon_final}
    Combining \cref{lem:supcon_outer_bound} and \cref{lem:supcon_inner_bound} implies that 
    the supervised contrasive loss $\LSC(Z;Y)$ is bounded from below by
    \begin{equation}
        \LSC(Z;Y) 
        \ge 
        \sum_{l =2}^{b} 
        l M_l
        \log 
        \left( 
            l - 1 + (b-l)
            \exp \left( 
                - \frac{|\mcY|}{|\mcY|-1} {\rho_{\mcZ}}^2                
            \right)
        \right)
        \enspace,
    \end{equation}
    where equality is attained if and only if there are $\zeta_1, \dots, \zeta_{|\mcY|} \in \R^h$ such that the following conditions hold:
    \begin{enumerate}[label={(C\arabic*)},labelindent=10pt,leftmargin=!,labelwidth=\widthof{\ref{con:supcon_final:2}}]
        \item\label{con:supcon_final:1} 
        $\forall n \in [N]: z_n = \zeta_{y_n}$
        \item\label{con:supcon_final:2} 
        $\{\zeta_y\}_{y\in \mcY}$ form a $\rho_{\mcZ}$-sphere-inscribed regular simplex
    \end{enumerate}
\end{lemma}
\begin{proof}
    We have that
    \begin{align}        
        \LSC(Z;Y)
        & \stackrel{\text{Lem.}~\ref{lem:supcon_outer_bound}\phantom{0}}{\ge}
            \sum_{l = 2}^b l M_l                           
            \log 
            \bigg( 
                l - 1 + (b-l)
                \exp \bigg( 
                    \frac {1} {M_l}
                    \sum_{y \in \mcY} \sum_ {\substack{ B \in \mcB\\\mult{\ymset(B)}(y) = l}}\!\!
                    S(Z;Y,B,y)
                \bigg)\!
            \bigg)
            \label{eq:lem:supcon_final:1}
            \\
        &\stackrel{\text{Lem.}~\ref{lem:supcon_inner_bound}}{\ge}
            \sum_{l = 2}^b l M_l                           
            \log 
            \left( 
                l - 1 + (b-l)
                \exp \left(
                    -\frac{|\mcY|}{|\mcY|-1}{\rho_{\mcZ}}^2 
                \right)\!
            \right)
        \enspace.
    \end{align}
    Equality holds if and only if the equality conditions of \cref{lem:supcon_outer_bound} and \cref{lem:supcon_inner_bound} are fulfilled, i.e. if and only if:
    \begin{enumerate}[labelindent=10pt,leftmargin=!,labelwidth=\widthof{\ref{con:supcon_inner_bound:mean}}]
        \item [\ref{con:supcon_outer_bound:intra}]
            There exists a constant $\alpha$, such that $\forall n,m\in [N]$, $y_n = y_m$ implies $\inprod{z_n}{z_m} = \alpha$\enspace.
        \item [\ref{con:supcon_outer_bound:inter}]
        There exists a constant $\beta$, such that $\forall n,m\in [N]$, $y_n \neq y_m$ implies $\inprod{z_n}{z_m} = \beta$\enspace.
        \item[\ref{con:supcon_inner_bound:collapse}] For every $n,m\in [N]$, $y_n = y_m$ implies $z_n = z_m$\enspace. 
        \item [\ref{con:supcon_inner_bound:mean}]
        $\sum_{n\in [N]} z_n = 0$ 
    \end{enumerate}

    \emph{
    Note that \cref{lem:supcon_inner_bound} does not hold for $l=b$, so the exponent in \cref{eq:lem:supcon_final:1} might differ in this case. However, this is irrelevant as, in this case, the factor $(b-l)$ in front of the exponential function vanishes.}

    To finish the proof, we need to show under the assumption $\mcZ = \bbS_{\rho_{\mathcal Z}}$, that these conditions are equivalent to that there are $\zeta_1,\dots,\zeta_{|\mcY|}$ such that
    \begin{enumerate}[ labelindent=10pt,leftmargin=!,labelwidth=\widthof{\ref{con:supcon_final:2}}]
        \item[\ref{con:supcon_final:1}]
        $\forall n \in [N]: z_n = \zeta_{y_n}$ and
        \item[\ref{con:supcon_final:2}]
        $\{\zeta_y\}_{y\in \mcY}$ form a $\rho_{\mcZ}$-sphere-inscribed regular simplex, i.e.,
        \begin{enumerate}[label=(S\arabic*)]
            \item\label{def:simplex:s1:sc} $\sum_{y \in \mcY} \zeta_y = 0$,
            \item\label{def:simplex:s2:sc} $\| \zeta_y \| = \rho_{\mathcal Z}$ for $y \in \mcY$,
            \item\label{def:simplex:s3:sc} $\exists d \in \R: d = \inprod{\zeta_y}{\zeta_{y'}}$ for $y,y'\in \mcY$ with $y\neq y'$. 
        \end{enumerate}
    \end{enumerate}
    Obviously, \ref{con:supcon_inner_bound:collapse} $\Longleftrightarrow$ \ref{con:supcon_final:1} and
    \ref{def:simplex:s2:sc} holds by assumption. Now, assuming that \ref{con:supcon_inner_bound:collapse} resp. \ref{con:supcon_final:1} already hold, we have
    \ref{con:supcon_inner_bound:mean} $\Longleftrightarrow$ \ref{def:simplex:s1:sc},
    \ref{con:supcon_outer_bound:inter} $\Longleftrightarrow$ \ref{def:simplex:s3:sc},
    and 
    \ref{con:supcon_outer_bound:intra} $\Longleftarrow$
    \ref{def:simplex:s2:sc}.
    All implications are straightforward. We state an exemplary proof for \ref{con:supcon_outer_bound:intra} $\Longleftarrow$
    \ref{def:simplex:s2:sc}.
    Let $n,m\in [N]$ such that $y=y_n = y_m$. By condition \ref{con:supcon_final:1}, $z_n = z_m = \zeta_y$, so by condition \ref{def:simplex:s2:sc}, $\inprod{z_n}{z_m} = \norm{\zeta_y}^2 = {\rho_{\mathcal Z}}^2$, which does not depend on $n$ and $m$.
\end{proof}
  \clearpage
\section{Proofs for \cref{subsection:analysis_ce}}
    
In this section, we will prove \cref{thm:ce_bound_frob} of the main manuscript and its corollaries. First, we recall the main definitions of the paper and introduce an auxiliary function.
\vskip1ex
Throughout this section the following objects will appear repeatedly an thus are introduced one-off:
\begin{itemize}
    \item $h, N, K \in \N$
    \item $\mcZ = \R^h$
    \item $\mcY = \{1, \dots, K\} = [K]$
\end{itemize}
We additionally assume $|\mcY| = K \leq h +1$.

\rest@def@ce*

\begin{definition}[Auxiliary function $S$]
    \label{def:ce_aux_s}
    Let $\mcZ = \R^h$, then we define
    \begin{align*}
        \SCE(\,\cdot\,,\,\cdot\,;\,Y): \mcZ^N \times \mcZ^K &\to \bbR \\
        (Z,W) &\mapsto 
        \frac 1 N \frac K {K-1}
                \sum_{n \in [N]}
                \inprod{z_n}{\bar w - w_{y_n}}
                \enspace,
    \end{align*}
    where $\bar{w} = \frac 1 {|\mcY|} \sum_{y\in \mcY} w_y$.
\end{definition}

\begin{lemma}
\label{lem:ce_bound_1}
    Let $h, K, N\in \N$, $\mcZ = \bbR^h$. Further, let
    \begin{alignat*}{2}    
    Z &= (z_n)_{i=n}^N &&\in \mcZ^N,\\
    W &= (w_y)_{y=1}^K &&\in \mcZ^K,\\
    Y &= (y_n)_{i=n}^N &&\in \mcY^N\enspace.
    \end{alignat*}
    It holds that 
    \[
    \LCE(Z,W;Y)
            \ge \log \Big(
                1 + (K-1) \exp \big(
                    S(Z, W;\,Y)
                \big)
            \Big)\enspace,
    \]
    where $S$ is as in \cref{def:ce_aux_s}. 
    Equality is attained if and only if the following conditions hold:
    \begin{enumerate}[label={(P\arabic*)},labelindent=10pt,leftmargin=!,labelwidth=\widthof{\ref{con:ce_bound:P1}}]
        \item $\forall n \in [N]$ $\exists M_n$ such that $\forall y\in \mcY\setminus\set{y_n}$ all inner products $\inprod{z_n}{w_y} = M_n$ are equal. 
        \label{con:ce_bound:P1}
        \item $\exists M\in \bbR$ such that $\forall n\in [N]$ it holds that $\sum_{\substack{y \in \mcY \\ y \neq y_n}} (\inprod{z_n}{w_y} - \inprod{z_n}{w_{y_n}}) = M$\enspace.
        \label{con:ce_bound:P2}
    \end{enumerate}

\end{lemma}
\begin{proof}
    Using the identities $\log(t) = - \log(1/t)$ and $\exp(a)/\exp(b)= \exp (a -b)$, rewrite the cross-entropy loss in the equivalent form
    \begin{equation}
        \LCE(Z,W;Y)
            = \frac 1 N \sum_{n \in [N]} \log \left(
                1 + \sum_{\substack{y \in \mcY \\ y \neq y_n}}
                \exp \left(
                    \inprod{z_n}{w_y} - \inprod{z_n}{w_{y_n}}
                \right)
            \right)
        \enspace.
    \end{equation}
    In order to bound $\LCE$ from below, we apply Jensen's inequality twice; first for the convex function $t \mapsto \exp(t)$ and then for the convex function $t \mapsto \log(1+\exp(t))$:
    \begin{align}
        \LCE(Z,W;Y)
            &\stackrel{\ref{con:ce_bound:P1}}{\ge} 
            \!\frac 1 N \sum_{n \in [N]} \! \log \left(\!
                1 + (K-1) \exp \left(
                    \frac {1}{K-1} 
                    \sum_{\substack{y \in \mcY \\ y \neq y_n}}
                    \left(\inprod{z_n}{w_y} - \inprod{z_n}{w_{y_n}}\right)
                \!\right)                 
            \!\right)
            \\
            &\stackrel{\ref{con:ce_bound:P2}}{\ge}
            \!\log \left(\!
                1 + (K-1) \exp \left(\frac 1 N \frac 1 {K-1} \sum_{n \in [N]}\sum_{\substack{y \in \mcY \\ y \neq y_n}}
                \left( \inprod{z_n}{w_y} - \inprod{z_n}{w_{y_n}} \right)
                \!\right)
            \!\right)
        \label{lem:ce_bound_1:eq1}
    \end{align}
    By the linearity of the inner product and as the addend for $y = y_n$ equals zero, the exponent in \cref{lem:ce_bound_1:eq1} is simply $S(Z,W;Y)$, which proves the bound. 
    
    The equality condition is obtained from the combination of the equality cases in both applications of Jensen's inequality. These are:
    \begin{enumerate}
        [label={(P\arabic*)},labelindent=10pt,leftmargin=!,labelwidth=\widthof{\ref{con:ce_bound:P2}}]
        \item[\ref{con:ce_bound:P1}] $\forall n \in [N]$ $\exists M_n$ such that $\forall y\in \mcY\setminus\set{y_n}$ all inner products $\inprod{z_n}{w_y} = M_n$ are equal.
        \item[\ref{con:ce_bound:P2}] $\exists M\in \bbR$ such that $\forall n\in [N]$ it holds that $\sum_{\substack{y \in \mcY \\ y \neq y_n}} (\inprod{z_n}{w_y} - \inprod{z_n}{w_{y_n}} )= M$.
    \end{enumerate}
\end{proof}

\begin{lemma}
    \label{lem:ce_bound_2}
    Let $h, K, N\in \N$, $\rho_{\mcZ}>0$ and $\mcZ = \{z \in \R^h: \|z\| \leq \rho_{\mcZ}\}$. Further, let 
    \begin{alignat*}{2}    
    Z &= (z_n)_{n=1}^N &&\in \mcZ^N\enspace,\\
    W &= (w_y)_{y=1}^K &&\in (\R^h)^K\enspace,\\
    Y &= (y_n)_{n=1}^N &&\in \mcY^N\enspace.
    \end{alignat*}
    
    If the class configuration $Y$ is balanced, i.e., for all ${y \in \mcY}$, 
    $N_y = \big|\{ i \in [N]: y_i = y\}\big| = \nicefrac{N}{K}$,
    then
    \begin{equation}
        \SCE(Z, W;\,Y)
        \geq 
        - \rho_{\mathcal Z} \frac{\sqrt{K}}{K-1} \|W\|_F\enspace,
    \end{equation}
    where $\norm{\cdot}_F$ denotes the Frobenius norm.
    We get equality if and only if the following conditions hold:
    \begin{enumerate}[label={(P\arabic*)},labelindent=10pt,leftmargin=!,labelwidth=\widthof{\ref{con:ce_bound:P6}}]
        \setcounter{enumi}{2}
        \item\label{con:ce_bound:P3} $\forall n\in [N]$ there $\exists \lambda_n \le 0$ such that $z_n = \lambda_n (\bar w  - w_{y_n})$
        \item\label{con:ce_bound:P4} $\forall n: \norm{z_n} = \rho_{\mathcal Z}$
        \item\label{con:ce_bound:P5} $\forall y \in \mcY$ the terms $\norm{\bar w}^2+ \norm{w_y}^2 - 2 \inprod{\bar w}{w_y}$ are equal
        \item\label{con:ce_bound:P6} $\bar w = 0$
    \end{enumerate}    
\end{lemma}

\begin{proof}
    We will bound the function $\SCE$ from Lemma \ref{lem:ce_bound_1}, using the norm constraint on each $z_n \in \mcZ$. In particular, 
    \begin{align*}
        \SCE(Z,W;Y)
            & = \frac 1 N \frac K {K-1}
            \sum_{n \in [N]}
            \inprod{z_n}{ \bar w - w_{y_n}} 
            \\
            &\stackrel{\ref{con:ce_bound:P3}}{\ge}  
            - \frac 1 N \frac K {K-1}
            \sum_{n \in[N]}
            \norm{z_n}\norm{\bar w - w_{y_n}}
            \\
            &\stackrel{\ref{con:ce_bound:P4}}{\geq} 
            - \frac 1 N \frac K {K-1}
            \sum_{n \in [N]}
            \rho_{\mathcal Z} \norm{\bar w - w_{y_n}}
            \\
            & = - \frac 1 N \frac K {K-1} \rho_{\mathcal Z}
            \sum_{y \in \mcY}                        
            \norm{\bar w - w_{y}}
            \left(
                \sum_{\substack{n \in [N] \\ y_n= y}}
                1
            \right)
            \\
            & = - \frac 1 N \frac K {K-1} \rho_{\mathcal Z}
            \sum_{y \in \mcY}                    
            \norm{\bar w - w_{y}}
            N_y
            \\
            &= - \frac 1 N \frac K {K-1} \rho_{\mathcal Z}
            \frac N K
            \sum_{y \in \mcY}
            \norm{\bar w - w_{y}}
            \tag{by assumption $N_y = \frac N K$}
            \\
            &= - \frac 1 {K-1} \rho_{\mathcal Z}
            \sum_{y\in\mcY}
            \sqrt{
                \norm{\bar w}^2 + \norm{w_y}^2 - 2  \inprod{\bar w}{w_y}
            }
            \\
            &\stackrel{\ref{con:ce_bound:P5}}{\ge} 
            - \rho_{\mathcal Z} \frac K  {K-1}
            \sqrt{ \frac 1 K
                \sum_{y \in \mcY} \left(
                    \norm{\bar w}^2 + \norm{w_y}^2 - 2  \inprod{\bar w}{w_y}
                \right)                
            }
            \\ 
            &= - \rho_{\mathcal Z} \frac K  {K-1}
            \sqrt{   \frac 1 K \left(
                K \norm{\bar w}^2 +\sum_{y \in \mcY} \norm{w_y}^2 - 2  \inprod{\bar w}{\sum_{y \in \mcY}w_y}
            \right)         
            }
            \\
            &=
            - \rho_{\mathcal Z} \frac K  {K-1}
            \sqrt{     \frac 1 K           
                \sum_{y \in \mcY} \norm{w_y}^2 -   \norm{\bar w}^2
            }
            \\
            & \stackrel{\ref{con:ce_bound:P6}}{\ge}
            - \rho_{\mathcal Z} \frac K {K-1} 
            \sqrt{  \frac 1 K              
                \sum_{y \in \mcY} \norm{w_y}^2
            }
            \\
            &=
            - \rho_{\mathcal Z} \frac{\sqrt{K}}{K-1} \|W\|_F                
            \enspace,
    \end{align*}
    where
    \begin{itemize}[label={(P\arabic*)},labelindent=10pt,leftmargin=!,labelwidth=\widthof{\ref{con:ce_bound:P6}}]
        \setcounter{enumi}{2}
        \item[\ref{con:ce_bound:P3}] follows from the Cauchy-Schwarz inequality with equality if and only if $\forall n\in [N]$ there $\exists \lambda_n \le 0$ such that $z_n = \lambda_n (\bar w  - w_{y_n})$,
        \item[\ref{con:ce_bound:P4}] follows from the assumption on the space $\mcZ$, with equality if and only if $\forall n$, $\norm{z_n} = \rho_{\mathcal Z}$ is maximal,
        \item[\ref{con:ce_bound:P5}] follows from Jensen's inequality for the convex function $t\mapsto - \sqrt t$ with equality if and only if $\forall y \in \mcY$ the terms $\norm{\bar w}^2+ \norm{w_y}^2 - 2 \inprod{\bar w}{w_y}$ are equal,
        \item[\ref{con:ce_bound:P6}]follows from the positivity of the norm, with equality if and only if $\bar w = 0$, i.e. $W$ is centered at the origin.
    \end{itemize}    
\end{proof}

\rest@thm@ce@bound@frob*
\vspace{-0.5cm}
    \begin{proof}
        To prove the bound, we consecutively leverage \cref{lem:ce_bound_1} and \cref{lem:ce_bound_2}. 
        \begin{align*}
            \LCE(Z, W;\,Y) 
            &\stackrel{\text{Lem.}~\ref{lem:ce_bound_1}}{\ge}
            \log 
            \Big(
            1 + (K-1) \exp 
            \big( S(Z, W;\,Y)\big)
            \Big)\\
            &\stackrel{\text{Lem.}~\ref{lem:ce_bound_2}}{\ge}
            \log 
            \left(
            1 + (K-1) \exp 
            \left( 
                - \rho_{\mathcal Z} \frac{\sqrt{K}}{K-1} \|W\|_F
            \right)
            \right)
            \enspace.
        \end{align*} 
    
        The application of \cref{lem:ce_bound_1} and \ref{lem:ce_bound_2} above also yields the sufficient and necessary conditions for equality, which are 
         \ref{con:ce_bound:P1}, \ref{con:ce_bound:P2}, \ref{con:ce_bound:P3}, \ref{con:ce_bound:P4}, \ref{con:ce_bound:P5} and \ref{con:ce_bound:P6}. 
        It remains to prove that those conditions are equivalent to \ref{thm:ce_bound_frob:c1}, \ref{thm:ce_bound_frob:c2}, \ref{thm:ce_bound_frob:c3}. 
        That is, we need to show that
        \begin{enumerate}[label={(P\arabic*)},labelindent=10pt,leftmargin=!,labelwidth=\widthof{\ref{con:ce_bound:P6}}]
            \item[\ref{con:ce_bound:P1}] $\forall n \in [N]$ $\exists M_n$ such that $\forall y\in \mcY\setminus\set{y_n}$ all inner products $\inprod{z_n}{w_y} = M_n$ are equal, 
            \item[\ref{con:ce_bound:P2}] $\exists M\in \bbR$ such that $\forall n\in [N]$ it holds that $\sum_{\substack{y \in \mcY \\ y \neq y_n}} \inprod{z_n}{w_y} - \inprod{z_n}{w_{y_n}} = M$.
            \item[\ref{con:ce_bound:P3}]$\forall n\in [N]$ there $\exists \lambda_n \leq 0$ such that $z_n = \lambda_n (\bar w  - w_{y_n})$,
            \item[\ref{con:ce_bound:P4}] $\forall n: \norm{z_n} = \rho_{\mathcal Z}$,
            \item[\ref{con:ce_bound:P5}] $\forall y \in \mcY$ the terms $\norm{\bar w}^2+ \norm{w_y}^2 - 2 \inprod{\bar w}{w_y}$ are equal,
            \item[\ref{con:ce_bound:P6}] $\bar w = 0$
        \end{enumerate}
        are equivalent to that the existence of $\zeta_1,\dots,\zeta_{|\mcY|}\in \bbR^h$ such that
        \begin{itemize}[labelindent=10pt,leftmargin=!,labelwidth=\widthof{\ref{thm:ce_bound_frob:c3}}]
            \item[\ref{thm:ce_bound_frob:c1}] $\forall n \in [N]: z_n = \zeta_{y_n}$ 
            \item[\ref{thm:ce_bound_frob:c2}]  $\{\zeta_y\}_{y\in \mcY}$ form a $\rho_{\mcZ}$-sphere-inscribed regular simplex, \ie, it holds that 
            \begin{itemize}
                \item[\ref{def:simplex:s1}] $\sum_{y \in \mcY} \zeta_y = 0$\enspace,
                \item[\ref{def:simplex:s2}] $\| \zeta_y \| = \rho_{\mcZ}$ for $y \in \mcY$\enspace,
                \item[\ref{def:simplex:s3}] $\exists d \in \R: d = \inprod{\zeta_y}{\zeta_{y'}}$ for $1 \leq y < y' \leq K$\enspace. 
            \end{itemize}
            \item[\ref{thm:ce_bound_frob:c3}] $\exists \rho_{\mcW} > 0: \forall y \in \mcY: w_{y} = \frac{\rho_{\mcW}}{\rho_{\mcZ}}  \zeta_{y}$\enspace. 
        \end{itemize}
        The arguments for the equivalencies are given below:

        First, we show (P1) - (P6) $\implies$ (C1) - (C3):
    
        \underline{Ad \ref{thm:ce_bound_frob:c1}}: We need to show that $\forall n \in [N]:~z_n=\zeta_{y_n}$.

        Let $n\in [N]$. Conditions \ref{con:ce_bound:P3} and \ref{con:ce_bound:P6} yield $ z_n = - \lambda_n w_{y_n}$ where $\lambda_n \leq 0$.
        If $w_{y_n}=0$, this immediately implies $\ref{thm:ce_bound_frob:c1}$ with $\zeta_{y_n} = 0$. If $w_{y_n}\neq 0$, we have $|\lambda| = \norm{z_n}/\norm{w_{y_n}}$, and thus
        by \ref{con:ce_bound:P4}
        \begin{equation}
            z_n = - 
            \left(-\frac{\|z_n\|}{\|w_{y_n}\|}\right) 
            w_{y_n}
            \stackrel{\ref{con:ce_bound:P4}}{=}
            \frac{\rho_{\mcZ}}{\|w_{y_n}\|} w_{y_n}
            \enspace.
            \label{eq:thm1:1}
        \end{equation}
        Consequently, condition $\ref{thm:ce_bound_frob:c1}$ is fulfilled with $\zeta_{y_n} = \rho_{\mathcal Z} \frac{w_{y_n}}{\norm{w_{y_n}}}$.

        \underline{Ad \ref{thm:ce_bound_frob:c3}}: We need to show that $\exists \rho_{\mcW} > 0$ such that $\forall y \in \mcY$ we have $w_{y} = \frac{\rho_{\mcW}}{\rho_{\mcZ}}  \zeta_{y}$\enspace.
        
        Since $Y$ is balanced, for every label $y\in \mcY$ we have that $N_y = N/K>0$ and so there exists $n\in [N]$ with $y_n = y$.
        Thus \cref{eq:thm1:1} implies for every $y\in \mcY$ that $\zeta_y = \rho_{\mathcal Z} \frac{w_{y}}{\norm{w_{y}}}$. Hence, condition \ref{thm:ce_bound_frob:c3} is fulfilled  with $\rho_{\mathcal W} = \norm{w_y}$ if all such norms $\norm{w_y}$ agree. 
        Indeed, by condition \ref{con:ce_bound:P5}, there is $C\in \bbR$ such that for each $y\in \mcY$
        \begin{equation*}
            C
            \stackrel{\ref{con:ce_bound:P5}}{=} 
            \norm{\bar w}^2 + \norm{w_y}^2 - 2 \inprod{\bar w}{w_y}
            \stackrel{\ref{con:ce_bound:P6}}{=}
            0 + \norm{w_y}^2 - 2 \cdot 0 
            = 
            \norm{w_y}^2
            \enspace.
        \end{equation*} 

        \underline{Ad \ref{thm:ce_bound_frob:c2}}: We need to show that $\{\zeta_y\}_{y \in \mcY}$ fulfill the requirements \ref{def:simplex:s1}, \ref{def:simplex:s2} and \ref{def:simplex:s3} of the regular simplex from \cref{def:simplex}. 

        From condition~\ref{thm:ce_bound_frob:c1} and  condition~\ref{thm:ce_bound_frob:c3}, we already know that 
        \begin{equation}
        \label{thm:ce_loss_frob:eq1}
            \frac{\rho_{\mcZ}}{\rho_{\mcW}} \cdot w_{y} = \zeta_{y} 
            \text{  for  } y \in \mcY
            \enspace,
        \end{equation}
        which we will use in the following. 

        \underline{Ad \ref{def:simplex:s1}}: We need to show that $\sum_{y \in \mcY} \zeta_y = 0$.
        
        This follows directly from \cref{thm:ce_loss_frob:eq1} and condition \ref{con:ce_bound:P6}, because
        \begin{equation}
        \sum\limits_{y \in \mcY}
        \zeta_y
        \stackrel{\text{Eq.~}\eqref{thm:ce_loss_frob:eq1}}{=}
        \frac{\rho_{\mcZ}}{\rho_{\mcW}}
        \sum\limits_{y \in \mcY}
        w_{y}
        \stackrel{\ref{con:ce_bound:P6}}{=} 0
        \enspace. 
        \end{equation}
        
        \underline{Ad \ref{def:simplex:s2}}: We need to show for every $y\in \mcY$ that $\| \zeta_y \| = \rho_{\mcZ}$.

        This follows directly from \cref{thm:ce_loss_frob:eq1} and the already proven condition \ref{thm:ce_bound_frob:c3}, because
        \begin{equation}
          \|\zeta_y\|  
          \stackrel{\text{Eq.~}\eqref{thm:ce_loss_frob:eq1}}{=}
           \|\frac{\rho_{\mcZ}}{\rho_{\mcW}} \cdot w_{y_n}  \| 
          \stackrel{\ref{thm:ce_bound_frob:c3}}{=}
          \frac{\rho_{\mcZ}}{\rho_{\mcW}} \cdot \rho_{\mcW}
          = \rho_{\mcZ}
          \enspace. 
        \end{equation}

        \underline{Ad \ref{def:simplex:s3}}: We need to show that for every $y,y'\in \mcY$ with $y\neq y'$ there $\exists d \in \R: d = \inprod{\zeta_y}{\zeta_{y'}}$\enspace.

        Let $y,y' \in \mcY$ with $y\neq y'$.
        Since $Y$ is balanced, we have that $N_{y'} = \nicefrac{N}{K}>0$. Hence, there exists $n\in [N]$ with $y' = y_n$ and so
        \begin{equation}  
            \label{eq:s3:1}
            \frac{\rho_{\mcW}}{\rho_{\mcZ}} \inprod{ \zeta_{y'}}{\zeta_y}
            =
            \inprod{ \zeta_{y_n}}{\frac{\rho_{\mcW}}{\rho_{\mcZ}}\zeta_y}
            \stackrel{\text{Eq.~}\eqref{thm:ce_loss_frob:eq1}}{=} 
            \inprod{\zeta_{y_n}}{w_y}
            \stackrel{\ref{thm:ce_bound_frob:c1}}{=}
            \inprod{z_n}{w_y}
            \stackrel{\ref{con:ce_bound:P1}}{=} M_n\enspace.
        \end{equation}
        Similarly,
        \begin{equation} 
            \label{eq:s3:2}
                \inprod{z_n}{w_{y_n}}       
                \stackrel{\ref{thm:ce_bound_frob:c1}}{=}
                \inprod{\zeta_{y_n}}{w_{y_n}}
                \stackrel{\text{Eq.~}\eqref{thm:ce_loss_frob:eq1}}{=}      
                \inprod{ \zeta_{y_n}}{\frac{\rho_{\mcW}}{\rho_{\mcZ}}\zeta_{y_n}}
                =
                \frac{\rho_{\mcW}}{\rho_{\mcZ}} \|\zeta_{y_n}\|^2
                \stackrel{\ref{def:simplex:s2}}{=}
                {\rho_{\mcW}}{\rho_{\mcZ}}
                \enspace.
        \end{equation}
        We leverage condition~\ref{con:ce_bound:P1} and condition~\ref{con:ce_bound:P2} to get that there exists $M\in \bbR$ such that
        \begin{align}
            M &
            \stackrel{\ref{con:ce_bound:P2}}{=}
            \sum_{\substack{\hat y \in \mcY \\ \hat y \neq y_n}}( \inprod{z_n}{w_{\hat y}} - \inprod{z_n}{w_{y_n}}) \\
            &\stackrel{\eqref{eq:s3:2}}{=}
            \Big( \sum_{\substack{\hat y \in \mcY \\ \hat y \neq y_n}} \inprod{z_n}{w_{\hat y}} \Big)
            - (K-1){\rho_{\mcW}}{\rho_{\mcZ}} \\
            &\stackrel{\ref{con:ce_bound:P1}}{=}
            (K-1) (M_n - {\rho_{\mcW}}{\rho_{\mcZ}})\\
            &\stackrel{\eqref{eq:s3:1}}{=}
            (K-1) (\frac{\rho_{\mcW}}{\rho_{\mcZ}}\inprod{ \zeta_{y'}}{\zeta_y} - {\rho_{\mcW}}{\rho_{\mcZ}})
            \enspace. 
        \end{align}
        Thus $\inprod{ \zeta_{y'}}{\zeta_y}=d$ is constant, and $d$ can be calculated by rearranging the equation above. 
       
        Next, we show (C1) - (C3) $\implies$ (P1) - (P6) :

        We assume that there exist $\zeta_1, \dots, \zeta_K \in \R^h$ such that conditions (C1) - (C3) are fulfilled.
    
        \underline{Ad \ref{con:ce_bound:P1}}: We need to show that $\forall n \in [N]$ $\exists M_n$ such that $\forall y\in \mcY\setminus\set{y_n}$ all inner products $\inprod{z_n}{w_y} = M_n$ are equal.

         Let $n\in [N]$ and $y \in \mcY \setminus\{y_n\}$, then 
        \begin{equation}     
            \label{eq:thm1:2}
            \inprod{z_n}{w_y} 
            \stackrel{ \ref{thm:ce_bound_frob:c1}}{=}
            \inprod{\zeta_{y_n}}{w_y}
            \stackrel{\ref{thm:ce_bound_frob:c3}}{=}
            \inprod{\zeta_{y_n}}{\frac{\rho_{\mcW}}{\rho_{\mcZ}} \zeta_{y}}
            \stackrel{\ref{def:simplex:s3}}{=}
            \frac{\rho_{\mcW}}{\rho_{\mcZ}} d
            \enspace,
        \end{equation}
        so condition \ref{con:ce_bound:P1} is fulfilled with $M_n = \frac{\rho_{\mcW}}{\rho_{\mcZ}} d$.
    
        \underline{Ad \ref{con:ce_bound:P2}}: We need to show that $\exists M$ such that $\forall n\in [N]$ it holds that $\sum_{\substack{y \in \mcY \\ y \neq y_n}} ( \inprod{z_n}{w_y} - \inprod{z_n}{w_{y_n}}) = M$.
        Let $n\in [N]$. From \cref{eq:thm1:2}, we already now that for $y\in \mcY\setminus\set{y}$ it holds that $\inprod{z_n}{w_y}=  \frac{\rho_{\mcW}}{\rho_{\mcZ}} d$.
        Similarly,
        \begin{equation}  
            \inprod{z_n}{w_{y_n}} 
            \stackrel{ \ref{thm:ce_bound_frob:c1}}{=}
            \inprod{\zeta_{y_n}}{w_y}
            \stackrel{\ref{thm:ce_bound_frob:c3}}{=}
            \inprod{\zeta_{y_n}}{\frac{\rho_{\mcW}}{\rho_{\mcZ}} \zeta_{y_n}}
            \stackrel{\ref{def:simplex:s2}}{=}
            {\rho_{\mcW}}{\rho_{\mcZ}}
            \enspace.
        \end{equation}
        Therefore,
        \begin{equation}
            \sum_{y \in \mcY\setminus\set{y_n}} \left( \inprod{z_n}{w_y} - \inprod{z_n}{w_{y_n}} \right)
            = (K-1) \left(  \frac{\rho_{\mcW}}{\rho_{\mcZ}} d - {\rho_{\mcW}}{\rho_{\mcZ}} \right)
        \end{equation}
        and condition~\ref{con:ce_bound:P2} is fulfilled with $M = (K-1) \left(  \frac{\rho_{\mcW}}{\rho_{\mcZ}} d - {\rho_{\mcW}}{\rho_{\mcZ}} \right)$.       
    
        \underline{Ad \ref{con:ce_bound:P4}}: We need to show that $\forall n: \norm{z_n} = \rho_{\mathcal Z}$.
        
        This follows immediately from condition~\ref{thm:ce_bound_frob:c1} and condition~\ref{def:simplex:s2}: 
        \begin{equation}
            \norm{z_n} 
            \stackrel{\ref{thm:ce_bound_frob:c1}}{=} 
            \norm{\zeta_{y_n}} 
            \stackrel{\ref{def:simplex:s2}}{=} 
            \rho_{\mcZ}
            \enspace.
        \end{equation}
    
        \underline{Ad \ref{con:ce_bound:P6}}: We need to show that $\bar w = 0$.
        
        This follows immediately from condition~\ref{thm:ce_bound_frob:c3} and condition~\ref{def:simplex:s1}: 
        \begin{equation}
            \bar{w}=\frac{1}{K} \sum\limits_{y \in \mcY} w_y 
            \stackrel{\ref{thm:ce_bound_frob:c3}}{=}
            \frac{1}{K} \frac{\rho_{\mcW}}{\rho_{\mcZ}} \sum\limits_{y \in \mcY}  \zeta_{y}
            \stackrel{\ref{def:simplex:s1}}{=}
            0
            \enspace.
        \end{equation}

        \underline{Ad \ref{con:ce_bound:P5}}: We need to show that $\forall y \in \mcY$ the terms $\norm{\bar w}^2+ \norm{w_y}^2 - 2 \inprod{\bar w}{w_y}$.

        This follows from conditions~\ref{thm:ce_bound_frob:c3} and \ref{def:simplex:s2}, such as the already proven condition~\ref{con:ce_bound:P6}.

         Let $y \in \mcY$ then 
        \begin{equation}
            \norm{\bar w}^2+ \norm{w_y}^2 - 2 \inprod{\bar w}{w_y}
            \stackrel{\ref{con:ce_bound:P6}}{=}
            0 + \norm{w_y}^2 - 2\cdot 0 
            \stackrel{\ref{thm:ce_bound_frob:c3}}{=} 
            \| \frac{\rho_{\mcW}}{\rho_{\mcZ}} \zeta_{y_n}\|^2
            \stackrel{\ref{def:simplex:s2}}{=}
            \rho_{\mcW}^2
            \enspace,
        \end{equation}
        which, indeed, does not depend on $y$.
    
        \underline{Ad \ref{con:ce_bound:P3}}: We need to show that $\forall n\in [N]$ there $\exists \lambda_n \leq 0$ such that $z_n = \lambda_n (\bar w  - w_{y_n})$.

        Let $n \in [N]$ and consider
        \begin{equation}
            z_n 
            \stackrel{\ref{thm:ce_bound_frob:c1}}{=} 
            \zeta_{y_n} 
            \stackrel{\ref{thm:ce_bound_frob:c3}}{=} 
            \frac{\rho_{\mcZ}}{\rho_{\mcW}} w_{y_n}
        \end{equation}
        Thus, from the already proven condition~\ref{con:ce_bound:P6}, it follows that
        \begin{equation}
            \bar w - w_{y_n} 
            \stackrel{\ref{con:ce_bound:P6}}{=} - w_{y_n}
            = - \frac{\rho_{\mcW}}{\rho_{\mcZ}} z_n
        \end{equation}
        and condition~\ref{con:ce_bound:P3} is fulfilled with $\lambda_n = - \frac{\rho_{\mcZ}}{\rho_{\mcW}} \le 0$.
\end{proof}

\rest@cor@ce@bound@r*
\begin{proof}
    By leveraging \cref{thm:ce_bound_frob}, 
    we get 
    \begin{align}
        \LCE(Z, W;\,Y) 
        &\stackrel{\text{Thm.~}\ref{thm:ce_bound_frob}}{\ge}
        \log \left(
            1 + (K-1) \exp \left( 
                - \rho_{\mathcal Z} \frac {\sqrt{K}}{K-1}
                \|W\|_F
            \right)
        \right)\\
        &=
        \log \left(
            1 + (K-1) \exp \left( 
                - \rho_{\mathcal Z} \frac {\sqrt{K}}{K-1}
                \sqrt{\sum\limits_{y\in \mcY} \|w_y\|^2}
            \right)
        \right)\\
        &\ge
        \log \left(
            1 + (K-1) \exp \left( 
                - \rho_{\mathcal Z} \frac {\sqrt{K}}{K-1}
                \sqrt{\sum\limits_{y\in \mcY} r_{\mcW}^2}
            \right)
        \right)\\
        &=
        \log \left(
            1 + (K-1) \exp \left( 
                - \rho_{\mathcal Z} \frac {K}{K-1}
                r_{\mcW}
            \right)
        \right)
        \enspace, 
    \end{align}
    where equality is attained if and only if the bound from \cref{thm:ce_bound_frob} is tight, i.e.,  conditions~\ref{thm:ce_bound_frob:c1}, \ref{thm:ce_bound_frob:c2}, \ref{thm:ce_bound_frob:c3} are fulfilled and, additionally,
    \begin{equation}
    \label{cor:ce_bound_r:eq1}
        r_{\mcW} = \|w_y\| \text{  for  }y \in \mcY
        \enspace. 
    \end{equation}
    It remains to show that if conditions \ref{thm:ce_bound_frob:c1} \ref{thm:ce_bound_frob:c2} are fulfilled, then the following equivalency holds:
    \begin{equation}
        \big(
        r_{\mcW} = \|w_y\| \text{  for  }y \in \mcY
        \;\land
        \;  \text{\ref{thm:ce_bound_frob:c3}}
        \big)
        \Longleftrightarrow 
        \text{\ref{cor_ce_bound_r:c3}}
        \enspace. 
    \end{equation}

    ``$\Longrightarrow$'': We need to show that $\forall y \in \mcY: w_y =\frac{r_{\mcW}}{\rho_{\mcZ}}\zeta_y$.

    So, let $y \in \mcY$.    
    By condition~\ref{thm:ce_bound_frob:c3}, there is $\rho_{\mcW} > 0$ such that
    \begin{equation} 
        \label{cor:ce_bound_r:eq2}
        w_{y} = \frac{\rho_{\mcW}}{\rho_{\mcZ}}  \zeta_{y}
        \enspace.
    \end{equation}
    Thus \ref{cor_ce_bound_r:c3} holds if $\rho_{\mathcal W} = r_{\mathcal W}$.
    Indeed,
    \begin{equation}
        r_{\mcW} 
        \stackrel{\eqref{cor:ce_bound_r:eq1}}{=} 
        \|w_y\| 
        \stackrel{\eqref{cor:ce_bound_r:eq2}}{=} 
        \frac{\rho_{\mcW}}{\rho_{\mcZ}} \|\zeta_y\|
        \stackrel{\ref{thm:ce_bound_frob:c2}}{=}
        \frac{\rho_{\mcW}}{\rho_{\mcZ}} \rho_{\mcZ}
        =
        \rho_{\mcW}
        \enspace. 
    \end{equation}     

    ``$\Longleftarrow$'': Follows immediately as we can choose $\rho_{\mcW}=r_{\mcW}>0$.
    
\end{proof}

\begin{lemma}
\label{lem:ce_bound_3}
    Let $\lambda$, $\rho_{\mcZ} > 0$, $K, h \in \N$ and $W \in (\R^h)^K$. The function
    \begin{equation}
    \label{lem:ce_bound_3:eq_1}
        f(x) = 
        \log
        \left(
        1 + (K-1)
        \exp
        \left(
            - \rho_{\mcZ} \frac{{K}}{K-1}x
        \right)
        \right)
        + \lambda K x^2
    \end{equation}
    is minimized by $x_0 = r_{\mathcal W}(\rho_{\mathcal Z},\lambda)>0$, \ie, the \emph{unique} solution to
    \begin{equation}
        0 = 
        K \left(2 \lambda  x-\frac{\rho_{\mathcal Z} }{e^{\frac{K \rho_{\mathcal Z}  x}{K-1}}+K-1}\right)
        \enspace.
    \end{equation}
\end{lemma}
\begin{proof}
    The first derivative of $f$ is given by
    \begin{equation}
        f'(x) = 
        K \left(2 \lambda  x-\frac{\rho_{\mathcal Z} }{e^{\frac{K \rho_{\mathcal Z}  x}{K-1}}+K-1}\right)
        \enspace.
    \end{equation}
    Note that $f'$ is strictly increasing. 
    Thus $f$ is strictly convex and has a unique minimum at the point $x_0$ where $f'(x_0) = 0$. 
    As $f'$ is continuous on $(0,\infty)$ with 
    \begin{equation}
        f'(0) = - \rho_{\mathcal Z}
        < 0
    \end{equation}
    and 
    \begin{equation}
        \lim\limits_{x \rightarrow \infty} f'(x)= \infty \enspace,
    \end{equation}
    the intermediate value theorem implies $0 < x_0 = r_{\mathcal W}(\rho_{\mcZ}, \lambda) < \infty$.
    
\end{proof}
\rest@cor@ce@bound@wd*
\begin{proof}
    By leveraging \cref{thm:ce_bound_frob} and \cref{lem:ce_bound_3} (with $x = {\norm{W}_F}/{\sqrt{K}}$), 
    we get 
    \begin{align*}
        &\LCE(Z, W;\,Y) + \lambda  \norm{W}_F^2\\
        &\stackrel{\text{Thm.~}\ref{thm:ce_bound_frob}}{\ge}
        \log \left(
            1 + (K-1) \exp \left( 
                - \rho_{\mathcal Z} \frac {\sqrt{K}}{K-1}
                \|W\|_F
            \right)
        \right)
        +
        {\lambda}  \norm{W}_F^2\\
        & \stackrel{\phantom{\text{Lem.~1}}}{=} 
        \log \left(
            1 + (K-1) \exp \left( 
                - \rho_{\mathcal Z} \frac {K}{K-1}
                x
            \right)
        \right)
        +
        {\lambda K}   x^2
        \tag{by setting $x = {\norm{W}_F}/{\sqrt{K}}$}\\
        &\stackrel{\text{Lem.~}\ref{lem:ce_bound_3}}{\ge}
        \log \left(
            1 + (K-1) \exp \left( 
                - \rho_{\mathcal Z} \frac {{K}}{K-1}
                r_{\mathcal W}(\rho_{\mcZ}, \lambda)
            \right)
        \right)
        +
        \lambda K r_{\mathcal W}(\rho_{\mcZ}, \lambda)^2
        \enspace, 
    \end{align*}
    where equality is attained if and only if 
    the bound from \cref{thm:ce_bound_frob} is tight, \ie, conditions~\ref{thm:ce_bound_frob:c1}, \ref{thm:ce_bound_frob:c2}, \ref{thm:ce_bound_frob:c3} are fulfilled and, additionally,
    \begin{equation}
    \label{cor:ce_bound_wd:eq1}
        \|W\|_F/\sqrt K = r_{\mathcal W}(\rho_{\mcZ}, \lambda)
        \enspace. 
    \end{equation}
    It remains to show that if \ref{thm:ce_bound_frob:c1} and \ref{thm:ce_bound_frob:c2} are fulfilled, it holds that 
    \begin{equation}
        \big(
            \|W\|_F/\sqrt K =  r_{\mathcal W}(\rho_{\mcZ}, \lambda)\;\land\;  \text{\ref{thm:ce_bound_frob:c3}}
        \big)
        \Longleftrightarrow 
        \text{\ref{cor_ce_bound_wd:c3}}
        \enspace. 
    \end{equation}
    ``$\Longrightarrow$``: We need to show for every $y \in \mcY$ that $w_y = \frac{r_{\mathcal W}(\rho_{\mcZ},\lambda)}{\rho_{\mcZ}}\zeta_y$.
    
    So let $y \in \mcY$. By condition~\ref{thm:ce_bound_frob:c3}, 
    there exists $\rho_{\mathcal W}>0$ such that  $w_y = \nicefrac{\rho_{\mcW}}{\rho_{\mcZ}} \zeta_y$.
    Thus, condition~\ref{cor_ce_bound_wd:c3} is fulfilled if $\rho_{\mathcal W} = r_{\mathcal W}(\rho_{\mcZ},\lambda)$.
    Indeed,
    \begin{align}
        r_{\mathcal W}(\rho_{\mcZ}, \lambda) 
        &\stackrel{\eqref{cor:ce_bound_wd:eq1}}{=} 
        \frac{\|W\|_F}{\sqrt{K}} 
        = 
        \sqrt{\frac 1 K \sum\limits_{y\in \mcY} \|w_y\|^2}
        \stackrel{\ref{thm:ce_bound_frob:c3}}{=} 
        \sqrt{\sum\limits_{y\in \mcY} \|\frac{\rho_{\mcW}}{\rho_{\mcZ}} \zeta_{y}\|^2}\\
        &\stackrel{\ref{thm:ce_bound_frob:c2}}{=} 
        \sqrt{\frac 1 K \sum\limits_{y \in \mcY} \rho_{\mcW}^2}
        = \rho_{\mcW} 
        \enspace. 
    \end{align}

    ``$\Longleftarrow$``: 

    Condition~\ref{thm:ce_bound_frob:c3} Is fulfilled as we can choose
    \begin{equation}
        \rho_{\mathcal W} = r_{\mathcal W}(\rho_{\mathcal Z},\lambda) \stackrel{\text{Lem.~}\ref{lem:ce_bound_3}}{>} 0 
        \enspace. 
    \end{equation}
    Finally, $r_{\mathcal W}(\rho_{\mathcal Z},\lambda) = \norm{W}_F / \sqrt{K}$, as
    \begin{align*}
        \|W\|_F 
        &= 
        \sqrt{\sum\limits_{y\in \mcY} \|w_y\|^2}
        \stackrel{\ref{cor_ce_bound_wd:c3}}{=}
        \sqrt{\sum\limits_{y\in \mcY} \|\frac{r_{\mathcal W}(\rho_{\mcZ}, \lambda)}{\rho_{\mcZ}} \zeta_{y}\|^2}\\
        &\stackrel{\ref{thm:ce_bound_frob:c2}}{=}
        \sqrt{\sum\limits_{y\in \mcY} r_{\mathcal W}(\rho_{\mcZ}, \lambda)^2}
        =
        \sqrt{K} r_{\mathcal W}(\rho_{\mathcal Z},\lambda)
        \enspace. 
    \end{align*}
\end{proof}
  \clearpage
\section{Additional Experiments}
\label{sec:suppmat_exp}

The experiments in \cref{subsection:theory_vs_practice} suggest that representations learned by minimizing the \textbf{SC} loss might arrange closer to the (theoretically optimal) simplex configuration, compared to representations learned by minimizing the \textbf{CE} loss.
To corroborate that this disparity is due to differing optimization dynamics of the loss functions, i.e., differing trajectories in the parameter space, and not an artifact of terminating the loss minimization to early, we repeat\footnote{
    i.e., the same setup and hyperparameters as in §\ref{subsection:theory_vs_practice}, except for the \emph{number of training iterations}
    }
these experiments when optimizing over \textbf{500k} SGD iterations instead of 100k. After every 10k iterations, we freeze the model, compute the class means of representation of the training data and evaluate two geometric properties on all of the training data: (1) the \emph{cosine similarity \textbf{across} class means} and (2) the \emph{cosine similarity \textbf{to} class means}\footnote{We omit the \emph{cosine similarity across weights} as, for \textbf{SC}, this requires to train an additional linear classifier each time.}, as illustrated in Figs. \ref{fig:suppmat_experiments_cifar10}, \ref{fig:suppmat_experiments_cifar100}.

\begin{figure}
    \captionsetup{width=.99\linewidth}
    \centering
    \begin{subfigure}[b]{0.99\textwidth}
        \centering
        \includegraphics[width=0.99\textwidth]{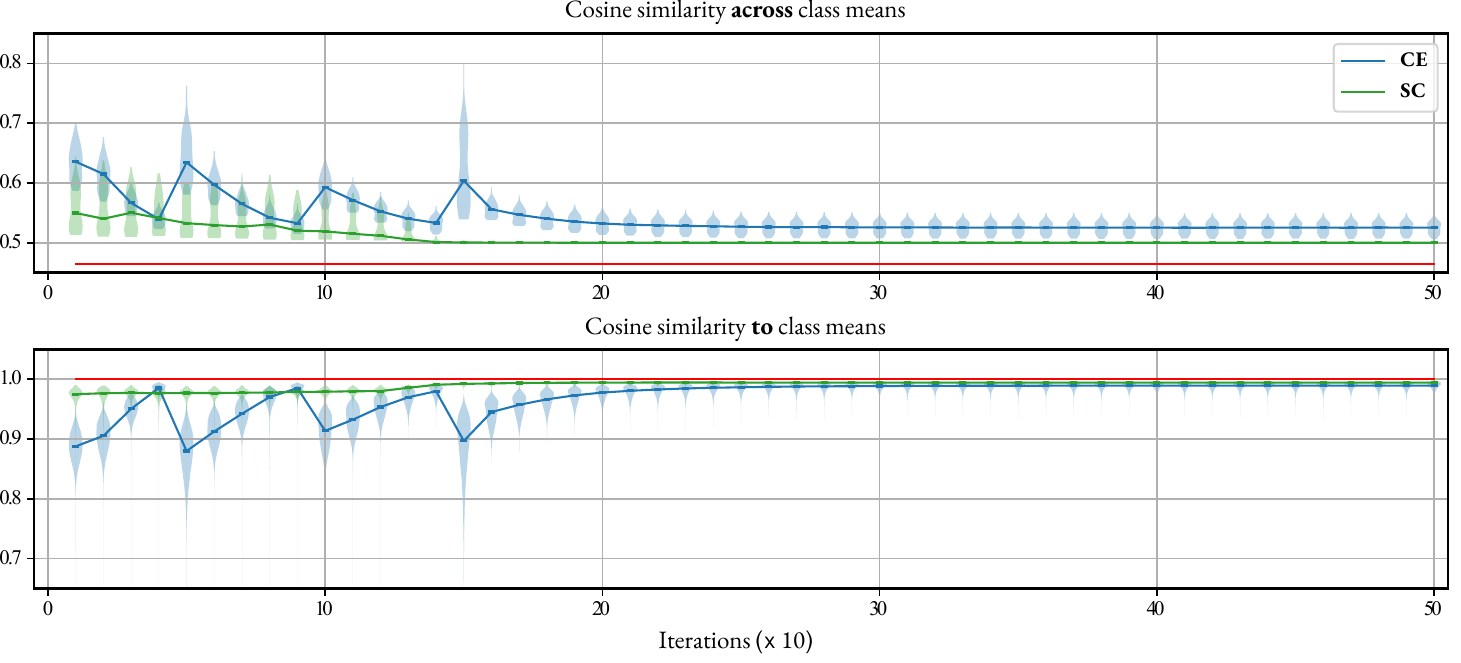}
        \caption{\textbf{CIFAR10} (\underline{without} augmentation)}
        \label{appendix:figure:experiment_cifar10_woaug}
    \end{subfigure}
    \vskip4ex
    \begin{subfigure}[b]{0.99\textwidth}
        \centering
        \includegraphics[width=0.99\textwidth]{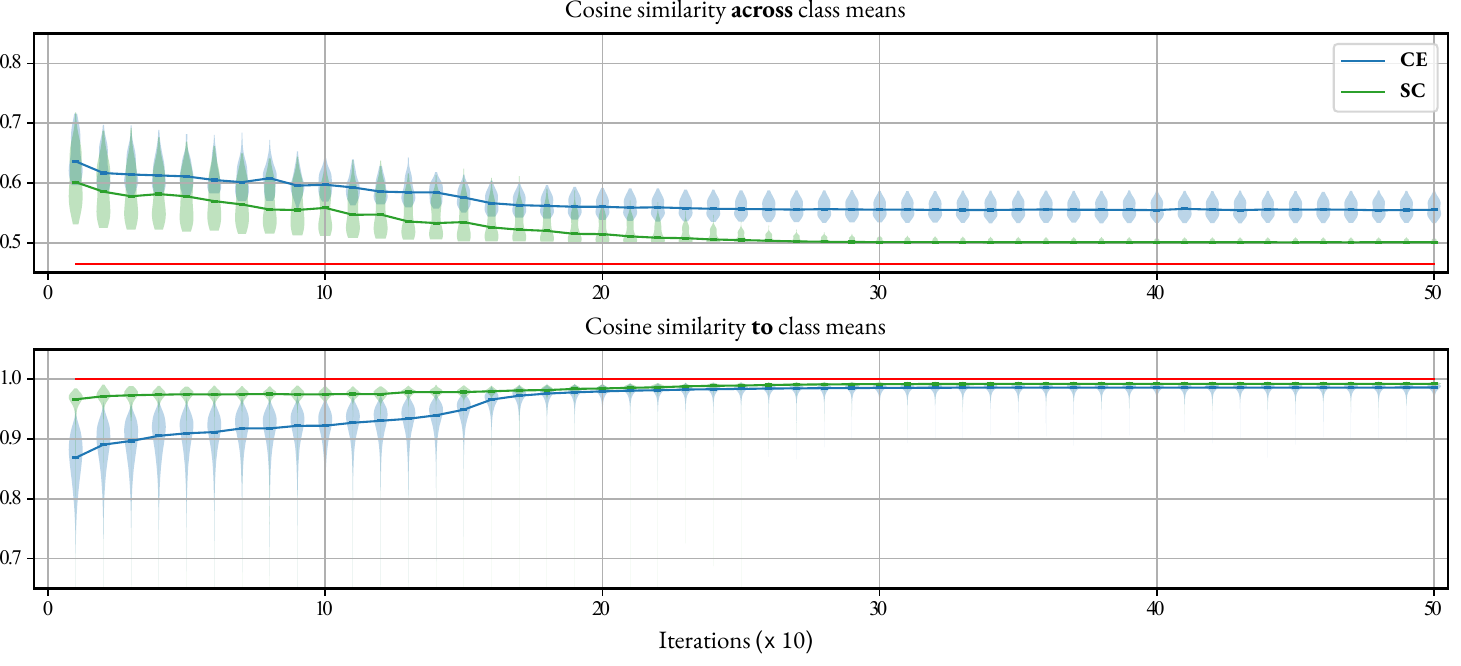}
        \caption{\textbf{CIFAR10} (\underline{with} augmentation)}
        \label{appendix:figure:experiment_cifar10_waug}
    \end{subfigure}
    \caption{
        Distribution of geometric properties of representations, $\enc_\theta(x_n)$, tracked during training. Representations are obtained from a ResNet-18 model trained \subref{appendix:figure:experiment_cifar10_waug} \underline{with}  and \subref{appendix:figure:experiment_cifar10_woaug} \underline{without} data augmentation on \textbf{CIFAR10}, with \textbf{CE} and \textbf{SC}, respectively. \textcolor{tabblue}{Blue} and \textcolor{tabgreen}{green} lines indicate the evolution of the medians over the iterations;
        \textcolor{tabred}{Red} lines indicate the \emph{sought-for} value at a regular simplex configuration.
        \label{fig:suppmat_experiments_cifar10}
        }
\end{figure}

The results reveal that (1) optimizing for 500k iterations improves convergence to the optimal state, yet at a very low speed, and (2) minimizing \textbf{SC} still yields representations closer to the simplex, compared to \textbf{CE}. The latter not only holds at the terminal stage of training, but at (almost) every evaluation step. 
Interestingly, on both datasets, the distributions of the computed properties obtained from the model trained via \textbf{CE} have notably more spread than the ones obtained from the model trained with \textbf{SC}.

Finally, we compare the geometric properties after training for 500k iteration with the ones from training over 100k iterations, i.e., Fig.~\ref{fig:geometry} in \cref{subsection:theory_vs_practice}. In case of \textbf{SC}, the distributions are roughly the same, whereas for \textbf{CE}, the distributions after 500k iterations are notably closer to the theoretical optimum than the ones after 100k iterations, particularly on the more complex CIFAR100 dataset. Once more, this highlights the faster convergence to the simplex arrangement via minimizing \textbf{SC}.
\vspace{0.2cm}

\begin{figure}
    \captionsetup{width=.99\linewidth}
    \centering
    \begin{subfigure}[b]{0.99\textwidth}
        \centering
        \includegraphics[width=0.99\textwidth]{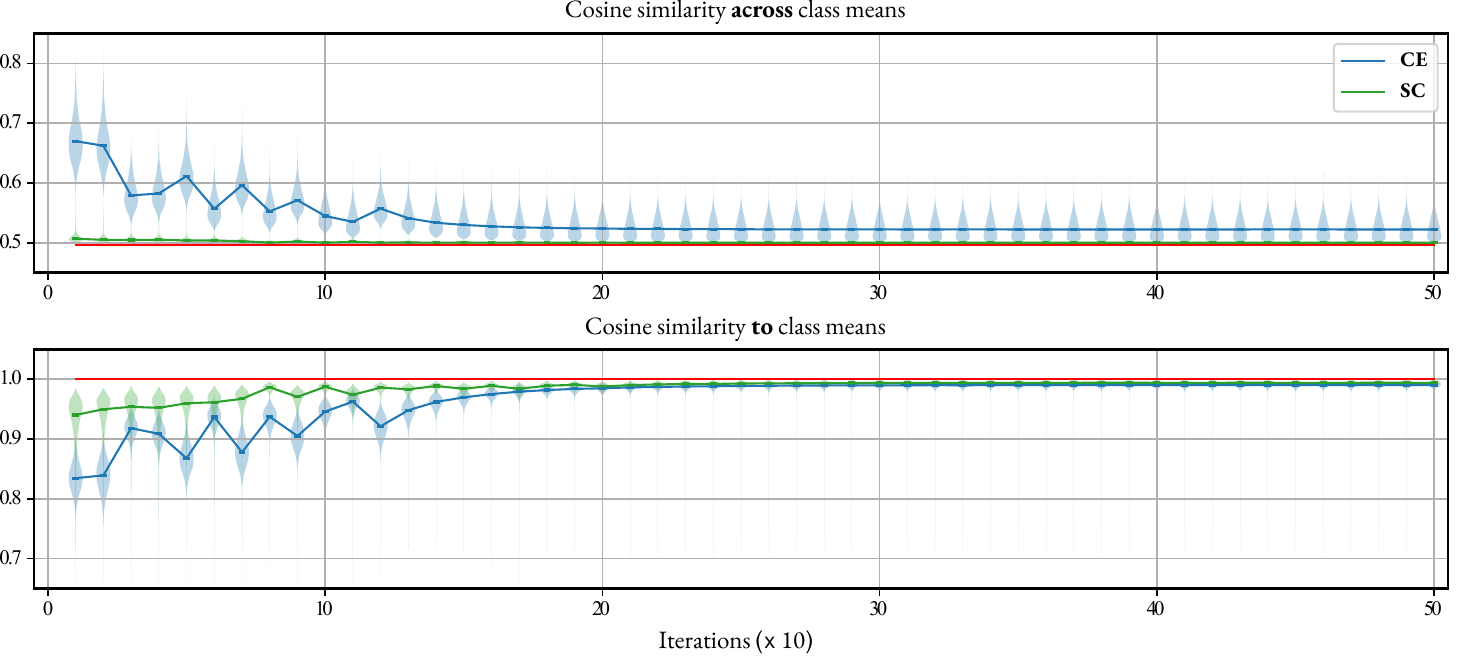}
        \caption{\textbf{CIFAR100} (\underline{without} augmentation)}
        \label{appendix:figure:experiment_cifar100_woaug}
    \end{subfigure}
    \vskip4ex
    \begin{subfigure}[b]{0.99\textwidth}
        \centering
        \includegraphics[width=0.99\textwidth]{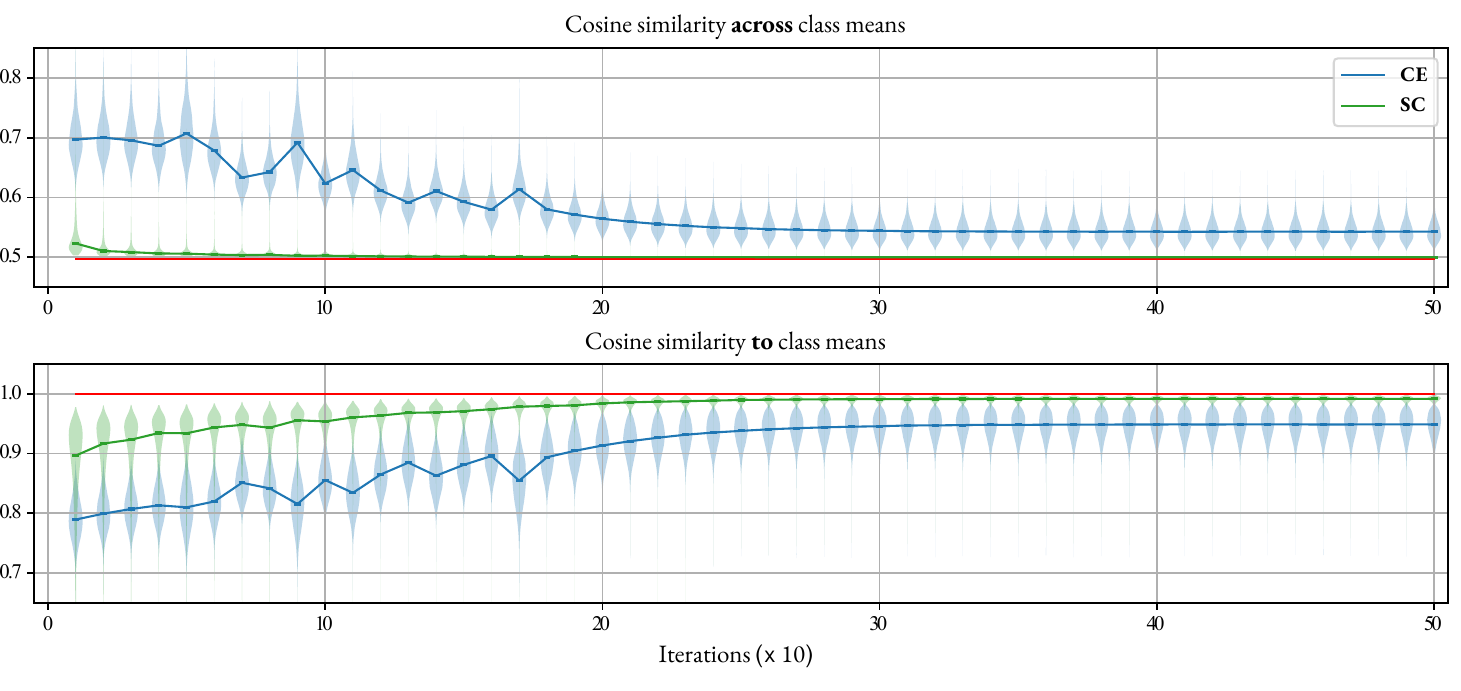}
        \caption{\textbf{CIFAR100} (\underline{with} augmentation)}
        \label{appendix:figure:experiment_cifar100_waug}
    \end{subfigure}
    \caption{
        Distribution of geometric properties of representations, $\enc_\theta(x_n)$, tracked during training. Representations are obtained from a ResNet-18 model trained \subref{appendix:figure:experiment_cifar100_waug} \underline{with}  and \subref{appendix:figure:experiment_cifar100_woaug} \underline{without} data augmentation on \textbf{CIFAR10}, with \textbf{CE} and \textbf{SC}, respectively. \textcolor{tabblue}{Blue} and \textcolor{tabgreen}{green} lines indicate the evolution of the medians over the iterations;
        \textcolor{tabred}{Red} lines indicate the \emph{sought-for} value at a regular simplex configuration.
    \label{fig:suppmat_experiments_cifar100}}
\end{figure}

\end{adjustwidth}

\end{document}